\documentclass[lettersize,journal]{IEEEtran}
\IEEEoverridecommandlockouts

\usepackage{amsmath,amssymb,amsfonts,bm,mathtools}
\mathtoolsset{showonlyrefs} 

\usepackage{graphicx}
\usepackage{booktabs}
\usepackage{multirow}
\usepackage{array}
\usepackage{tabularray}     
\usepackage{makecell}
\usepackage{float}
\usepackage{stfloats}       

\usepackage[caption=false,font=footnotesize]{subfig}

\usepackage{tikz}
\usepackage{pgfplots}
\pgfplotsset{compat=1.18}

\usepackage{enumitem}
\setlist{nosep,leftmargin=*} 

\usepackage{siunitx}
\usepackage{microtype} 
\usepackage{algorithm}
\usepackage{algpseudocode} 

\usepackage{xcolor}
\usepackage{listings}

\lstdefinelanguage{Solidity}{
  morekeywords={
    contract,event,function,external,public,private,internal,
    require,emit,returns,address,bytes32,bytes,calldata,memory,
    mapping,struct,modifier,uint256,bool
  },
  sensitive=true,
  morecomment=[l]{//},
  morecomment=[s]{/*}{*/},
  morestring=[b]"
}

\usepackage{amsthm}

\theoremstyle{plain}
\newtheorem{theorem}{Theorem}[section]
\newtheorem{lemma}[theorem]{Lemma}

\theoremstyle{definition}

\theoremstyle{remark}
\newtheorem{remark}[theorem]{Remark}

\newcolumntype{C}[1]{>{\centering\arraybackslash}p{#1}}

\usepackage{url}
\usepackage[hidelinks]{hyperref}

\usepackage[nameinlink,capitalize]{cleveref} 

\crefformat{equation}{(#2#1#3)}
\Crefformat{equation}{Equation~(#2#1#3)}

\crefname{section}{Sec.}{Secs.}
\Crefname{section}{Sec.}{Secs.}
\crefname{subsection}{Sec.}{Secs.}
\Crefname{subsection}{Sec.}{Secs.}
\crefname{figure}{Fig.}{Figs.}
\Crefname{figure}{Fig.}{Figs.}

\newcommand\orcidicon[1]{%
  \href{https://orcid.org/#1}{%
    \IfFileExists{orcid.png}{\includegraphics[width=8pt]{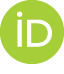}}{\textsuperscript{\tiny ORCID}}%
  }%
}

\title{ZK-HybridFL: Zero-Knowledge Proof-Enhanced Hybrid Ledger for Federated Learning}

\author{%
  \IEEEauthorblockN{%
    \textsuperscript{*}Amirhossein Taherpour\orcidicon{0000-0003-4647-102X}%
    , and
    \textsuperscript{*}Xiaodong Wang\orcidicon{0000-0002-2945-9240}, Fellow, IEEE%
  }\\%
  \IEEEauthorblockA{%
    \textsuperscript{*}Electrical Engineering Department, Columbia University\\%
    \href{mailto:at3532@columbia.edu}{at3532@columbia.edu},
    \href{mailto:xw2008@columbia.edu}{xw2008@columbia.edu}%
  }%
}

\IEEEpubid{}

\begin{document}

\maketitle

\begin{abstract}
Federated learning (FL) enables collaborative model training while preserving data privacy, yet both centralized and decentralized approaches face challenges in scalability, security, and update validation. We propose ZK-HybridFL, a secure decentralized FL framework that integrates a directed acyclic graph (DAG) ledger with dedicated sidechains and zero-knowledge proofs (ZKPs) for privacy-preserving model validation. The framework uses event-driven smart contracts and an oracle-assisted sidechain to verify local model updates without exposing sensitive data. A built-in challenge mechanism efficiently detects adversarial behavior. In experiments on image classification and language modeling tasks, ZK-HybridFL achieves faster convergence, higher accuracy, lower perplexity, and reduced latency compared to Blade-FL and ChainFL. It remains robust against substantial fractions of adversarial and idle nodes, supports sub-second on-chain verification with efficient gas usage, and prevents invalid updates and orphanage-style attacks. This makes ZK-HybridFL a scalable and secure solution for decentralized FL across diverse environments.
\end{abstract}


\IEEEpeerreviewmaketitle

\section{Introduction}
\label{sec:intro}

\IEEEPARstart{F}{ederated} learning (FL) has emerged as a promising paradigm for collaboratively training models across distributed data silos while preserving privacy~\cite{BB3}. However, traditional FL architectures typically rely on centralized coordinators, potentially creating single points of failure and trust bottlenecks~\cite{WWWZ1}. To overcome these limitations, decentralized FL frameworks have been proposed that leverage distributed ledger technologies (DLTs) for trustless coordination~\cite{BB6}.

A variety of blockchain-based approaches illustrate the potential of decentralizing FL. For instance, BlockFL~\cite{BB9} optimizes secure model exchanges and rewards by adjusting the block generation rate, while PIRATE~\cite{BB10} implements a sharding-based design to achieve Byzantine resilience through secure gradient aggregation. Other notable efforts include adaptive FL methods that reduce communication overhead by up to 50\%~\cite{BB11}, incentive mechanisms grounded in reputation systems~\cite{BB12}, proof-of-stake (PoS)-based resource-saving algorithms~\cite{BB13}, Top-$k$ compression for limiting communication costs~\cite{BB14}, and blockchain-empowered frameworks for 5G-enabled unmanned aerial vehicles (UAVs)~\cite{BB15}. Further studies have analyzed energy consumption in blockchain networks~\cite{BB16}, latency optimization~\cite{BB9}, and integration with mobile edge computing~\cite{BB18}, while foundational research has established theoretical bounds for resource-constrained FL settings~\cite{BB20} and proposed reputation-oriented incentive schemes for preserving Internet of Things (IoT) data privacy~\cite{BB21}.

Despite these advances, significant obstacles remain. Current frameworks often struggle to balance scalability, security, and privacy when operating at scale or in diverse, adversarial environments~\cite{WWZ2}. Synchronous systems (e.g., those relying on Proof of Work (PoW) or Practical Byzantine Fault Tolerance (PBFT)) can become bottlenecks under straggler nodes, while solutions that rely on public datasets for validation risk data exposure and fail to accommodate the heterogeneity of real-world data~\cite{gg5,Taherpour2024HybridChain,Taherpour2025SPIDChain}. Moreover, advanced gradient inversion techniques~\cite{BBB1}, label inference attacks~\cite{BBB2}, and disaggregation methods~\cite{BBB3} highlight the difficulty of protecting participant information purely through naive data-sharing defenses. Although differential privacy~\cite{BBB4} can mitigate certain risks, it does not resolve the fundamental challenge of publicly verifying the integrity of local model updates at scale~\cite{gg1}.

\subsection{State-of-the-Art and Challenges}
Blade-FL~\cite{BB7} removes the need for a centralized server by integrating blockchain directly into the FL workflow. In each epoch, nodes train their local models on mini-batches, then sign and broadcast the updates as blockchain transactions. These updates are verified against a public dataset based on a loss threshold and are aggregated---by averaging the validated updates---into a global model. Subsequently, nodes compete to mine a new block via PoW, which records both the verified transactions and the aggregated model. Once the block is validated network-wide, the updated global model is adopted for the next training round. This cycle repeats until convergence.

ChainFL~\cite{BB8} adopts a hierarchical blockchain structure of node, subchain, and mainchain layers to facilitate FL without a central server. A task is initiated via a smart contract on a DAG ledger, broadcasting the request across shards. Within each shard, the Subchain Leader Node (SLN) distributes the current global model to participating nodes, collects their locally trained updates, and aggregates them---weighted by dataset sizes---using the Raft consensus mechanism before committing the aggregated result to that shard's blockchain. The SLN then submits this shard-level model to the DAG, where top models from different shards are evaluated on a public dataset to form the next global model. This iterative process of local training, aggregation, and cross-shard validation continues until the smart contract conditions signal completion of the FL task.

Both Blade-FL and ChainFL introduce valuable decentralization strategies but face critical limitations in practice. Blade-FL's reliance on PoW leads to heavy computational overhead and high energy consumption, while the use of a public validation dataset raises privacy concerns and risks adversarial manipulation. ChainFL's hierarchical, sharded structure, though beneficial for scaling, still depends on centralized public dataset evaluations and suffers from cross-shard synchronization overhead that can allow stale or malicious updates to proliferate. Motivated by these shortcomings, our work proposes a novel framework that integrates a DAG-based ledger with sidechains and employs zero-knowledge proofs (ZKPs) for privacy-preserving, on-device model validation. This approach not only eliminates the need for public datasets but also boosts efficiency and scalability by ensuring only genuine, high-quality updates are accepted into the global model.

At a broader level, existing solutions also highlight the unmet need for robust and efficient methods of verifying local model updates without compromising user data. Conventional methods often resort to public datasets for legitimacy checks~\cite{gg5,WWZ3}, exposing privacy vulnerabilities and undermining real-world representativeness. Advanced cryptographic protocols, such as ZKPs~\cite{BBB6}, can address privacy and correctness simultaneously, but current implementations often support only specific model types (e.g., linear regression~\cite{BBB9}) or suffer from high computational overheads~\cite{BBB10}. These drawbacks underscore the need for a more scalable, generalizable solution that keeps participant data private while enabling secure, public verification of model contributions.

In this paper, we introduce ZK-HybridFL, a novel decentralized FL architecture that addresses these gaps through an innovative ledger design and cryptographic validation pipeline. Our framework combines a DAG for high-throughput storage of model updates with sidechains dedicated to Event-Driven Smart Contracts (EDSCs). By embedding ZKPs in the sidechain-based contracts, we create a method for verifying training correctness without disclosing sensitive information, thereby eliminating the overhead and security pitfalls of traditional solutions.

\subsection{Contributions}
Our primary contributions are as follows:
\begin{enumerate}
    \item \textbf{Hybrid Ledger System:} We propose a DAG-based ledger for scalable storage of model updates, augmented by sidechains running EDSCs to manage consensus and validation tasks. This structure alleviates bottlenecks inherent in purely PoW- or DAG-based systems.
    \item \textbf{ZKP-Driven Secure Validation:} By integrating ZKPs into EDSCs, our framework enables public verification of inference correctness without exposing private test data or imposing prohibitive computational costs.
    \item \textbf{Experimental Validation:} Through extensive simulations, we demonstrate that our approach outperforms Blade-FL~\cite{BB7} and ChainFL~\cite{BB8} in accuracy, convergence speed, and resilience, even in the presence of adversarial or lazy nodes.
\end{enumerate}

The remainder of the paper is organized as follows. \Cref{sec:consensus} introduces the ZKP-based consensus mechanism and DAG ledger for decentralized FL. \Cref{sec:edsc} describes the event-driven smart contracts and oracle-assisted sidechain. \Cref{sec:procedure-analysis} outlines the ZK-HybridFL workflow and analysis. \Cref{sec:simulations} presents simulation results and comparisons with existing schemes, and \Cref{sec:conclusion} concludes the paper with future research directions.

\section{ZKP-based Consensus Mechanism of ZK-HybridFL}
\label{sec:consensus}
In this section, we first present the use of zero-knowledge proofs (ZKPs) for privacy-preserving model validation. Next, we explain the directed acyclic graph (DAG) ledger, which manages scalable and decentralized model updates. Finally, we discuss the challenge mechanism for verifying blocks and resolving conflicts to ensure system integrity.

\subsection{ZKP}
\label{Sec:ZKP}

\subsubsection{Motivation of using private data for model validation}
When selecting between a public test dataset and private test datasets secured via ZKP in decentralized federated learning (FL), opting for private test datasets offers significant advantages in both privacy and model quality. Utilizing a public test dataset simplifies evaluation by providing a uniform benchmark for all nodes, which facilitates straightforward comparisons. However, the public and static nature of such datasets may result in models being validated on a limited variety of unseen data, potentially restricting their robustness and adaptability~\cite{BBB4}. Additionally, as nodes continuously refine their local models to perform well on the public test set, the evolving patterns in their loss or accuracy can inadvertently expose information about their private training data through techniques such as membership inference or model inversion~\cite{kkkk6,kkkk7}. In contrast, employing private test datasets with ZKPs ensures that each local model is assessed using diverse and private test samples. This approach enhances exposure to a broader data distribution, thereby improving the global model's quality and generalizability while simultaneously protecting data privacy.

\subsubsection{The ZKP process}
A ZKP is a cryptographic technique that enables one party (the prover) to convince another party (the verifier) of the validity of a statement without revealing any information beyond the fact that the statement is valid. In computation-related applications, the prover and verifier agree on a specific algorithm or function. The goal is for the prover to demonstrate to the verifier that the output of the algorithm or function is indeed the result of a particular input, without disclosing any additional details about the input or process.

In ZK-HybridFL, ZKPs enable a node \(j\) (prover) to demonstrate to other nodes \(j'\) (verifiers) that its inference output \(\mathcal{Y}_j^t\) at time step \(t\) is correctly derived by applying its publicly available model \(\mathbf{W}_j^t\), trained during epoch \(t\), to its private test batch \(\mathcal{D}_j^{t,\mathrm{test}}\) using the inference algorithm \(\mathcal{I}\) without revealing any details about \(\mathcal{D}_j^{t,\mathrm{test}}\) or the internal computations. Recall that each node \(j\) maintains a local dataset \(\mathcal{D}_j^t\), and a training mini-batch \(\mathcal{D}_j^{t,\mathrm{train}}\) is used to obtain the model \(\mathbf{W}_j^t\). The test mini-batch \(\mathcal{D}_j^{t,\mathrm{test}}\) is drawn from the remaining data, i.e., \(\mathcal{D}_j^{t,\mathrm{test}} \subset \mathcal{D}_j^t \setminus \mathcal{D}_j^{t,\mathrm{train}}\).

To achieve ZKP, a predefined, immutable program called the circuit is created based on the given inference algorithm \(\mathcal{I}\). Once compiled, this circuit executes a sequence of unforgeable operations that generate both the predicted output \(\mathcal{Y}_j^t\) and its associated loss value \(\mathcal{L}(\mathcal{D}_j^{t,\mathrm{test}}, \mathcal{Y}_j^t)\), which measures the quality of the model \(\mathbf{W}_j^t\). In the process, it produces explicitly defined, immutable intermediate results, referred to as witnesses and denoted by \(\mathcal{U}_j^t\). To ensure the integrity of this computation, a verifier must check the consistency of the inputs \(\mathcal{D}_j^{t,\mathrm{test}}\) and \(\mathbf{W}_j^t\), the intermediate results \(\mathcal{U}_j^t\), and the outputs \(\mathcal{Y}_j^t\) and \(\mathcal{L}(\mathcal{D}_j^{t,\mathrm{test}}, \mathcal{Y}_j^t)\). In particular, the prover node \(j\) and the verifier node \(j'\) follow the sequence of actions outlined below:

\begin{enumerate}
  \item \textbf{Commit:}
  At the beginning of epoch~\(t\), the trainer binds both its freshly trained weight tensor \(\mathbf{W}_j^t\) and its private test batch \(\mathcal{D}_j^{t,\mathrm{test}}\) by posting polynomial commitments
  \[
    C_{j}^{t,\mathrm{model}} = \mathsf{Commit}(\mathbf{W}_j^t), 
    \qquad
    C_{j}^{t,\mathrm{test}} = \mathsf{Commit}(\mathcal{D}_j^{t,\mathrm{test}})
  \]
  to the sidechain.\footnote{A simple hash is sufficient at this stage; \cref{sec:zkp-inst} gives the concrete KZG instantiation and its \SI{35}{k.gas} verification cost.}

  \item \textbf{Key availability:}
  A one-time powers-of-tau ceremony (run off-chain by the oracle committee) produces a circuit-specific proving key \(\mathsf{pk}\) and a public verification key \(\mathsf{vk}\). Because the inference circuit is fixed for the entire task, \(\mathsf{pk}\) and \(\mathsf{vk}\) are reused by every node and for every epoch; only the private witness \((\mathbf{W}_j^t,\mathcal{D}_j^{t,\mathrm{test}})\) changes.

  \item \textbf{Proof generation:}
  Using the proving key \(\mathsf{pk}\) and its private witness \((\mathbf{W}_j^t,\mathcal{D}_j^{t,\mathrm{test}})\), node \(j\) evaluates the circuit to obtain \(\mathcal{U}_j^t\) (the full witness) and the public outputs \(\mathcal{Y}_j^t\) and \(\mathcal{L}_j^t\). It then computes the non-interactive Groth16 proof
  \[
    \Pi_j^t = \mathsf{Prove}\bigl(\mathsf{pk}, \mathcal{U}_j^t\bigr),
  \]
  where all Fiat–Shamir challenges are derived deterministically from
  \(\mathsf{H}\bigl(\mathrm{blockHash}\,\Vert\,C_{j}^{t,\mathrm{model}}\,\Vert\,C_{j}^{t,\mathrm{test}}\bigr)\), so no live randomness exchange is needed. Finally, trainer \(j\) broadcasts
  \[
    Z_j^t =
    \bigl(
      \Pi_j^t,\,
      \mathcal{Y}_j^t,\,
      \mathcal{L}_j^t,\,
      C_{j}^{t,\mathrm{model}},\,
      C_{j}^{t,\mathrm{test}}
    \bigr).
  \]

  \item \textbf{Verification:}
  Any verifier \(j'\) fetches the verification key \(\mathsf{vk}\) from the sidechain and checks
  \[
    \mathsf{Verify}\bigl(
      \mathsf{vk},\,
      \Pi_j^t,\,
      C_{j}^{t,\mathrm{model}},\,
      C_{j}^{t,\mathrm{test}},\,
      \mathcal{Y}_j^t,\,
      \mathcal{L}_j^t
    \bigr) = 1.
  \]
  A result of \(1\) certifies that the public outputs \(\mathcal{Y}_j^t\) and \(\mathcal{L}_j^t\) were produced exactly by the committed model and test data; otherwise, the update is rejected. Once verification succeeds, a \textsf{ProofOK} event is emitted, allowing the contribution to enter the DAG.
\end{enumerate}

A detailed cryptographic instantiation, including the polynomial-commitment scheme and trusted setup, is given in the supplementary material (\cref{sec:zkp-inst}). The keys \(\mathsf{pk}\) and \(\mathsf{vk}\) are fixed across nodes and epochs, while each proof \(\Pi_j^t\) is specific to its witness. Completeness, soundness, and zero-knowledge follow from the Groth16 construction (see \cref{sec:zkp-security}).

\begin{remark}
ZKPs are generally categorized into two main types: succinct non-interactive arguments of knowledge (SNARKs) and scalable transparent arguments of knowledge (STARKs). In our scheme, we utilize ZK-SNARKs due to their succinct proofs, efficient verification, and compatibility with blockchain-based systems. SNARKs such as Groth16~\cite{kkkk1} and Pinocchio~\cite{kkkk2} rely on elliptic-curve cryptography and a trusted setup, which is effectively managed in our scheme by an oracle. While STARKs offer advantages such as transparency (no trusted setup) and post-quantum security, their larger proof sizes and higher computational costs make them less practical in our context. For a detailed discussion of these ZKP methods, see~\cite{kkkk3,kkkk4}.
\end{remark}

\begin{remark}
To prevent proof-generation latency from inflating overall training time, ZK-HybridFL employs a \emph{predict-then-prove} workflow. Immediately after completing a local training step, each node (i) broadcasts its model update together with its hash, and then (ii) proceeds with the next stochastic gradient descent (SGD) iteration without waiting for the zero-knowledge proof to be produced. Proofs are generated asynchronously in the background and thus are effectively amortized over subsequent training steps. A two-epoch (configurable) grace period is provided before any proof failures trigger the challenge mechanism, ensuring that transient delays do not stall the protocol under normal hardware conditions.
\end{remark}

While the core ZKP workflow verifies inference correctness, it does not prevent subtler ``stealth'' attacks, such as submitting trivially small or semantically unchanged updates. We develop extended ZKP defenses to counter these threats. A formal security analysis of these defenses is presented in the supplementary material (\cref{subsec:attacks}), with empirical validation provided by detailed simulations (\cref{sec:extZKP-setup}).

\subsection{DAG}
\label{subsec:dag}
By using a modified IOTA Coordicide consensus mechanism~\cite{VVV1} as described next, ZK-HybridFL employs a DAG ledger to facilitate interactions among trainer nodes while keeping the overall network scalable and secure.

The basic DAG operations used in ZK-HybridFL, such as tip-based parent selection, weight accumulation for confirmation, and graph reachability checks that support the challenge process, are adopted almost verbatim from the Coordicide design of the IOTA ledger.

Our contributions lie in how these established primitives are repurposed for FL. First, every block must include a ZKP that is verified on-chain before the block can gain any weight. This guarantees that only correctly trained updates can influence the ledger. Second, the loss-aware aggregation policy (\cref{sec:loss-aware}) admits only the most recent confirmed blocks and ranks candidate parents by their validated loss. As a result, low-quality or adversarial updates accrue insufficient weight and are naturally pruned from the ledger. To the best of our knowledge, no prior DAG-based blockchain combines proof-gated admission with loss-aware parent selection in the context of a decentralized parameter server. This dual gating mechanism is the core algorithmic contribution of ZK-HybridFL's DAG layer.

In our proposed scheme, each node \(j\) submits a block \(D_j(t) = [\mathbf{W}_j^t, Z_j^t]\) to the DAG ledger during epoch \(t\), where \(\mathbf{W}_j^t\) is the updated local model and \(Z_j^t\) is the ZKP bundle described in \cref{Sec:ZKP}. Blocks \(D_j(t)\) within the DAG can exist in one of three states: confirmed, tip, or unconfirmed. Confirmed blocks \(D_j^{C}(t)\) have an aggregated weight (AW) exceeding a predefined threshold, making them validated and integrated into the stable, immutable part of the DAG. Tip blocks \(D_j^{T}(t)\) are the most recent blocks at the DAG's frontier, ready to be extended by incoming blocks. Unconfirmed blocks \(D_j^{U}(t)\) are those added to the DAG but still in a pending state, awaiting sufficient AW to reach confirmation.

During each epoch \(t\), the global model \(\tilde{\mathbf{W}}_j^{t}\) is obtained through local model aggregation using the model updates from the most recently confirmed blocks, as detailed in \cref{SC4}. Next, we provide a detailed explanation of the consensus mechanism for the DAG component and the computation of the AW used for block confirmation within the DAG.

To submit its contributed block \(D_j(t)\) for epoch \(t\) to the DAG, node \(j\) begins by randomly selecting \(K_T\) tip blocks. These chosen blocks form a set \(\mathcal{D}_j^{T}(t)\) of potential parents with \(|\mathcal{D}_j^{T}(t)| = K_T\). For each block \(D_{j'}^{T}(t')\) in this set, where \(t' < t\), node \(j\) verifies its ZKP and extracts its associated loss. As a result, node \(j\) identifies \(K_V\) blocks with valid ZKPs that have the lowest loss values to form the parent block set \(\mathcal{D}_j^{V}(t)\) for \(D_j(t)\), with \(|\mathcal{D}_j^{V}(t)| = K_V\). Then \(D_j(t)\) is labeled as a tip block \(D_j^{T}(t)\), and all tip blocks \(D_{j'}^{T}(t') \in \mathcal{D}_j^{V}(t)\) are changed to unconfirmed status, i.e., \(D_{j'}^{U}(t')\). Finally, \(D_j^{T}(t)\) is integrated into the DAG by selecting the blocks in \(\mathcal{D}_j^{V}(t)\) as its predecessors. As a result of these actions performed by all nodes, the AW values of the blocks on the DAG change, which may alter the validity status of some blocks. Specifically, certain blocks \(D_{j''}^{U}(t'')\) for \(t'' < t' < t\) become confirmed, i.e., they are updated to \(D_{j''}^{C}(t'')\). Next, we explain how the AW of each block on the DAG is altered and how the transition from tip to confirmed status occurs.

\Cref{DAG_Final} illustrates a DAG ledger for two consecutive epochs \(t-1\) and \(t\). In epoch \(t-1\), each vertex represents a block \(D_{j_i}(t_i)\) by node \(j_i\), for \(i = 1,\ldots,4\), at epoch \(t_i < t-1\). A direct edge from one vertex to another, for example, from \(D_{j_2}^{T}(t_2)\) to \(D_{j_1}^{U}(t_1)\), signifies that \(j_2\) verifies block \(D_{j_1}^{U}(t_1)\). \Cref{DAG_Final} also displays the status evolution of the blocks after one epoch. A block is labeled as a tip if it lacks incoming validation edges. Therefore, in this figure, for epochs \(t-1\) and \(t\), the blocks \(\{D_{j_2}^{T}(t_2), D_{j_3}^{T}(t_3), D_{j_4}^{T}(t_4)\}\) and \(\{D_{j_5}^{T}(t_5), D_{j_6}^{T}(t_6)\}\) are tip blocks, respectively.

Moreover, a block that is neither a tip nor confirmed is referred to as an unconfirmed block. Hence, while blocks \(\{D_{j_2}^{T}(t_2), D_{j_3}^{T}(t_3), D_{j_4}^{T}(t_4)\}\) are tips in epoch \(t-1\), block \(D_{j_3}^{U}(t_3)\) is classified as unconfirmed in epoch \(t\), while \(\{D_{j_2}^{T}(t_2), D_{j_4}^{T}(t_4)\}\) remain tips. Observe that the status of \(D_{j_1}^{U}(t_1)\) transitions from unconfirmed to confirmed, which is denoted by \(D_{j_1}^{C}(t_1)\). Specifically, a block achieves confirmed status when its AW surpasses a predefined threshold \(\eta\). The AW is intricately linked to a weight vector that reflects the relative influence of the nodes.

The weight \(\omega_j\) of each node \(j\) with \(\omega_j > 0\) and \(\sum_{j} \omega_j = 1\) is determined by the number of tokens it has staked. In DAG ledgers, staking refers to the process by which nodes lock their tokens to gain voting power in validating transactions and appending new blocks to the ledger. The more tokens a node stakes, the higher its associated validation weight, influencing the cumulative weight of the blocks it approves. This weight is assumed to remain constant over different epochs but can be dynamically adjusted based on node behavior. Specifically, if a node incorrectly disputes the validity of a block, as discussed in the challenge mechanism outlined in \cref{subsection:challenge}, then a portion of its staked tokens is slashed as a penalty, directly reducing its weight \(\omega_j\). This incentivizes honest participation by making malicious behavior costly. To establish the network's initial trust and ensure a secure starting point, a set of genesis blocks with confirmed status are proposed by the oracle committee (as discussed in \cref{sec:oracle-committee}) as trusted and correct blocks, providing a foundation for subsequent validations.

The AW of a block \(D_{j_i}(t_i)\) is calculated as
\[
  \omega_{j_i} + \sum_{D_{j_{i'}}(t_{i'}) \in \mathcal{F}(D_{j_i}(t_i))} \omega_{j_{i'}},
\]
where \(\mathcal{F}(D_{j_i}(t_i))\) denotes the future cone of \(D_{j_i}(t_i)\), encompassing all blocks that it validates, either directly or through a series of validations. For example, in \Cref{DAG_Final} and in epoch \(t\), blocks \(\{D_{j_1}^{C}(t_1), D_{j_4}^{U}(t_4)\}\) are in the future cone of block \(D_{j_6}^{T}(t_6)\), as block \(D_{j_4}^{U}(t_4)\) is validated by it directly, while \(D_{j_6}^{T}(t_6)\) validates block \(D_{j_1}^{C}(t_1)\) indirectly. Moreover, \Cref{DAG_Final} illustrates that the AW of block \(D_{j_1}(t_1)\) surpasses \(\eta\) after the addition of the weights of blocks \(D_{j_5}^{T}(t_5)\) and \(D_{j_6}^{T}(t_6)\), and its status transitions from unconfirmed to confirmed.

\begin{figure*}[ht]
  \centering
  \begin{tikzpicture}[scale=0.43, transform shape]
    \begin{scope}[shift={(-3,0)}]
      \node at (-5,6) {\Huge Epoch $t-1$};
      \node[circle, draw, fill=black, inner sep=1pt,
            label=above:{\Huge $\bigl(\textcolor{yellow!60!black}{D_{j_{1}}^{U}(t_{1})},\, \textcolor{blue}{\omega_{j_{1}} + \omega_{j_{2}} + \omega_{j_{3}} + \omega_{j_{4}}}\bigr)$}] (D1) at (-5,3) {};
      \node[circle, draw, fill=black, inner sep=1pt,
            label=below:{\Huge $\bigl(\textcolor{red}{D_{j_{2}}^{T}(t_{2})},\, \textcolor{blue}{\omega_{j_{2}}}\bigr)$}] (D2) at (-10,0) {};
      \node[circle, draw, fill=black, inner sep=1pt,
            label=below:{\Huge $\bigl(\textcolor{red}{D_{j_{3}}^{T}(t_{3})},\, \textcolor{blue}{\omega_{j_{3}}}\bigr)$}] (D3) at (-4,0) {};
      \node[circle, draw, fill=black, inner sep=1pt,
            label=below:{\Huge $\bigl(\textcolor{red}{D_{j_{4}}^{T}(t_{4})},\, \textcolor{blue}{\omega_{j_{4}}}\bigr)$}] (D4) at (2,0) {};

      \draw[->, >=stealth, thick] (D2) -- (D1);
      \draw[->, >=stealth, thick] (D3) -- (D1);
      \draw[->, >=stealth, thick] (D4) -- (D1);
    \end{scope}

    \begin{scope}[shift={(3,0)}]
      \node at (5,6) {\Huge Epoch $t$};
      \node[circle, draw, fill=black, inner sep=1pt,
            label=above:{\Huge $\bigl(\textcolor{green}{D_{j_{1}}^{C}(t_{1})},\, \textcolor{blue}{\omega_{j_{1}} + \omega_{j_{2}} + \omega_{j_{3}} + \omega_{j_{4}} + \omega_{j_{5}} + \omega_{j_{6}}}\bigr)$}] (D5) at (5,3) {};

      \node[circle, draw, fill=black, inner sep=1pt] (D6) at (1,0) {};
      \node at (-1.5,1.3) {\Huge $\bigl(\textcolor{yellow!60!black}{D_{j_{2}}^{U}(t_{2})},\, \textcolor{blue}{\omega_{j_{2}} + \omega_{j_{5}}}\bigr)$};

      \node[circle, draw, fill=black, inner sep=1pt] (D7) at (5,0) {};
      \node at (8.5,-0.5) {\Huge $\bigl(\textcolor{yellow!60!black}{D_{j_{3}}^{U}(t_{3})},\, \textcolor{blue}{\omega_{j_{3}} + \omega_{j_{5}}}\bigr)$};

      \node[circle, draw, fill=black, inner sep=1pt] (D8) at (15,0) {};
      \node at (19,0.5) {\Huge $\bigl(\textcolor{yellow!60!black}{D_{j_{4}}^{U}(t_{4})},\, \textcolor{blue}{\omega_{j_{4}} + \omega_{j_{6}}}\bigr)$};

      \node[circle, draw, fill=black, inner sep=1pt,
            label=below:{\Huge $\bigl(\textcolor{red}{D_{j_{5}}^{T}(t_{5})},\, \textcolor{blue}{\omega_{j_{5}}}\bigr)$}] (D9) at (1,-3) {};
      \node[circle, draw, fill=black, inner sep=1pt,
            label=below:{\Huge $\bigl(\textcolor{red}{D_{j_{6}}^{T}(t_{6})},\, \textcolor{blue}{\omega_{j_{6}}}\bigr)$}] (D10) at (7,-3) {};

      \draw[->, >=stealth, thick] (D6) -- (D5);
      \draw[->, >=stealth, thick] (D7) -- (D5);
      \draw[->, >=stealth, thick] (D8) -- (D5);
      \draw[->, >=stealth, thick] (D9) -- (D6);
      \draw[->, >=stealth, thick] (D9) -- (D7);
      \draw[->, >=stealth, thick] (D10) -- (D8);
    \end{scope}
  \end{tikzpicture}
  \caption{Evolution of a DAG ledger over two consecutive epochs. Nodes are color-coded to indicate their status: red for tips, yellow for unconfirmed, and green for confirmed. The blue text represents the AW associated with each node.}
  \label{DAG_Final}
\end{figure*}
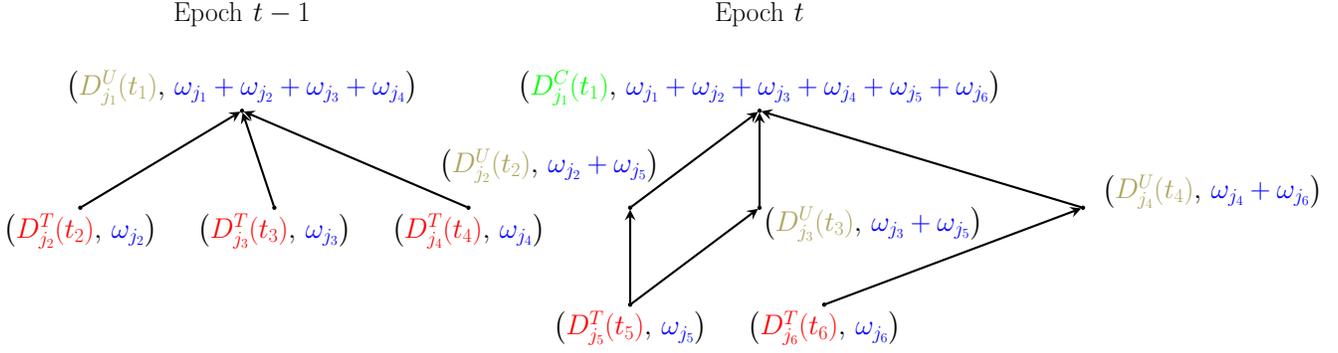

\subsection{Challenge Mechanism}
\label{subsection:challenge}
In DAG ledgers, each block submission requires selecting parent blocks from the current set of tip blocks. Under normal circumstances, this process ensures that new blocks are seamlessly integrated into the network by referencing recent, valid blocks. However, a critical vulnerability known as the \emph{orphanage attack}~\cite{KII9} exploits this mechanism. In an orphanage attack, malicious nodes intentionally select their own invalid blocks, or those of their colluding partners, as parents when submitting new blocks. By continuously referencing their own invalid blocks, these dishonest nodes effectively remove such blocks from the tip pool. Consequently, other honest nodes are unable to examine the validity of these blocks while selecting parents, allowing the invalidity of these strategically added blocks to go unnoticed. As a result, these invalid blocks remain hidden within the DAG despite being removed from the tip set. Later on, as the DAG grows, valid blocks may indirectly reference these previously isolated invalid blocks, causing them to accumulate weight over time. Ultimately, this process can lead to the confirmation of invalid blocks. The confirmation of such blocks not only compromises the quality of the global model derived from confirmed blocks but also alters the topology of the network, disrupting the integrity of the ledger.

The challenge mechanism is motivated by the need to counter this type of attack and restore transparency and fairness to the block selection process. This mechanism empowers honest nodes to flag and contest blocks that remain in an unconfirmed status within the DAG. The process begins with the detection of blocks impacted by an orphanage attack. At the core of this detection mechanism is graph reachability analysis (GRA)~\cite{KII10,KII11,KII12}, a structural algorithm designed to identify blocks that have caused other blocks to become isolated within the DAG. The key principle behind GRA is that any valid block should be reachable from at least one tip block through a path in the DAG graph. If a block is found to be disconnected from all active tips, it indicates a potential issue: either the block was unintentionally abandoned due to network delays, or it was strategically isolated as part of an orphanage attack.

For each suspicious block identified through GRA, the corresponding ZKP is verified. By definition, blocks involved in an orphanage attack are invalid, and their associated ZKPs will fail verification. Once such blocks are detected, the detecting node broadcasts a challenge to all trusted oracles within the network. Upon receiving a challenge, the oracles compile a list of disputed blocks and independently verify their validity based on the provided ZKPs.

Following verification, the oracles participate in a supermajority voting process: if more than \(2/3\) of the oracles agree that a block is invalid, it is formally revoked from the DAG. Block revocation means that the invalid block is excluded from future parent selections, and the AW of affected blocks is updated accordingly to neutralize the impact of the revoked block on the network. Conversely, if a supermajority of oracles determine that the disputed block is valid, the network penalizes the challenging node \(j\) by slashing a fraction of its tokens. This penalty reduces the challenger's \(\omega_j\), thereby decreasing its influence and voting power within the DAG for future block submissions.

\begin{remark}
While an incorrect dispute results in penalizing the disputing node, ZK-HybridFL does not penalize the submitter of an invalid block. This design choice reflects a balance between encouraging participation and acknowledging that honest mistakes or network delays can lead to inadvertent errors. Penalizing such actions could stifle honest contributions, especially when no malicious intent is involved. In contrast, initiating a challenge, especially one that incorrectly disputes a valid block, incurs significant costs by triggering a consensus process among oracles and resulting in computational overhead.
\end{remark}

\section{EDSC Mechanism of ZK-HybridFL}
\label{sec:edsc}
\label{SC:background}
\label{SCs}
Our proposed ZK-HybridFL system employs event-driven smart contracts (EDSCs) to automate predefined actions in response to specific network events. By routing every off-chain emission through a single \texttt{EventAdmission} contract (backed by the \texttt{OracleRegistry}), we avoid constant polling and heavyweight sidechain consensus while ensuring that only vetted, canonical events can trigger on-chain logic. \Cref{subsec:smart-contracts} details the main EDSCs used in ZK-HybridFL, and \Cref{subsuboracle} describes the Oracle Committee's verification and event-publication process.

In ZK-HybridFL, all EDSCs are deployed on a dedicated sidechain that isolates contract execution from the main ledger. This design boosts throughput and scalability by handling high-frequency interactions off-chain, without burdening the primary network. The sidechain architecture and its integration within the broader protocol are presented in \cref{sssidechain}. A more detailed technical description of the Oracle Committee and the sidechain's event-ordering mechanism is available in the supplementary material (\cref{sec:sup_arch}).

\subsection{Smart Contracts in ZK-HybridFL}
\label{subsec:smart-contracts}
Each node \( j \) in the ZK-HybridFL process deploys five EDSCs \(\{S_j^k\}_{k=1}^{5}\) on the sidechain, which are independently triggered by validated events from the oracle. Next, we delineate the function and process of each of these five EDSCs.

\subsubsection{Validation Smart Contract \(S_j^1\)}
\label{SC1}

\(S_j^1\) validates the ZKP bundles of the blocks in \(\mathcal{D}_j^T(t)\). It subscribes to the \texttt{EventApproved} log emitted by the \texttt{EventAdmission} contract (cf.\ \cref{subsuboracle_supp}), which the Oracle Committee emits once it has vetted node \(j'\)'s DAG-published bundle \(Z_{j'}^{t'}\). For each block \(D_{j'}^T(t') \in \mathcal{D}_j^T(t)\), \(S_j^1\) retrieves from the sidechain the commitments
\[
  C_{j'}^{t',\mathrm{model}}, \quad
  C_{j'}^{t',\mathrm{test}},
\]
and the global verification key \(\mathsf{vk}\). It then fetches \(\Pi_{j'}^{t'}\), \(\mathcal{Y}_{j'}^{t'}\), and \(\mathcal{L}_{j'}^{t'}\) from the bundle \(Z_{j'}^{t'}\) on the DAG, invokes
\[
  \mathsf{Verify}\bigl(
    \mathsf{vk},\,
    \Pi_{j'}^{t'},\,
    C_{j'}^{t',\mathrm{model}},\,
    C_{j'}^{t',\mathrm{test}},\,
    \mathcal{Y}_{j'}^{t'},\,
    \mathcal{L}_{j'}^{t'}
  \bigr),
\]
and retains the block only if the call returns \(1\). After processing all candidates, \(S_j^1\) ranks the verified blocks by their loss values \(\mathcal{L}_{j'}^{t'}\) and selects the top \(K_V\) with the lowest loss to form \(\mathcal{D}_j^V(t)\). This finalizes the contract's state, enabling node \(j\) to adopt these blocks as parents in the next stage of ZK-HybridFL.

\subsubsection{Submission Smart Contract \(S_{j}^{2}\)}
\label{SC2}

\(S_j^2\) facilitates the submission of node \(j\)'s new block \(D_j(t)\) to the DAG and updates the aggregated weight (AW) of existing blocks. It subscribes to the \texttt{EventApproved} log emitted by the \texttt{EventAdmission} contract (cf.\ \cref{subsuboracle_supp}), which is triggered once the Oracle Committee has vetted and approved the validated set \(\mathcal{D}_j^V(t)\). The event communicates that node \(j\) has determined the parent blocks from \(\mathcal{D}_j^V(t)\) for its new block \(D_j(t)\); it includes necessary information on these parent blocks alongside the model update \(\mathbf{W}_j^t\) and corresponding ZKP \(Z_j^t\) contained in \(D_j(t)\). \(S_j^2\) then updates the AW values associated with the vertices corresponding to these blocks within \(\mathcal{D}_j^V(t)\) on the DAG and their ancestors, acknowledging their role as parents of the new block. Following this update, \(S_j^2\) labels \(D_j(t)\) as a new tip of the DAG (denoted \(D_j^T(t)\)), while altering the status of each parent block \(D_{j'}^T(t') \in \mathcal{D}_j^V(t)\) from tip to unconfirmed (denoted \(D_{j'}^U(t')\)). If the updated AW of any block exceeds a predetermined threshold, \(S_j^2\) marks such blocks as confirmed. In its final step, \(S_j^2\) transitions its state to complete by embedding \(D_j^T(t)\) into the DAG as a new vertex. This vertex is connected via directed edges to the vertices corresponding to the parent blocks in \(\mathcal{D}_j^V(t)\), effectively updating the DAG structure to reflect the addition of node \(j\)'s block and the revised parent relationships.

\subsubsection{Challenge Smart Contract \(S_{j}^{3}\)}
\label{SC3}

\(S_j^3\) manages the dispute mechanism for potentially invalid blocks within the ZK-HybridFL framework. It subscribes to the \texttt{EventApproved} log emitted by the \texttt{EventAdmission} contract (cf.\ \cref{subsuboracle_supp}), which the Oracle Committee emits once it has vetted node \(j\)'s off-chain graph reachability analysis (GRA) detection of a suspicious block \(D_{j'}(t')\). Upon triggering, \(S_j^3\) receives the GRA detection results as inputs comprising the identifiers of the suspicious blocks \(\{D_{j'}(t')\}\), the identity of the node \(j'\) that created them, relevant ZKPs or metadata supporting the suspicion, and a fraction of node \(j\)'s tokens staked as collateral.

Following receipt of these inputs, \(S_j^3\) notifies trusted oracles of the questionable blocks by emitting an event containing details of the suspicious blocks, GRA findings, and supporting evidence. Simultaneously, the contract holds the staked tokens in escrow. This notification prompts the oracles to validate the challenge: they assess the evidence to determine the validity of the suspicious blocks. The output of \(S_j^3\) thus involves alerting the oracles and securely holding the stake while awaiting their consensus. Depending on the oracles' decision, \(S_j^3\) then updates its state accordingly. If the oracles confirm that the challenged blocks are invalid, \(S_j^3\) marks these blocks as revoked within the DAG ledger, ensuring the integrity of the ledger by potentially reconnecting or adjusting references as needed, and releases the staked tokens back to node \(j\). Conversely, if the oracles find the challenge unfounded, \(S_j^3\) imposes a penalty by slashing the staked tokens, reducing node \(j\)'s future influence, and maintains the status of the challenged blocks. This final state transition solidifies the resolution of the dispute, either by cleansing the DAG of invalid blocks or by penalizing incorrect challenges.

\subsubsection{Model Aggregator Smart Contract \(S_{j}^{4}\)}
\label{SC4}
\label{sec:loss-aware}

\(S_j^4\) is assigned to compute the global model \(\tilde{\mathbf{W}}_j^{t}\) for node \(j\). It subscribes to the \texttt{EventApproved} log emitted by the \texttt{EventAdmission} contract (cf.\ \cref{subsuboracle_supp}), which the Oracle Committee emits once node \(j\) has signaled availability for epoch \(t+1\) and the dispute resolution in \(S_j^3\) has concluded. This event carries the latest AW of DAG blocks along with their updated validity status. Upon receiving that log, \(S_j^4\) determines the set \(\mathcal{J}_{j}^t\) of all nodes \(j'\) that produced at least one new confirmed block in epoch \(t\). For each \(j'\), it identifies the most recent epoch \(t_{j'}^*\) such that \(D_{j'}^{t_{j'}^*}\) became confirmed in epoch \(t\), and retrieves the corresponding model update \(\mathbf{W}_{j'}^{t_{j'}^*}\). The contract then computes the global model via weighted aggregation:
\begin{equation}
  \tilde{\mathbf{W}}_j^{t}
   \;=\;
   \sum_{j' \in \mathcal{J}_{j}^t}
       \frac{\omega_{j'}}{\sum_{k \in \mathcal{J}_{j}^t}\omega_{k}}
       \,\mathbf{W}_{j'}^{t_{j'}^*},
  \label{eq:global-model}
\end{equation}
ensuring that only the latest confirmed update from each node contributes and that weights are normalized over \(\mathcal{J}_{j}^t\). Once aggregation completes, \(S_j^4\) transitions to the complete state and returns \(\tilde{\mathbf{W}}_j^{t}\) to node \(j\), enabling it to commence the next training epoch. A detailed ablation study validating the benefits of this stake-weighted aggregation scheme over uniform averaging is presented in the supplementary material (\cref{sec:weighting-ablation}).

\subsubsection{Reward Distribution Smart Contract \(S_{j}^{5}\)}
\label{SC5}

\(S_j^5\) is allocated for dispensing rewards to trainer node \(j\) upon successful submission of its contributed block \(D_{j}^{T}(t)\). It subscribes to the \texttt{EventApproved} log emitted by the \texttt{EventAdmission} contract (cf.\ \cref{subsuboracle_supp}), which the Oracle Committee emits once the submission process in \(S_j^2\) has concluded. This event provides information regarding the contributed block \(D_{j}^{T}(t)\), its corresponding model \(\mathbf{W}_{j}^{t}\), and the verification details confirming successful integration into the DAG. Specifically, once triggered, \(S_j^5\) verifies the final status of block \(D_{j}^{T}(t)\) and the correctness of \(\mathbf{W}_{j}^{t}\), as endorsed by the oracle. It then calculates the reward amount based on established criteria such as model accuracy and the number of valid blocks that node \(j\) has successfully introduced into the system. After finalizing the reward, \(S_j^5\) transitions its state to complete by creating a transaction on the sidechain to credit the trainer node \(j\)'s account with the calculated sum.

\Cref{tab:smart_contracts} summarizes the structure of the smart contracts for ZK-HybridFL, where the input, output, and state variables of these EDSCs in epoch \(t\) are denoted as \(\{I_j^k(t), O_j^k(t), P_j^k(t)\}_{k=1}^{5}\).

\begin{table*}[htbp]
\centering
\caption{Summary of EDSCs for ZK-HybridFL.}
\label{tab:smart_contracts}
\resizebox{\textwidth}{!}{%
\renewcommand{\arraystretch}{1.3}%
\begin{tabular}{|c!{\vrule width 1.5pt}c|c|c|c|}
\hline
\textbf{SC} & \textbf{Major Defined Event Subscription} & \textbf{Input} & \textbf{Output} & \textbf{State} \\ \Xhline{1.5pt}
$S_j^{1}$ & \makecell{Generation of ZKP $Z_j^{t}$} & \makecell{Triggers validation of \\ proofs $Z_{j^\prime}^{t^\prime}$ in blocks \\ $D_{j^\prime}^{T}(t^\prime)$} & \makecell{Validation results of ZKPs \\ and performance evaluations, \\ forming set $\mathcal{D}_{j}^{V}(t)$} & \makecell{Transition to \\ complete after sending \\ results to node $j$} \\ \hline
$S_j^{2}$ & \makecell{Formation of set $\mathcal{D}_{j}^{V}(t)$} & \makecell{Identifies and verifies \\ valid blocks within $\mathcal{D}_{j}^{V}(t)$ \\ on the DAG} & \makecell{Updates DAG by labeling \\ block $D_j(t)$ as tip $D_j^{T}(t)$, \\ changes status of parent blocks} & \makecell{Releases contributed block \\ $D_j^{T}(t)$ into the DAG} \\ \hline
$S_j^{3}$ & \makecell{Detection of potentially \\ invalid block $D_{j^\prime}(t^\prime)$} & \makecell{Receives GRA detection results \\ and token stake \\ from node $j$, including \\ flagged blocks $\{D_{j^\prime}(t^\prime)\}$ and \\ relevant ZKPs} & \makecell{Notifies oracles of suspicious blocks \\ and stakes tokens in escrow; \\ oracles verify challenge validity \\ and determine slashing or refund} & \makecell{Updates DAG by marking \\ invalid blocks and slashing tokens \\ if the challenge fails; \\ returns tokens if challenge succeeds} \\ \hline
$S_j^{4}$ & \makecell{Announcement of trainer availability \\ for next epoch $t+1$} & \makecell{Aggregates model updates \\ $\mathbf{W}_{j^\prime}^{t^\prime}$ from confirmed blocks \\ $D_{j^\prime}^{C}(t^\prime)$} & \makecell{Generates global model \\ $\tilde{\mathbf{W}}_j^{t}$ for epoch $t+1$ \\ via aggregation rule} & \makecell{Sends global model \\ $\tilde{\mathbf{W}}_j^{t}$ to node $j$} \\ \hline
$S_j^{5}$ & \makecell{Formation of block $D_j^{T}(t)$} & \makecell{Verifies completion and integration \\ of model update $D_j^{T}(t)$ and \\ model $\mathbf{W}_j^{t}$} & \makecell{Determines reward amount \\ based on predefined criteria \\ (e.g., accuracy, contribution quality)} & \makecell{Releases calculated reward \\ to trainer node $j$ via blockchain transaction} \\ \hline
\end{tabular}%
}
\end{table*}

\subsection{Oracle-assisted Sidechain}
\subsubsection{Oracle Committee}
\label{subsuboracle}
\label{sec:oracle-committee}
In ZK-HybridFL, the Oracle Committee acts as a trusted intermediary responsible for verifying events published by the nodes before they trigger the EDSCs. When a node publishes an event intended to trigger a smart contract, the oracles initially place a hold on the event and verify it to ensure that it adheres to the expected structure and contains accurate information. For instance, the Oracle Committee checks that the event's content is consistent with the data recorded on the DAG ledger, such as the associated ZKPs, and that the performance metrics in the event align with the corresponding blocks in the DAG. If the event passes all verification checks, the oracles lift the hold on the event within the sidechain, allowing it to trigger any subscribed smart contracts. This process ensures that the event is neither malicious nor incorrectly formatted, thereby preventing invalid events from triggering smart contracts.

This verification mechanism is critical in ZK-HybridFL because, unlike conventional sidechains that use dedicated consensus protocols such as Raft in ChainFL, our sidechain does not employ a separate consensus process. Instead, the Oracle Committee acts as a lightweight verification layer, reducing the overhead typically associated with consensus protocols while ensuring that only valid events can trigger smart contracts. This approach minimizes the scalability limitations of the sidechain by offloading event validation to the Oracle Committee, allowing the DAG ledger to remain the primary consensus layer without bottlenecks caused by sidechain consensus.

\subsubsection{Sidechain}
\label{sssidechain}
The ZK-HybridFL sidechain serves as a specialized ledger for storing and executing EDSCs, ensuring that high-frequency interactions do not overload the DAG. Its block structure is designed to encapsulate the data elements essential for the protocol's operation and network logistics. In the initial blocks, the sidechain records identifiers that distinguish oracle nodes---responsible for validating events---from non-oracle participants, along with stake-related metadata that determines each node's weight \(\omega_j\). During the network's one-time deployment phase, EDSCs are deployed on the sidechain. Once live, subsequent blocks record interactions with these EDSCs, capturing dynamic information such as commitments \(H(\mathcal{D}_j^{t,\mathrm{test}})\) and \(H(\mathbf{W}_{j}^{t})\), proof verifiers \(v_{j}^{t}\) for each node \(j\) and epoch \(t\), reward distributions for successful model submissions, and updates to tokens staked by participants.

Unlike conventional blockchains that rely on consensus mechanisms such as proof of work (PoW) to validate blocks, the ZK-HybridFL sidechain adopts a different strategy. In traditional blockchains, consensus is used to ensure all nodes agree on a single, unique sequence of blocks. For example, in PoW-based systems, nodes compete to propose the next block by solving a computational puzzle. Once a node successfully proposes a block, the network agrees to attach that block to the chain, establishing transaction order and conflict resolution. Consequently, any transaction included in an earlier block (e.g., Block~A) is recognized as having occurred before transactions in subsequently attached blocks (e.g., Block~B).

In contrast, the ZK-HybridFL sidechain does not use a separate consensus algorithm for block validation. Instead, it leverages its event-driven architecture to inherently maintain transaction order. When an event is validated and a corresponding block is attached to the sidechain, the timestamp of this attachment determines the transaction order. Thus, if a transaction is included in a block that is attached with an earlier timestamp, it is automatically considered to have occurred before transactions in later-attached blocks. This built-in ordering mechanism eliminates the need for additional consensus protocols, simplifying transaction validation and ordering on the sidechain.

In decentralized networks like ZK-HybridFL, relying on a single global clock such as a real-world timestamp introduces significant challenges~\cite{KII13}. Clock drift, network latency, and the absence of a trusted time authority all make it difficult to maintain an accurate and uniform notion of time across all nodes. Consequently, event ordering becomes prone to inconsistencies and disagreements when purely dependent on physical clocks. To address this, ZK-HybridFL adopts Lamport clocks~\cite{KII14}, which assign logical timestamps reflecting the causal dependencies among events, rather than relying on physical time references. By doing so, the system circumvents the complexities of clock synchronization in a heterogeneous environment while still guaranteeing a deterministic and consistent event sequence. As each oracle increments its local counter and exchanges logical timestamps with peers, a coherent final order naturally emerges even under asynchronous conditions. This ensures conflict-free execution of smart contracts and recording their corresponding events on the sidechain and allows all nodes to reconcile their states without resorting to traditional consensus protocols. For more details about ordering events in a decentralized network using Lamport clocks, the reader can refer to~\cite{KII15}.

\section{The ZK-HybridFL Procedure and Analysis}
\label{sec:procedure-analysis}

\subsection{Workflow of ZK-HybridFL}
\label{subsection:workflow}
This subsection outlines the workflow of ZK-HybridFL from the viewpoint
of node \(j\) at epoch \(t\). It describes how node \(j\) performs local
training, commits its state, generates proofs, submits blocks, and
participates in the on-chain verification, challenge, and aggregation
stages.\footnote{All on-chain triggers in Stages 2–4 result from threshold-signed \texttt{EventApproved} logs emitted by the \texttt{EventAdmission} contract. These logs are produced only after off-chain validation by the Oracle Committee; see \cref{subsuboracle_supp}.} These four stages repeat until the model converges.

\subsubsection{Stage 1: Training and Proof Certification}
\label{subsec:stage1}
\begin{enumerate}[label=\alph*.]

  \item \textbf{Local training.}
        Node \(j\) downloads the latest global model
        \(\tilde{\mathbf{W}}_{j}^{\,t-1}\) from the sidechain and runs
        \(R\) stochastic gradient descent (SGD) iterations on its private training batch
        \(\mathcal{D}_j^{t,\mathrm{train}}\) (mini-batch size \(B\)), yielding
        updated weights \(\mathbf{W}_{j}^{t}\).

  \item \textbf{Commit.}
        It posts two KZG commitments
        \[
          C_{j}^{t,\mathrm{model}}
            = \mathsf{Commit}(\mathbf{W}_{j}^{t}),
          \quad
          C_{j}^{t,\mathrm{test}}
            = \mathsf{Commit}(\mathcal{D}_{j}^{t,\mathrm{test}})
        \]
        to the sidechain, together with a Lamport timestamp.

  \item \textbf{Proof generation (off-chain).}
        Using the universal proving key \(\mathsf{pk}\), node \(j\)
        evaluates the inference circuit on its private witness
        \(\bigl(\mathbf{W}_{j}^{t},\mathcal{D}_{j}^{t,\mathrm{test}}\bigr)\) and
        produces the noninteractive Groth16 proof \(\Pi_{j}^{t}\). It then
        assembles the proof bundle
        \[
          Z_{j}^{t}
            = \bigl(
                \Pi_{j}^{t},\,
                \mathcal{Y}_{j}^{t},\,
                \mathcal{L}_{j}^{t},\,
                C_{j}^{t,\mathrm{model}},\,
                C_{j}^{t,\mathrm{test}}
              \bigr)
        \]
        and buffers it locally until submission.

  \item \textbf{Block construction.}
        Once \(\Pi_{j}^{t}\) is ready, node \(j\) forms the block
        \[
          D_{j}(t)
            = \bigl[\mathbf{W}_{j}^{t},\,Z_{j}^{t}\bigr]
        \]
        for insertion into the DAG.

\end{enumerate}

\subsubsection{Stage 2: Block Submission}
\label{subsec:stage2}
\begin{enumerate}[label=\alph*.]

  \item \textbf{Fetch tips.}
        Node \(j\) reads the latest approved events to reconstruct the current tip set
        \(\mathcal{D}_j^{T}(t)\) and pulls each tip’s commitments
        \(C_{j'}^{t',\mathrm{model}}\), \(C_{j'}^{t',\mathrm{test}}\)
        and the global verification key \(\mathsf{vk}\) from the sidechain.

  \item \textbf{Verification.}
        Node \(j\) invokes the validation smart contract \(S_j^1\), which, for
        each candidate block, performs
        \[
          \mathsf{Verify}\bigl(
            \mathsf{vk},\,
            \Pi_{j'}^{t'},\,
            C_{j'}^{t',\mathrm{model}},\,
            C_{j'}^{t',\mathrm{test}},\,
            \mathcal{Y}_{j'}^{t'},\,
            \mathcal{L}_{j'}^{t'}
          \bigr)
        \]
        and returns the top-\(K_V\) verified parents
        \(\mathcal{D}_j^{V}(t)\).

  \item \textbf{Publish block.}
        With \(\mathcal{D}_j^{V}(t)\) confirmed, node \(j\) calls
        \(S_j^2\), which (i) records the chosen parents,
        (ii) embeds \(D_j(t)\) into the DAG, and (iii) marks it as a new tip.

\end{enumerate}

\subsubsection{Stage 3: Consensus-Driven Confirmation and Challenge}
\label{subsec:stage3}
\begin{enumerate}[label=\alph*.]

  \item Node \(j\) syncs with the sidechain, incorporating new blocks
        \(D_{j'}(t)\) whose approved events have appeared on chain.

  \item Upon synchronization, \(S_j^2\) advances to its second phase:
        it updates block statuses and records which blocks transition
        from unconfirmed to confirmed in epoch \(t\).

  \item Node \(j\) executes graph reachability analysis (GRA) on the DAG.

  \item Based on the GRA output, node \(j\) may trigger the challenge
        contract \(S_j^3\) to dispute the validity of certain blocks.

  \item If triggered, \(S_j^3\) stakes tokens from node \(j\) and notifies
        the oracles to adjudicate the dispute.

  \item Depending on the oracles’ consensus (cf.\ \cref{subsection:challenge}),
        either the DAG is updated or node \(j\)’s stake is slashed.

\end{enumerate}

\subsubsection{Stage 4: Global Model Aggregation}
\label{subsec:stage4}
\begin{enumerate}[label=\alph*.]

  \item Node \(j\) calls the aggregation smart contract \(S_j^4\) with the list
        of confirmed blocks from epoch \(t\).

  \item \(S_j^4\) retrieves the corresponding models from the DAG and
        computes the new global model
        \(\tilde{\mathbf{W}}_{j}^{\,t}\), storing it on the sidechain.

  \item Finally, node \(j\) invokes the reward contract \(S_j^5\),
        which disburses tokens to node \(j\) based on its contributed
        loss \(\mathcal{L}_{j}^{t}\) (cf.\ \cref{SC5}).

\end{enumerate}

At this point, node \(j\) holds the updated global model
\(\tilde{\mathbf{W}}_{j}^{\,t}\) and begins the next epoch’s local
training, returning to Stage~1.

\begin{remark}
The four-stage protocol in \cref{subsection:workflow} remains intact; we simply augment Stages~1 and~2 to consume the committee-published scalars \([L_t,B_t]\) and \(\tau_{\max}\). At the boundary between epochs \(t-1\) and \(t\), each node \(j\) retrieves the signed events \texttt{NormThresholdsPublished} and \texttt{CosineThresholdsPublished} (see \cref{sec:thresholds}) and caches the new bounds before beginning Stage~1. After local training and KZG commitments, Stage~1 proof generation now emits three noninteractive proofs: \(\Pi_{j}^{t}\) (Groth16 for correct inference and loss), \(\Sigma_{j}^{t}\) (Bulletproof enforcing \(L_t \le \|\Delta\mathbf{W}_j^t\|_2 \le B_t\)), and \(\Gamma_{j}^{t}\) (Groth16 subproof enforcing \(\cos(z_{\mathrm{old}},z_{\mathrm{new}}) \le \tau_{\max}\) on the fixed probe set). These are concatenated into the extended bundle \(Z_{j}^{t,\mathrm{ext}}\). Stage~2's validation contract \(S_j^1\) then runs the three-step verification (Groth16 \(\Pi_{j}^{t}\), Bulletproof \(\Sigma_{j}^{t}\), cosine subproof \(\Gamma_{j}^{t}\)) as in \cref{sec:validation-contract-extended} and returns the top-\(K_V\) parents exactly as before. Stages~3 and~4 continue to react only to the usual \textsf{ProofOK} events and require no modifications.
\end{remark}

\begin{remark}
In ZK-HybridFL, the network can be effectively divided into two types of nodes—full nodes and light nodes—based on their resource capabilities. This division is motivated by the need to accommodate diverse participant environments, optimize system performance, and ensure scalability. Full nodes, with ample computational power and storage, maintain the complete global DAG ledger, comprehensive sidechains, and carry out resource-intensive tasks, including oracle functions. These nodes handle heavy operations such as ledger maintenance, smart contract execution, and dynamic event validation, which are crucial for the integrity and reliability of the overall system.

In contrast, light nodes operate with lighter versions of the DAG and sidechain, focusing on essential local training and basic participation in the FL process. By delegating complex, resource-heavy tasks to full nodes, light nodes can efficiently contribute to model updates and generate proofs without the burden of maintaining an entire global state. This stratification not only leverages the strengths of more capable participants but also enables a wide range of devices with limited resources to join the network. The oracle functionality, consistent across both node types, serves as a trusted intermediary, validating events and bridging interactions between light and full nodes. \Cref{full_light} illustrates the architectural division between full nodes and light nodes in ZK-HybridFL. The overall interaction between the trainers, the Oracle Committee, the sidechain smart contracts, and the DAG ledger is illustrated in \Cref{fig:workflow}, including both the main training and challenge phases.
\end{remark}

\begin{figure}[t]
    \centering
    \includegraphics[width=\columnwidth]{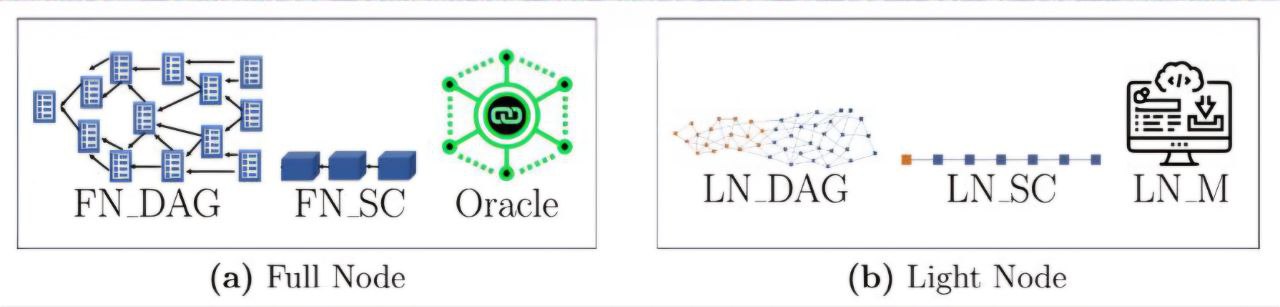}
    \caption{Full nodes (FNs) maintain a complete DAG (FN\_DAG) and sidechain (FN\_SC) while handling oracle functions. Light nodes (LNs) use a trimmed DAG (LN\_DAG) and lightweight sidechain (LN\_SC), with LN\_M serving as the LN's core controlling module.}
    \label{full_light}
\end{figure}

\begin{figure}[t]
    \centering
    \includegraphics[width=\columnwidth]{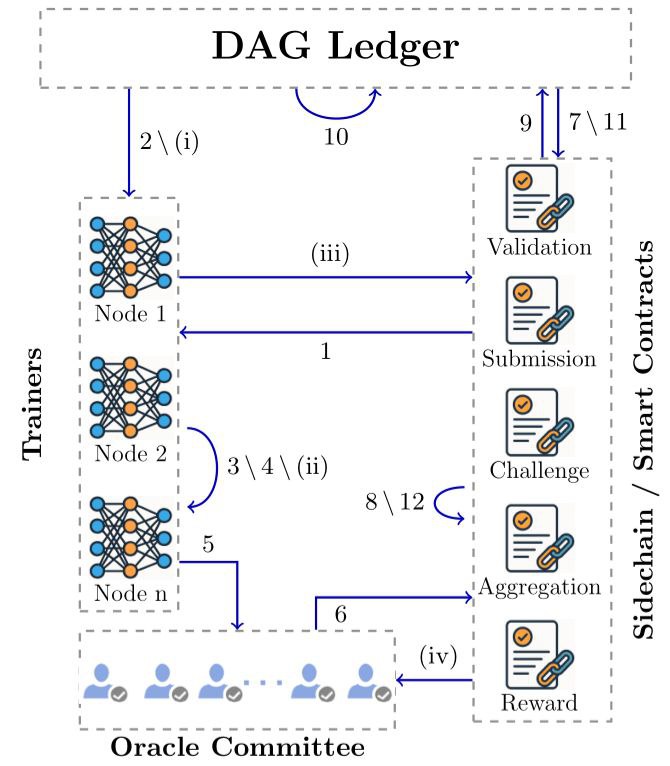}
    \caption{ZK-HybridFL workflow (1–12): 
      1) aggregate IDs; 2) fetch blocks; 3) local train; 4) bundle proofs; 
      5) submit bundle; 6) committee admit; 7) fetch ZKP tips; 8) validate ZKPs; 
      9) parent selection; 10) attach block; 11) update weights; 12) aggregate and reward. 
      Challenge loop (i–iv): (i) fetch proof; (ii) local GRA; 
      (iii) submit proof; (iv) reward or slash.}
    \label{fig:workflow}
\end{figure}

\section{Simulation Results}
\label{sec:simulations}
In this section, we describe the simulation setup and present experimental results comparing ZK-HybridFL with Blade-FL and ChainFL. A comprehensive analysis between ZK-HybridFL and ChainFL from both learning and distributed-ledger perspectives is provided in the supplementary material (\cref{subsec:comparison_chainfl}).

\subsection{Simulation Setup}
To simulate Blade-FL and ChainFL, we use their public implementations from the respective GitHub repositories~\cite{gggg1,gggg2}, with minor modifications to match our experimental protocol. For ZK-HybridFL, the federated learning (FL) process is built using TensorFlow Federated (TFF)~\cite{gggg3} to simulate local model training across distributed nodes. Training with a mini-batch of size \(B\) consists of \(R\) stochastic gradient descent (SGD) iterations per FL epoch, inducing heterogeneous compute loads across nodes.

For decentralized coordination, the blockchain architecture integrates a directed acyclic graph (DAG) ledger, simulated using the GoShimmer framework~\cite{gggg4}, to manage parallelized model-update submissions. A Substrate-based sidechain~\cite{gggg5} hosts the core smart contracts (cf.\ \cref{sec:edsc}) for ZKP verification, model aggregation, challenge resolution, and reward distribution. WebAssembly (Wasm)-based~\cite{gggg6} contracts ensure low-overhead execution of this logic. DAG–sidechain communication is implemented via gRPC~\cite{gggg7}, and Apache Kafka~\cite{gggg8} serves as a message broker for the event-driven architecture. ZKPs are generated using EZKL~\cite{gggg9} and zkML~\cite{gggg10}.

We denote \(n\) as the total number of nodes, \(\gamma\) as the percentage of lazy nodes, and \(\mu\) as the percentage of adversarial nodes. The corresponding counts are \([\!n\gamma]\) and \([\!n\mu]\), respectively, where \([\cdot]\) denotes rounding. Lazy nodes in all three schemes simply resubmit a model update from one or more past epochs instead of retraining.

Adversarial nodes are modeled as follows. In Blade-FL, the adversarial set collectively controls \(51\%\) of the proof-of-work (PoW) hashing power, dominating block creation and ensuring that their noisy updates are predominantly used in global aggregation. The parameter \(\mu\) controls \emph{how many} nodes are adversarial, but the group as a whole always holds \(51\%\) of the hash power. Thus, smaller \(\mu\) yields fewer but more powerful adversarial nodes, whereas larger \(\mu\) spreads the same total PoW across more attackers. In ChainFL and ZK-HybridFL, adversaries (i) add structured noise to their locally trained updates and (ii) conduct orphanage attacks by repeatedly selecting their own or colluding blocks as parents in the DAG, thereby boosting the aggregated weight of adversarial blocks.

Unless otherwise stated, the mini-batch size is \(B=50\), the number of local SGD iterations per epoch is \(R=5\) with learning rate \(\eta=0.01\), and each submitted block in ZK-HybridFL selects \(K_{V}=4\) parent blocks.

To emulate realistic network conditions, nodes are assigned heterogeneous bandwidths in the range \(\SI{10}{Mbps}\)–\(\SI{50}{Mbps}\) and latencies in the range \(\SI{50}{\milli\second}\)–\(\SI{200}{\milli\second}\). These parameters are intrinsic inputs to the simulator and are chosen to reflect typical wide-area settings.

\subsection{Learning Tasks}
\subsection*{Task~1: Image Classification}
Task~1 uses the MNIST dataset with \(70{,}000\) grayscale images of size \(28 \times 28\) across ten classes (digits \(0\)–\(9\)). A lightweight convolutional neural network (CNN) based on MobileNetV2~\cite{gggg12} is used. The evaluation metric is classification accuracy, i.e., the fraction of correctly classified test samples.

For Blade-FL and ChainFL, \(10{,}000\) images are reserved as a public validation dataset, and the remaining \(60{,}000\) images are evenly distributed among the \(n\) nodes. Each node splits its local data into \(80\%\) for training and \(20\%\) for testing. In ZK-HybridFL, there is no public validation dataset; instead, all \(70{,}000\) images are evenly split across nodes, again with an \(80\%\)–\(20\%\) train–test split, and \(20\%\) of each node’s local test data is further designated as a private inference batch for ZKP generation. All reported test accuracies are computed on the aggregated local test sets of each scheme; the public datasets (Blade-FL, ChainFL) and private inference batches (ZK-HybridFL) are used only for validation, not for final evaluation.

\subsection*{Task~2: Text Sentiment Analysis}
Task~2 is a next-word prediction task using a gated recurrent unit (GRU)-based language model trained on the Penn Treebank dataset with \(345{,}526\) tokens. The model outputs a probability distribution over the vocabulary for each next word. Performance is evaluated via perplexity,
\[
  \mathrm{Perplexity}
  = \exp\!\Bigl(
    -\tfrac{1}{N}\sum_{i=1}^{N} \log p_{y_i}
  \Bigr),
\]
where \(p_{y_i}\) is the predicted probability of the true next word \(y_i\) at position \(i\), and \(N\) is the total number of predictions. Lower perplexity indicates better predictions.

For Blade-FL and ChainFL, \(45{,}526\) tokens are reserved as a public validation dataset, and the remaining \(300{,}000\) tokens are evenly distributed among the \(n\) nodes. Each node uses \(80\%\) of its tokens for local training and \(20\%\) for local testing. In ZK-HybridFL, each node additionally marks \(20\%\) of its local test subset as a private inference batch for ZKP generation.

\subsection{Results and Analysis}
\label{sec:results}

\subsubsection{Learning Perspectives}

\subsubsection*{\textbf{FL training convergence}}
\label{subsec:simulation_learning}
\Cref{training_loss} shows the training-loss trajectories of Blade-FL, ChainFL, and ZK-HybridFL in a network with \(n=15\) nodes, \(\mu=20\%\) adversarial nodes, and \(\gamma=10\%\) lazy nodes. Blade-FL achieves moderate loss reduction in early epochs but ultimately fails to converge. Once adversarial nodes effectively control \(51\%\) of the PoW computation, they dominate block creation, and their noisy updates increasingly contaminate the global model, preventing convergence to a low-loss solution.

ChainFL converges better than Blade-FL due to its DAG-based and sharded design. However, it remains inferior to ZK-HybridFL. As discussed from the ledger perspective in \cref{subsubsec:ledger_perspective}, ChainFL’s dependence on a public reference dataset for validation makes it vulnerable to lazy nodes that resubmit stale updates which still pass the validation threshold. Furthermore, adversarial nodes can influence tip-based selection on the DAG and introduce subtle corruptions into the global model, preventing convergence to very low loss.

ZK-HybridFL, by contrast, leverages ZKPs and a loss-aware DAG policy (cf.\ \cref{subsec:dag}) to ensure that only freshly trained, correctly validated updates contribute to aggregation. Even in the presence of adversarial and lazy behavior, invalid or stale updates are systematically filtered. This is reflected in the consistently decreasing loss curve of ZK-HybridFL, which indicates steady progress toward a high-quality model.

\subsubsection*{\textbf{Invalid model detection}}
\Cref{detection} reports the number of invalid models detected over time for the same setting as \cref{training_loss}. A model is considered invalid if it fails to satisfy the scheme-specific inclusion criteria for global aggregation. In Blade-FL, an update is invalid if, after being broadcast and signed, it is never included in a mined block. In ChainFL, an update is invalid if it is not selected as a DAG parent within a prescribed staleness window, rendering it permanently ineligible for future parent selection. In ZK-HybridFL, a model is invalid if it is revoked via the challenge mechanism (cf.\ \cref{subsection:challenge}).

The horizontal reference line in \cref{detection} corresponds to an ideal system that perfectly detects all invalid models. Blade-FL performs worst: many invalid updates remain undetected or are detected late. ChainFL performs better but still lags significantly behind ZK-HybridFL.

ZK-HybridFL initially flags fewer invalid models than the other two schemes, because a formal revocation requires completion of the challenge procedure and oracle adjudication. Thus, a model may be added in one epoch and only revoked several epochs later. Over time, however, ZK-HybridFL “catches up,” and the cumulative number of revoked models approaches the ideal line. Combined with \cref{training_loss}, this behavior confirms the learning-theoretic analysis in \cref{subsubsec:learning_perspective}: by ensuring that only valid, high-quality updates influence the global model, ZK-HybridFL achieves both faster convergence and stronger robustness in the presence of adversarial and lazy nodes.

\begin{figure}[!t]
    \centering
    \subfloat[Task~1]{%
        \includegraphics[width=\columnwidth]{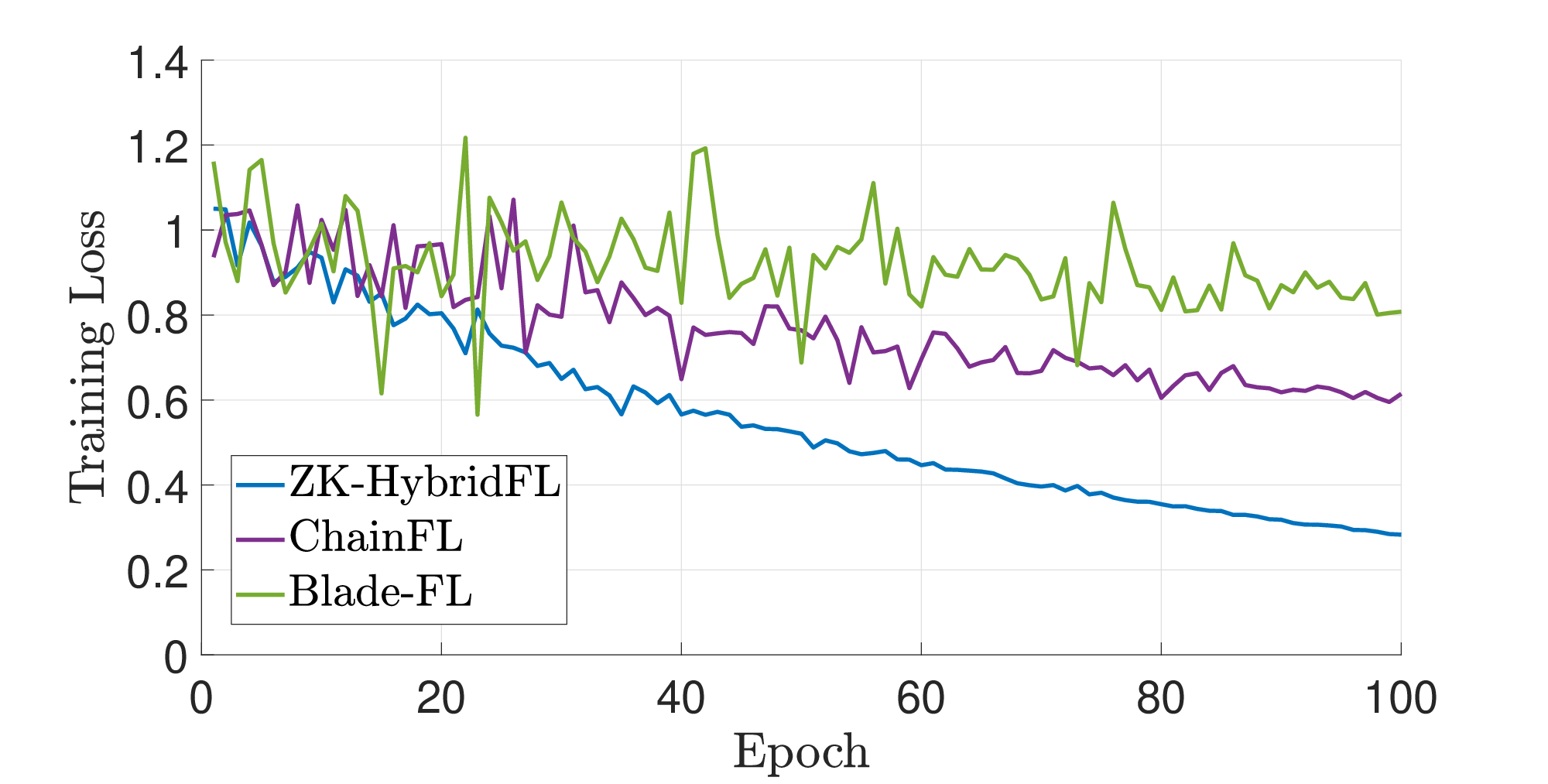}%
        \label{fff_1}%
    }\\[-0.25\baselineskip]
    \subfloat[Task~2]{%
        \includegraphics[width=\columnwidth]{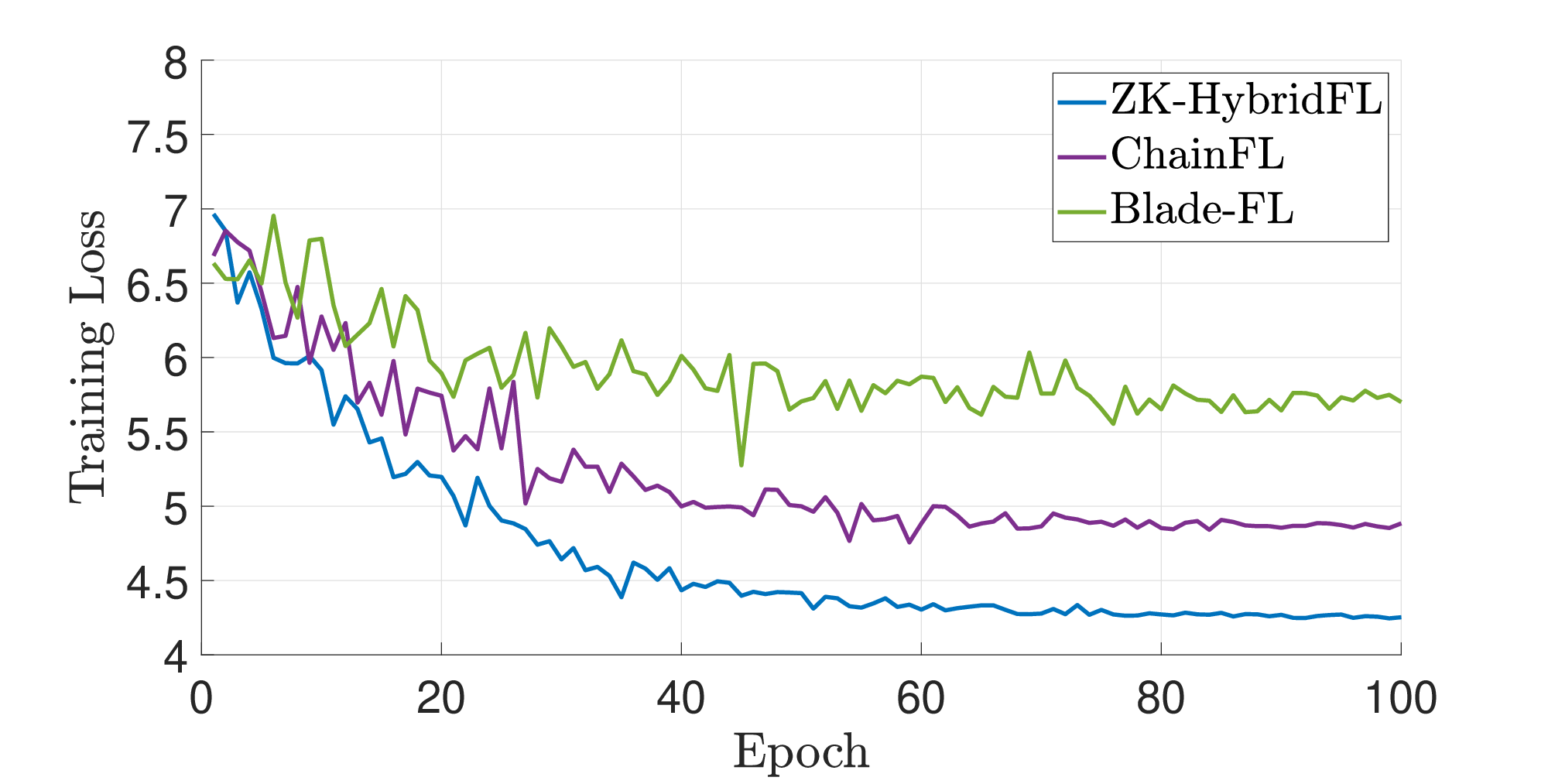}%
        \label{fff_2}%
    }
    \caption{Training loss of Blade-FL, ChainFL, and ZK-HybridFL for \(n=15\), \(\mu=20\%\) adversaries, \(\gamma=10\%\) lazy nodes, \(R=5\), and \(B=50\).}
    \label{training_loss}
\end{figure}

\begin{figure}[!t]
    \centering
    \subfloat[Task~1]{%
        \includegraphics[width=\columnwidth]{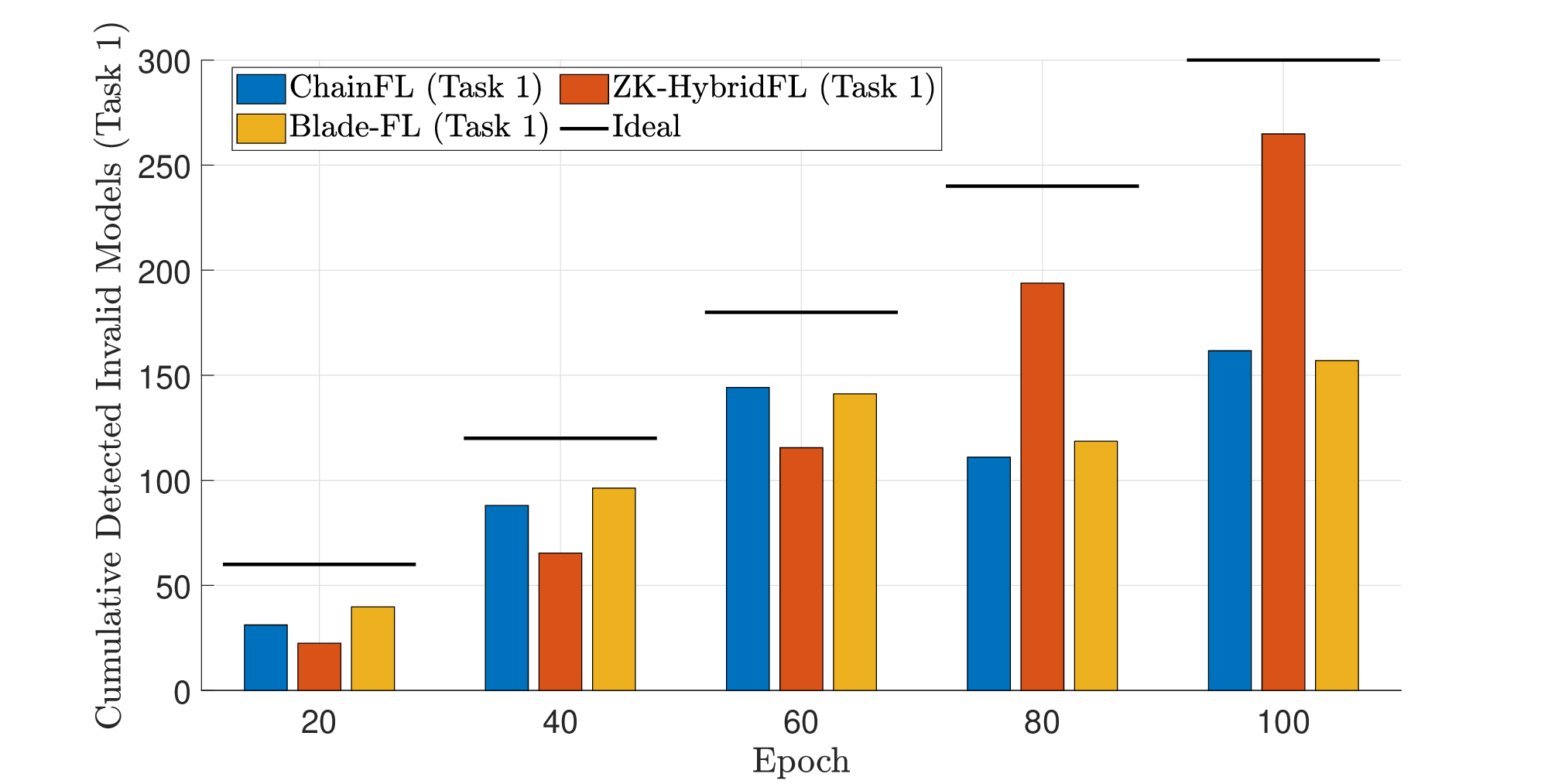}%
        \label{fff_3_1}%
    }\\[-0.25\baselineskip]
    \subfloat[Task~2]{%
        \includegraphics[width=\columnwidth]{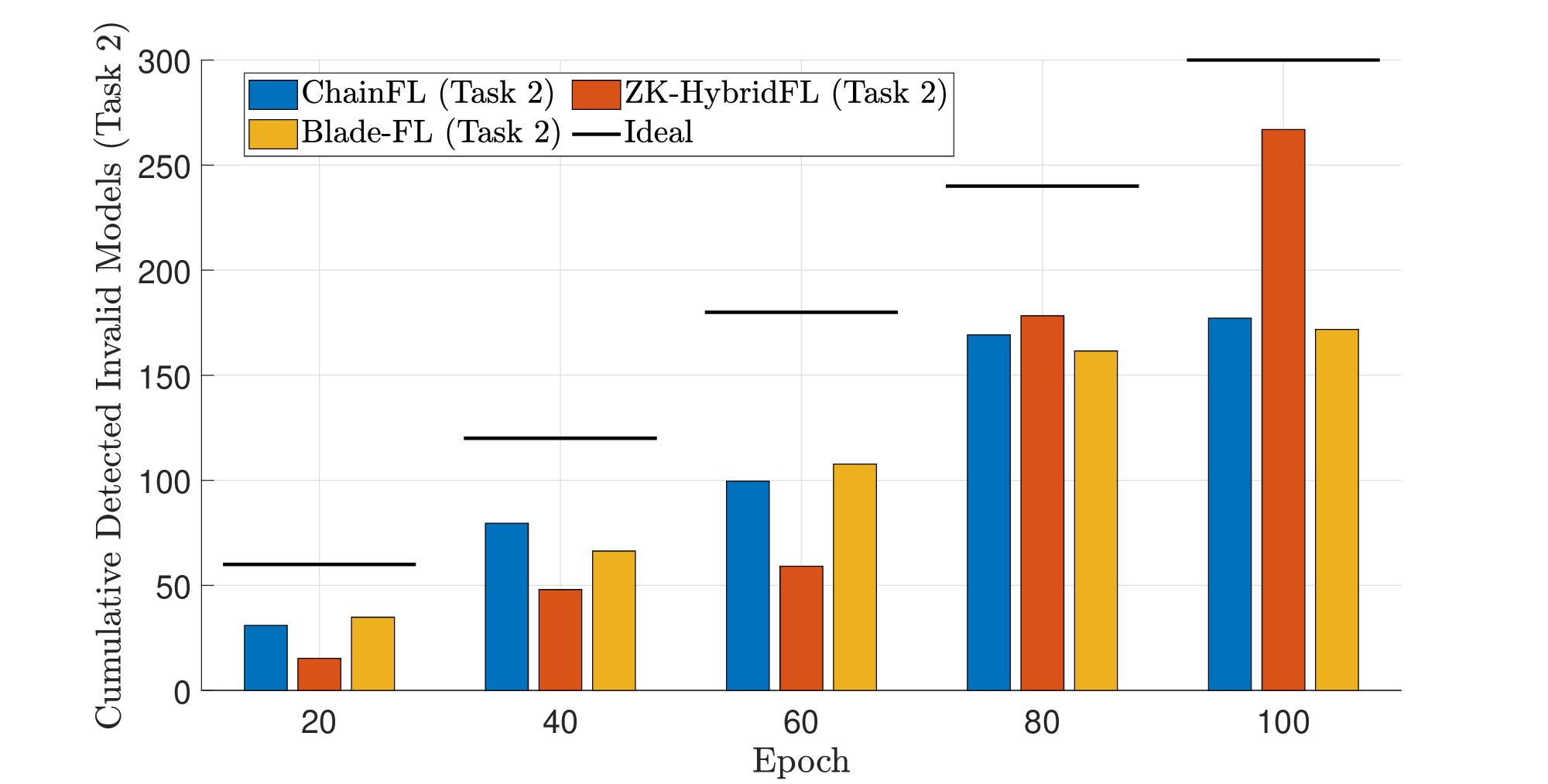}%
        \label{fff_3_2}%
    }
    \caption{Number of detected invalid models over training epochs, corresponding to \cref{training_loss}.}
    \label{detection}
\end{figure}

\subsubsection*{\textbf{Model performance vs.\ number of nodes}}
\Cref{nodes_vary} plots test accuracy (Task~1) and perplexity (Task~2) as the number of nodes \(n\) ranges from \(5\) to \(30\), with \(\mu=\gamma=15\%\) fixed. For Task~1, ZK-HybridFL rapidly improves from accuracy \(0.80\) at \(n=5\) to \(0.90\) at \(n=10\) and \(0.95\) at \(n=15\), then saturates near \(0.98\)–\(0.99\) for \(n\ge 20\). This steep improvement demonstrates that the ZKP-gated validation pipeline successfully exploits larger networks: additional honest nodes yield more high-quality updates, while adversarial and lazy contributions are suppressed.

ChainFL’s accuracy increases more slowly, from \(0.55\) at \(n=5\) to \(0.70\) at \(n=30\). Its architecture partially dilutes malicious contributions, but the use of a public dataset for validation still allows stale or subtly corrupted updates to pass the threshold, limiting the benefits of larger network sizes.

Blade-FL suffers as the network grows: accuracy starts at \(0.45\) with \(5\) nodes and drops to \(0.33\) at \(30\) nodes, highlighting the fragility of PoW-based consensus under adversarial conditions.

Task~2 exhibits an analogous pattern. ZK-HybridFL consistently achieves the lowest perplexity, decreasing from \(145.73\) at \(n=5\) to \(117.67\) at \(n=30\). ChainFL improves from \(218.59\) to \(167.08\) over the same range, while Blade-FL degrades from \(284.17\) to \(406.89\). These results reinforce that only ZK-HybridFL fully capitalizes on larger networks by ensuring that the additional capacity is translated into higher-quality, validated updates.

\begin{figure}[!t]
    \centering
    \subfloat[Task~1]{%
        \includegraphics[width=\columnwidth]{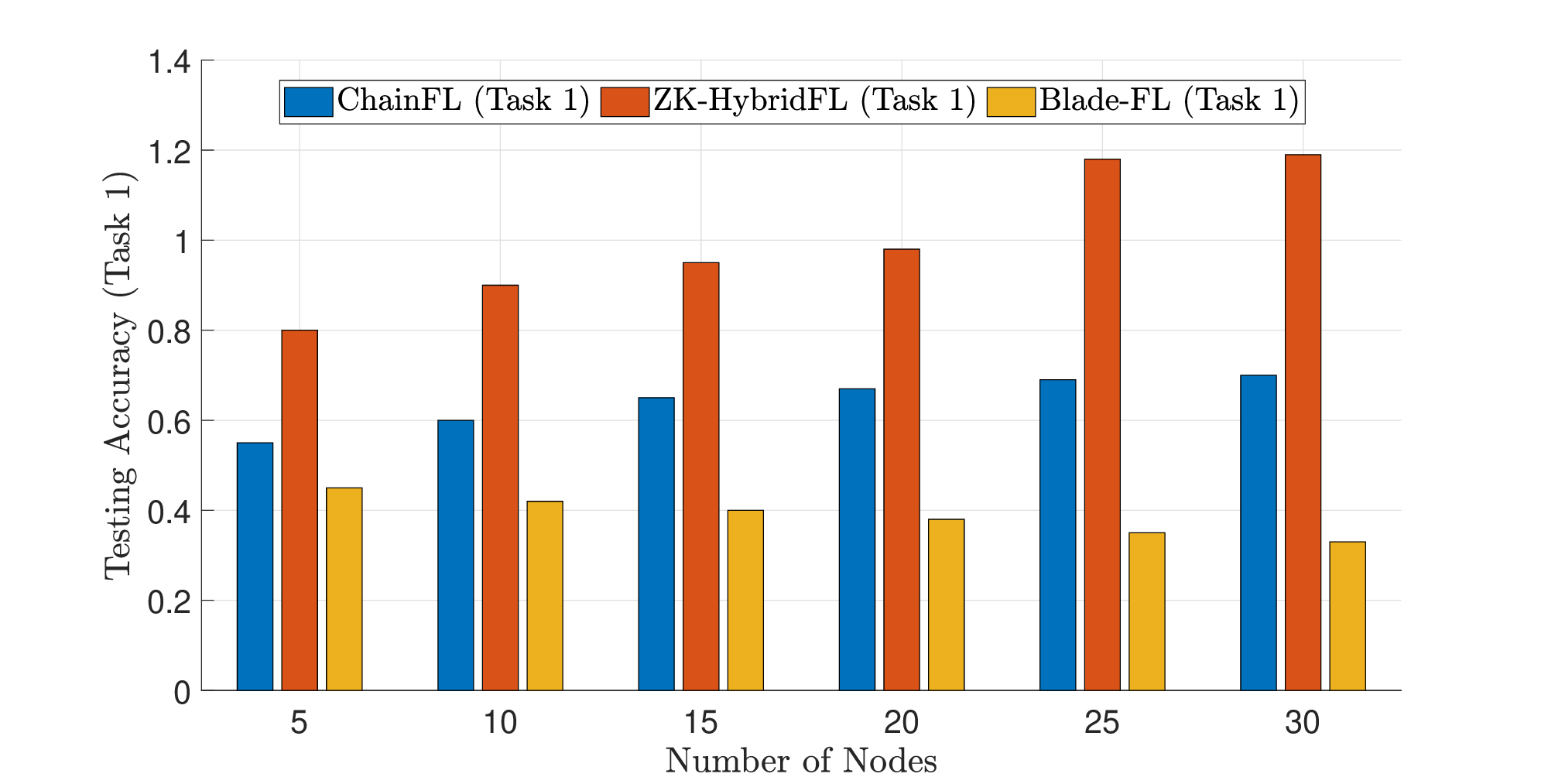}%
        \label{fff_4_1}%
    }\\[-0.25\baselineskip]
    \subfloat[Task~2]{%
        \includegraphics[width=\columnwidth]{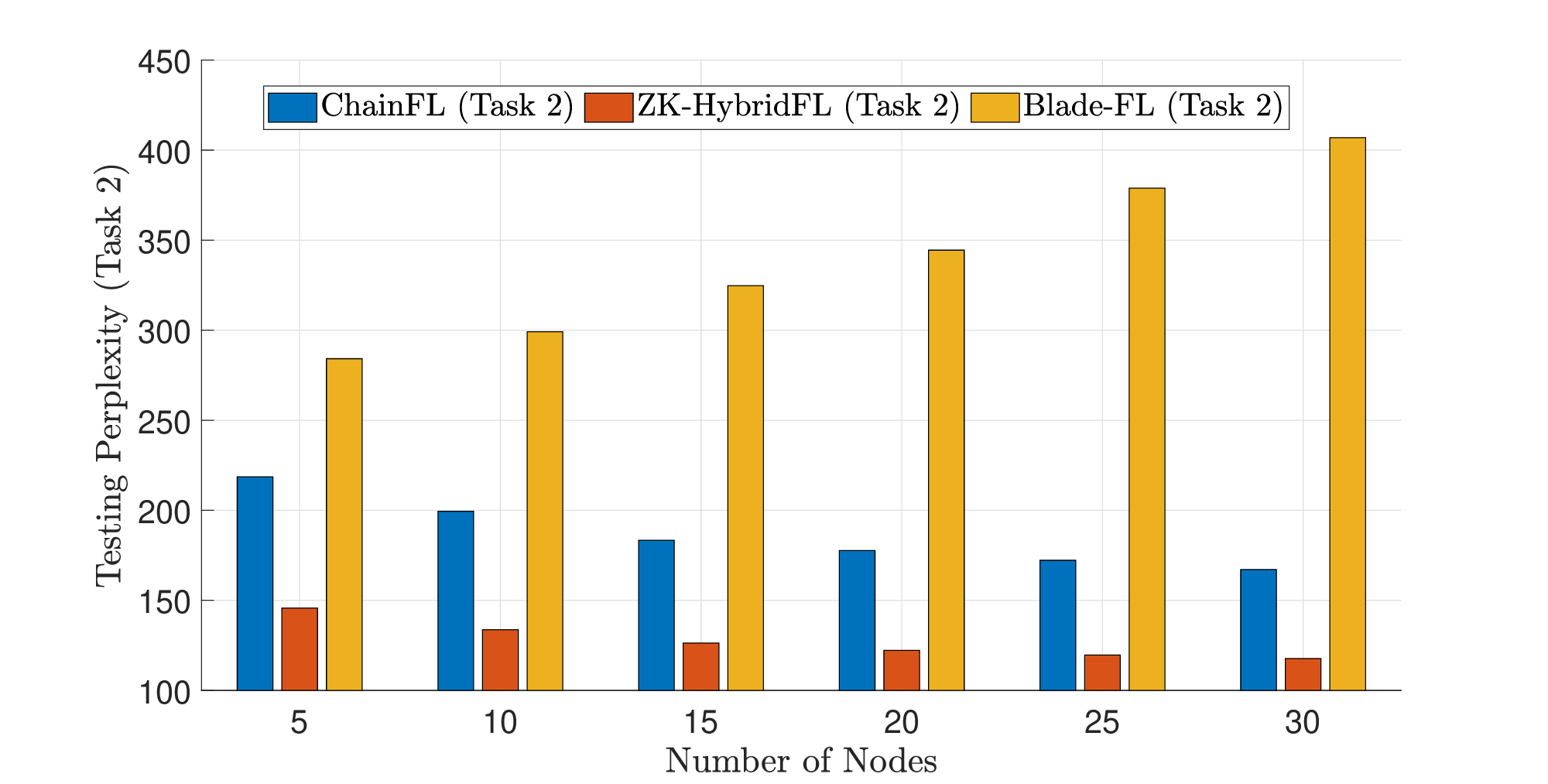}%
        \label{fff_4_2}%
    }
    \caption{Model performance versus number of nodes \(n\) for ZK-HybridFL, ChainFL, and Blade-FL with \(\mu=\gamma=15\%\).}
    \label{nodes_vary}
\end{figure}

\subsubsection*{\textbf{Model performance vs.\ percentage of adversarial nodes}}
\Cref{adversary_vary} shows test accuracy (Task~1) and perplexity (Task~2) as the adversarial ratio \(\mu\) increases from \(5\%\) to \(30\%\), with \(n=15\) and \(\gamma=15\%\) fixed. For Task~1, ChainFL’s accuracy drops from \(0.70\) at \(\mu=5\%\) to \(0.50\) at \(\mu=30\%\), reflecting its vulnerability to adversaries that exploit the public validation dataset and tip-based aggregation. Blade-FL starts at \(0.50\) and degrades sharply to \(0.30\), highlighting the susceptibility of PoW-based consensus to majority-hash attacks.

ZK-HybridFL exhibits strong robustness, maintaining accuracy near \(0.99\) at \(\mu=5\%\) and still achieving \(0.88\) at \(\mu=30\%\). The ZKP-gated inference validation and loss-aware DAG aggregation jointly prevent adversarial nodes from quietly introducing corrupted updates into the global model.

For Task~2, the same qualitative trends hold. ZK-HybridFL retains low perplexity, increasing only from \(117.18\) to \(120.92\) as \(\mu\) grows from \(5\%\) to \(30\%\). ChainFL’s perplexity worsens from \(174.87\) to \(218.59\), and Blade-FL’s from \(270.64\) to \(378.89\). These results underscore the benefit of ZK-HybridFL’s dual-layer validation over the mechanisms employed in Blade-FL and ChainFL.

\begin{figure}[!t]
    \centering
    \subfloat[Task~1]{%
        \includegraphics[width=\columnwidth]{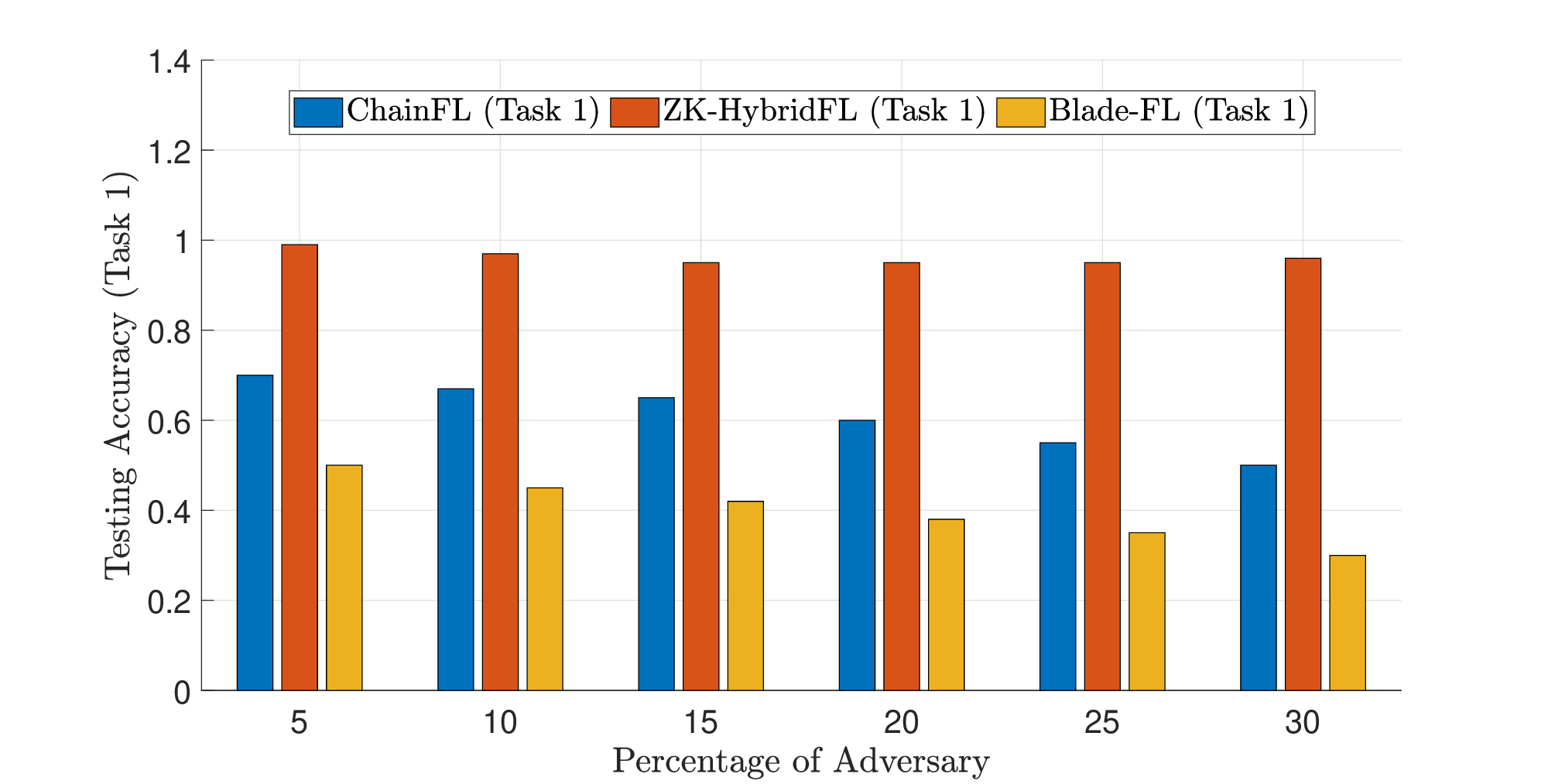}%
        \label{fff_5_1}%
    }\\[-0.25\baselineskip]
    \subfloat[Task~2]{%
        \includegraphics[width=\columnwidth]{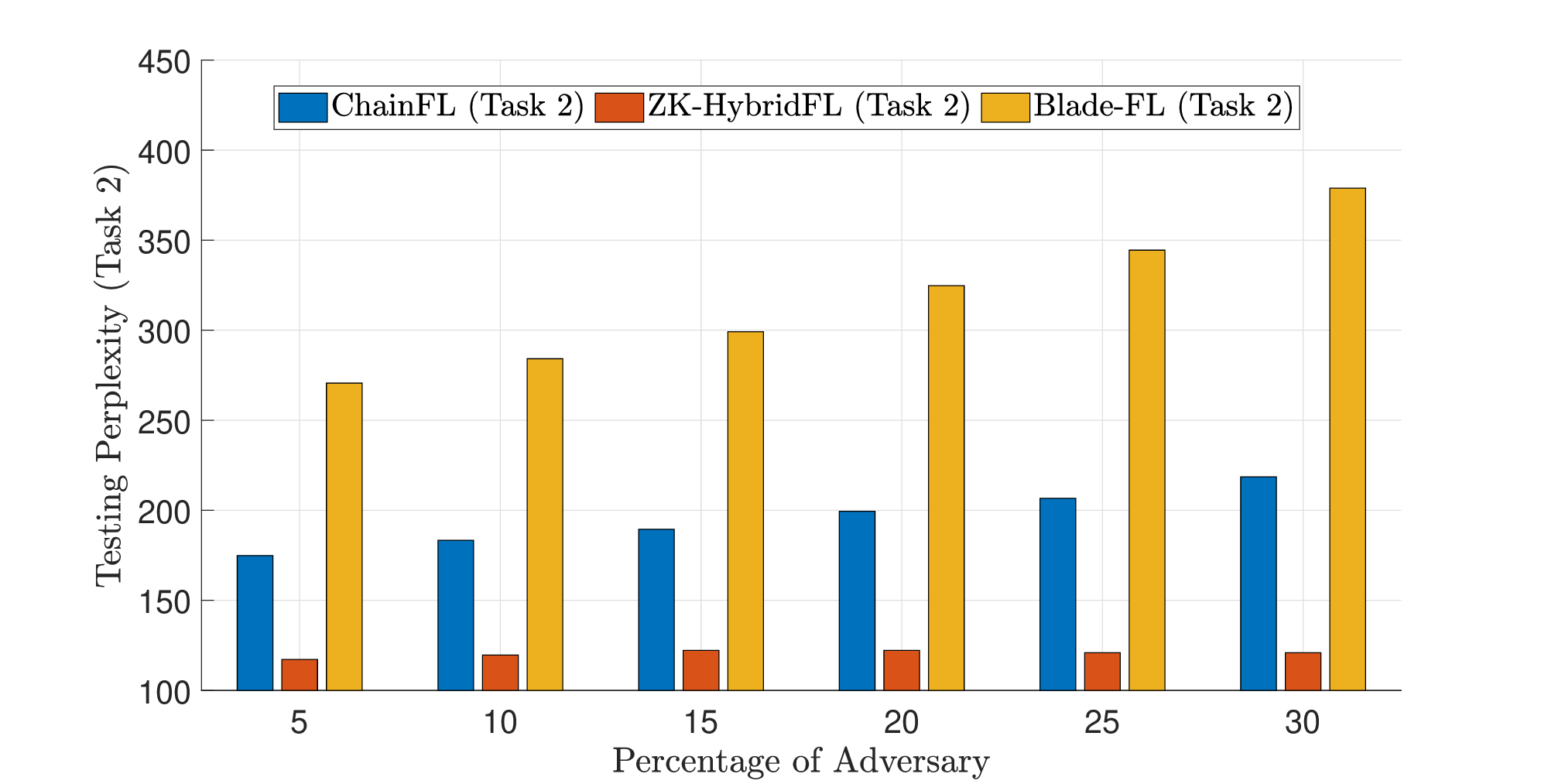}%
        \label{fff_5_2}%
    }
    \caption{Model performance versus adversarial ratio \(\mu\) for ZK-HybridFL, ChainFL, and Blade-FL with \(n=15\) and \(\gamma=15\%\).}
    \label{adversary_vary}
\end{figure}

\subsubsection*{\textbf{Model performance vs.\ percentage of lazy nodes}}
\Cref{lazy_vary} reports test accuracy (Task~1) and perplexity (Task~2) as the lazy-node ratio \(\gamma\) increases from \(5\%\) to \(30\%\), with \(n=15\) and \(\mu=15\%\) fixed. Lazy nodes resubmit past models without retraining, polluting aggregation with stale information.

For Task~1, ChainFL’s accuracy decreases from \(0.73\) to \(0.60\), and Blade-FL’s from \(0.55\) to \(0.42\), as \(\gamma\) increases. ZK-HybridFL remains highly accurate, dropping only from \(0.99\) to \(0.94\), owing to its ability to detect and discard stale updates.

For Task~2, ChainFL’s perplexity increases from \(169.65\) to \(199.42\) as \(\gamma\) grows, since stale submissions can still satisfy the public validation threshold. Blade-FL, already stressed by adversarial behavior, degrades from \(241.85\) to \(307.21\). In contrast, ZK-HybridFL’s committed-hash comparisons on the sidechain identify duplicate or outdated submissions, keeping perplexity low, with only a mild increase from \(117.18\) to \(123.55\).

These experiments further confirm that ZK-HybridFL’s sidechain-based validation and challenge mechanisms effectively prevent stale and malicious updates from influencing the global model. Additional experiments, including detailed ledger latency/throughput scaling, stake-weight ablations, and extended robustness studies, are provided in \cref{very_tired}.

\begin{figure}[!t]
    \centering
    \subfloat[Task~1]{%
        \includegraphics[width=\columnwidth]{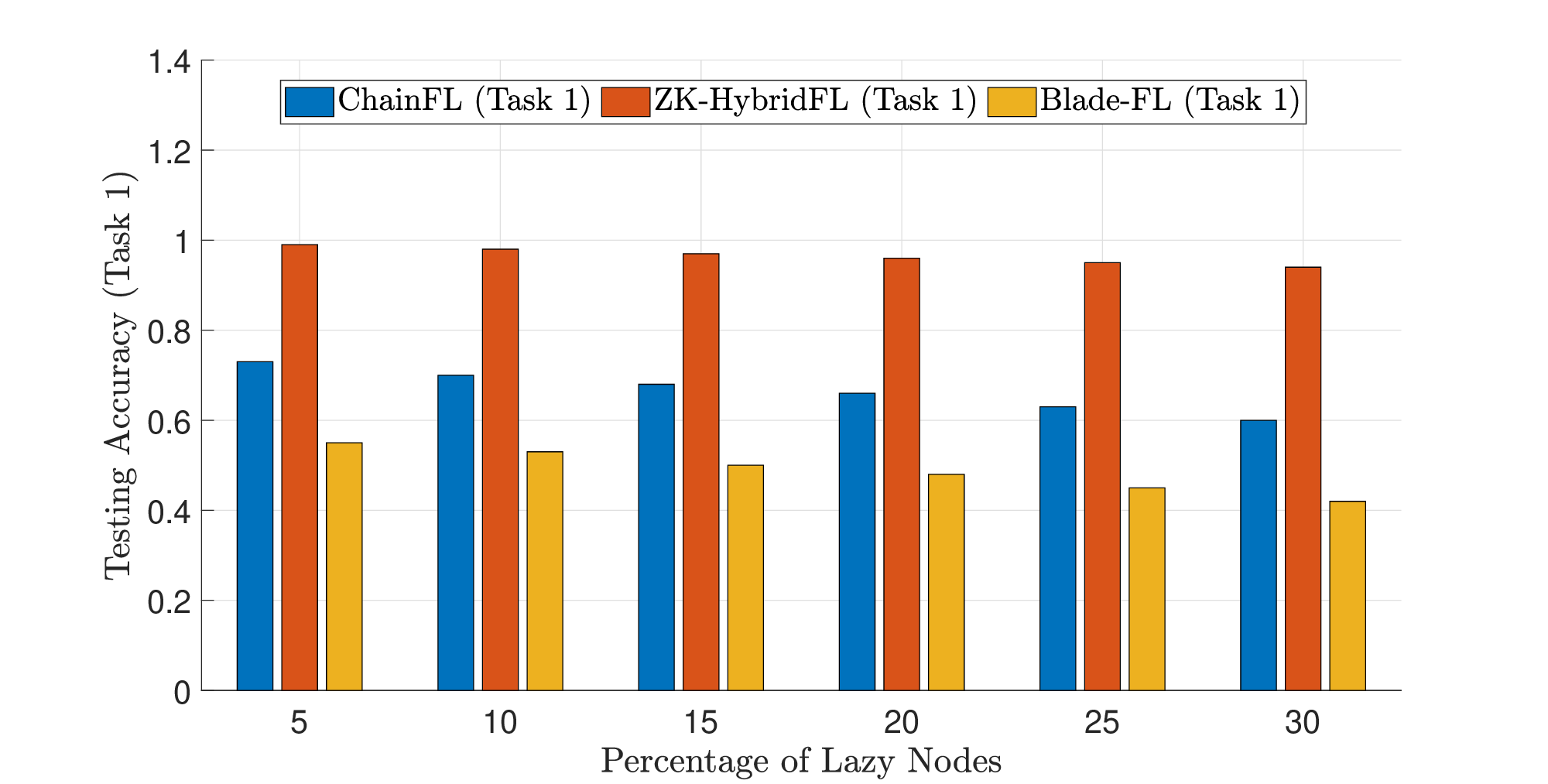}%
        \label{fff_6_1}%
    }\\[-0.25\baselineskip]
    \subfloat[Task~2]{%
        \includegraphics[width=\columnwidth]{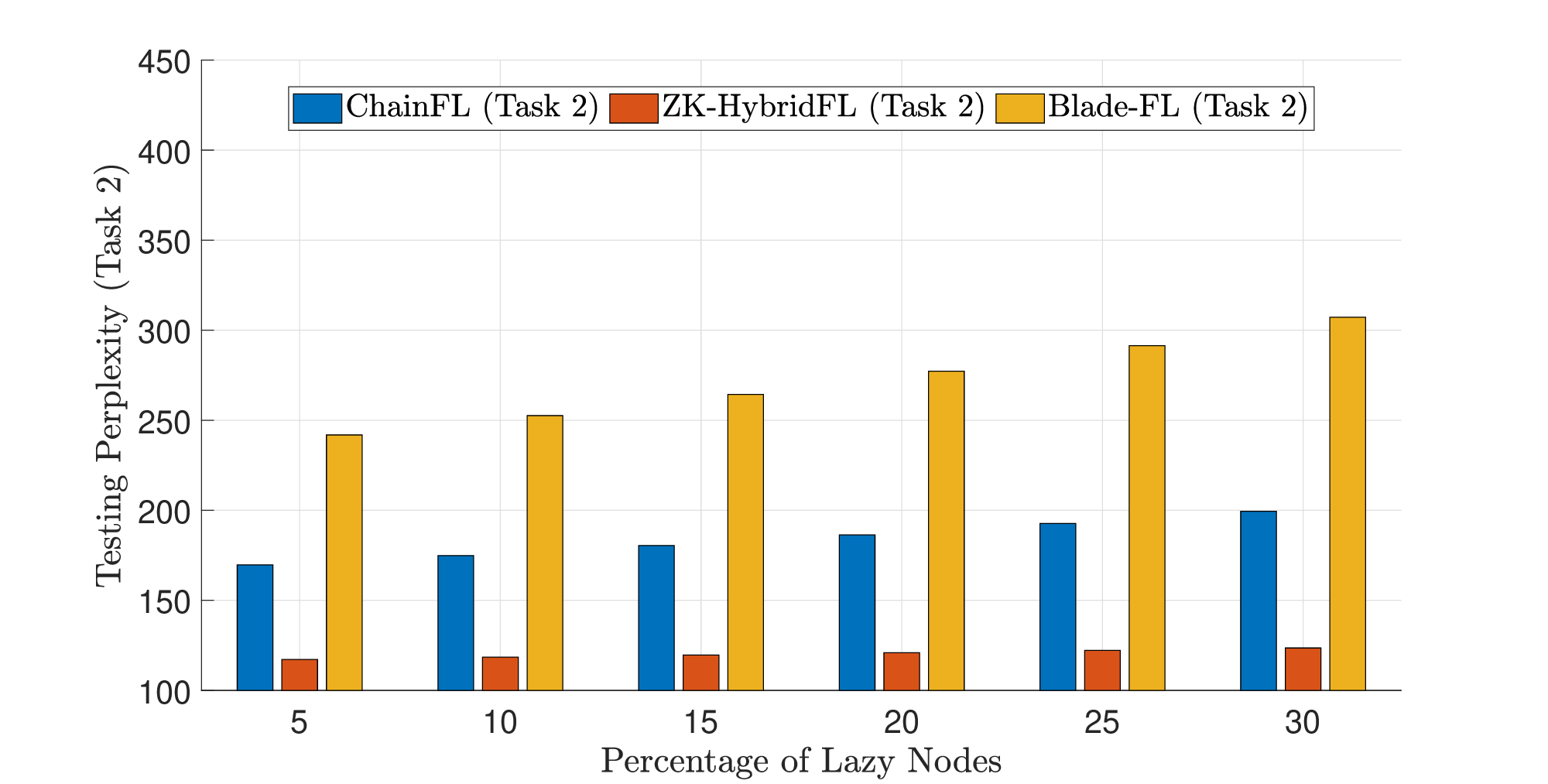}%
        \label{fff_6_2}%
    }
    \caption{Model performance versus lazy-node ratio \(\gamma\) for ZK-HybridFL, ChainFL, and Blade-FL with \(n=15\) and \(\mu=15\%\).}
    \label{lazy_vary}
\end{figure}

\subsubsection{Ledger Perspective}
\label{subsubsec:ledger_perspective}

\subsubsection*{\textbf{Latency and throughput}}
\label{tired_V1}
We measure \emph{latency} as the average time to complete one global update round, from the start of local training until the corresponding block is integrated into the ledger. \emph{Throughput} is the number of successful global update rounds completed by the network per minute; lower latency directly translates into higher throughput.

In ZK-HybridFL, latency is dominated by (i) local training and proof generation in \cref{subsec:stage1}, (ii) parent selection and DAG updates in \cref{subsec:stage2}, and (iii) occasional challenge resolution in \cref{subsec:stage3}. SNARK verification and sidechain contract execution are lightweight (\cref{sec:zkp-cost}), so proof checking adds only milliseconds per tip model. The predict–then–prove workflow overlaps most proof generation with subsequent training, further hiding ZKP costs. In ChainFL, latency is mainly driven by Raft-style consensus within each shard, whereas in Blade-FL it is dominated by the time spent solving proof-of-work (PoW) puzzles for block mining.

\Cref{latency,throughput} summarize latency and throughput as the network scales from \(5\) to \(30\) nodes under \(\mu=20\%\) adversaries and \(\gamma=10\%\) lazy nodes. For Task~1, ZK-HybridFL maintains the lowest latency across all sizes (roughly \(7.7\)–\(8.9\,\text{s}\)), compared with \(8.2\)–\(9.9\,\text{s}\) for ChainFL and \(9.1\)–\(10.2\,\text{s}\) for Blade-FL. Task~2, which uses a heavier GRU model, exhibits higher latencies overall, but ZK-HybridFL again remains fastest (about \(45.5\)–\(46.3\,\text{s}\)) versus \(51.4\)–\(52.3\,\text{s}\) for ChainFL and \(51.5\)–\(52.8\,\text{s}\) for Blade-FL.

These latency gains translate into higher throughput. As \Cref{throughput} shows, ZK-HybridFL sustains the largest number of completed update rounds per minute for both tasks. The sidechain-based event-driven contracts and low-cost ZKP verification allow the DAG to remain uncongested, whereas ChainFL suffers from cross-shard synchronization and Blade-FL from PoW overhead as the network grows.

\begin{figure}[!t]
    \centering
    \subfloat[Task~1]{%
        \includegraphics[width=\columnwidth]{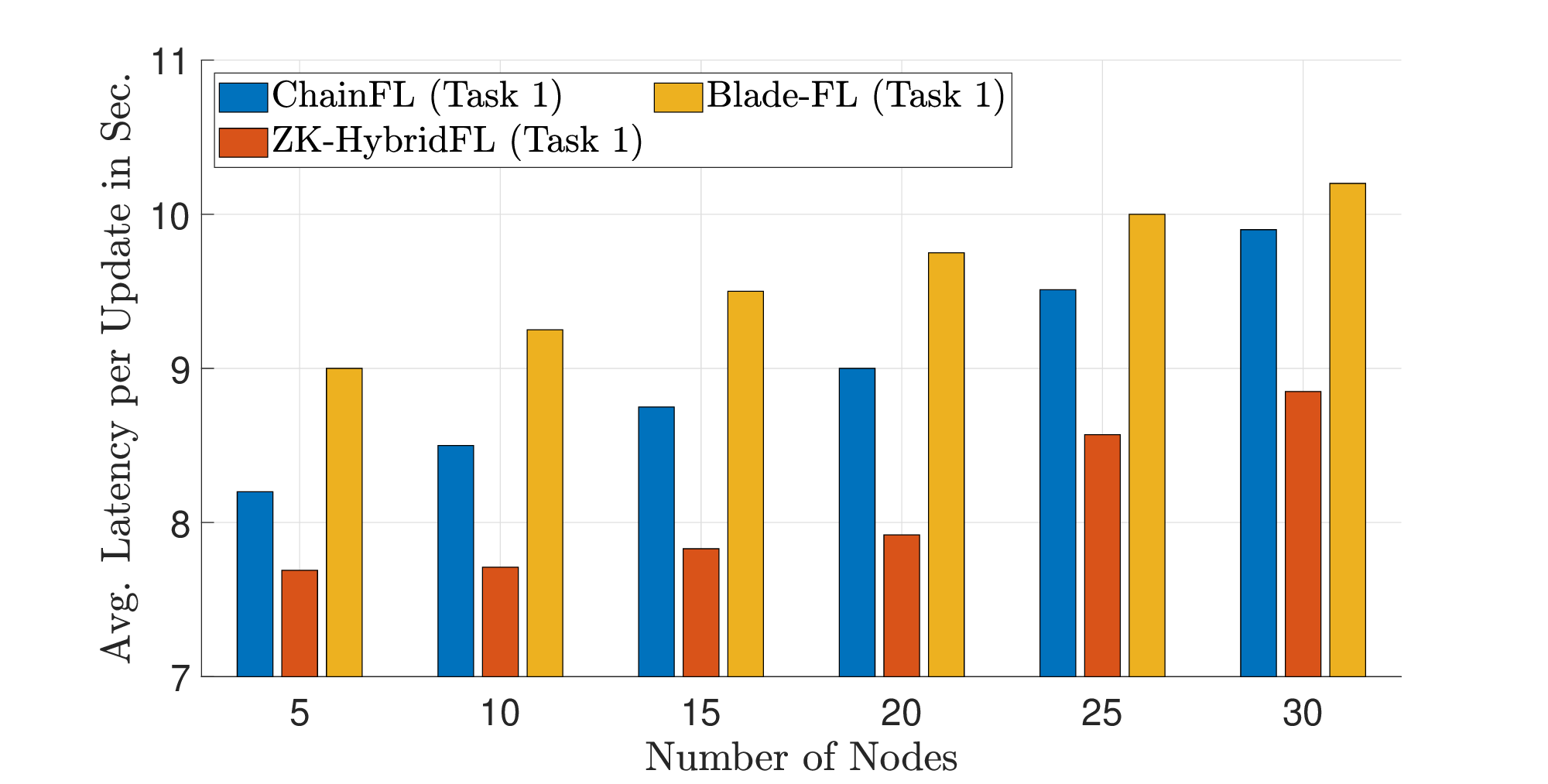}%
        \label{fff_7_1}%
    }\\[-0.25\baselineskip]
    \subfloat[Task~2]{%
        \includegraphics[width=\columnwidth]{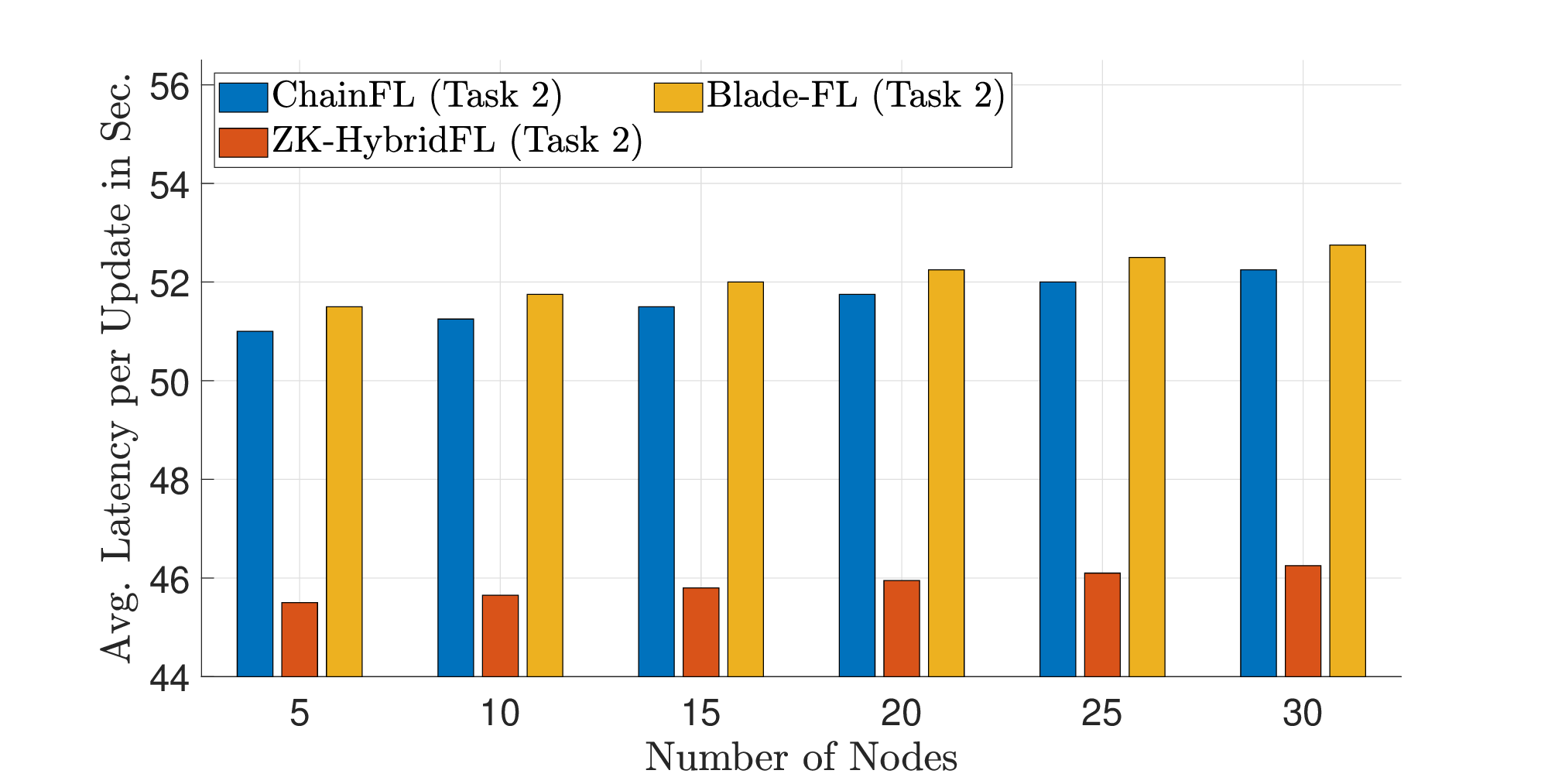}%
        \label{fff_7_2}%
    }
    \caption{Latency of Blade-FL, ChainFL, and ZK-HybridFL versus the number of nodes \(n\) with \(\mu=20\%\) adversarial nodes and \(\gamma=10\%\) lazy nodes.}
    \label{latency}
\end{figure}

\begin{figure}[!t]
    \centering
    \subfloat[Task~1]{%
        \includegraphics[width=\columnwidth]{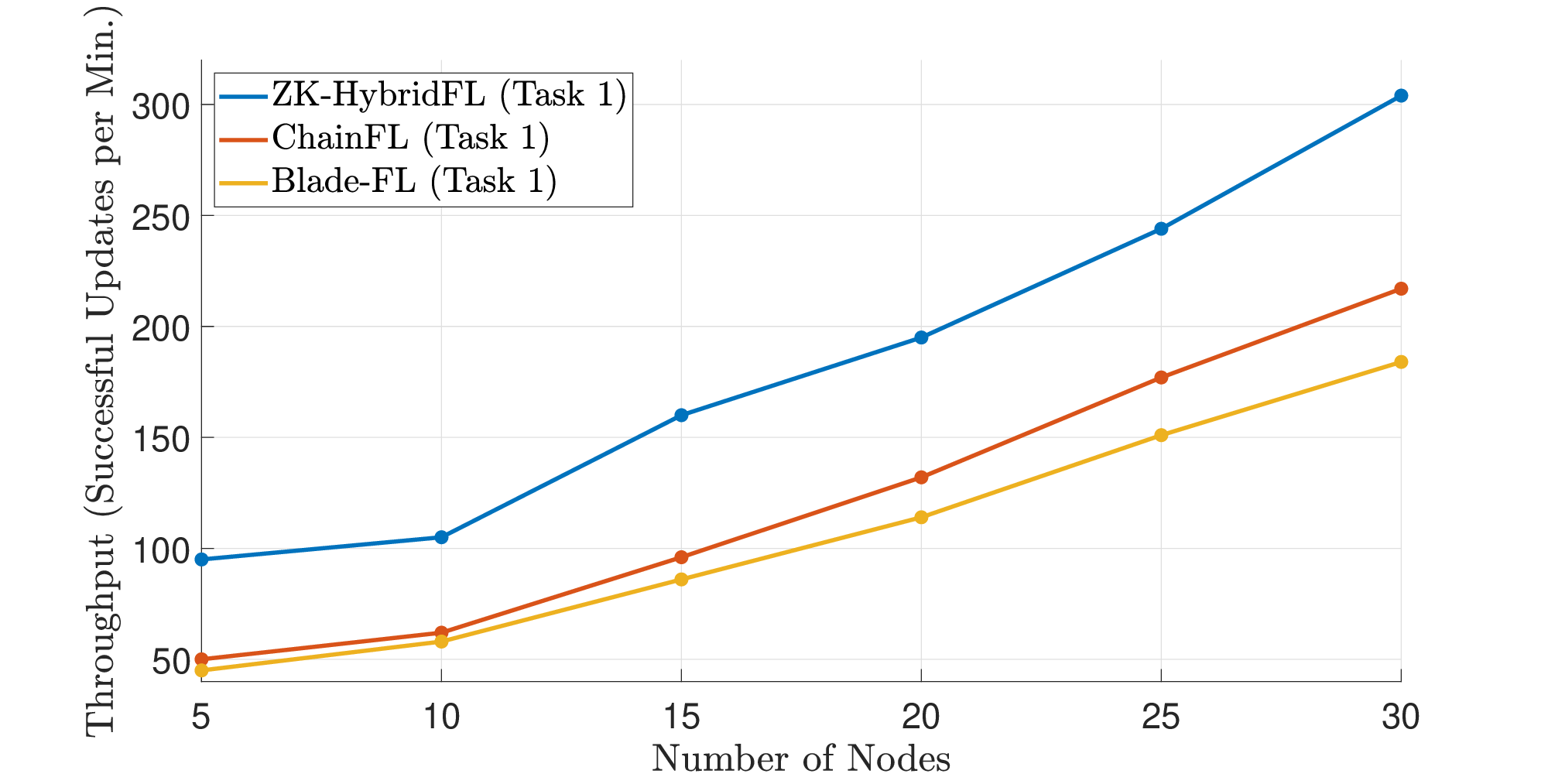}%
        \label{fff_9_1}%
    }\\[-0.25\baselineskip]
    \subfloat[Task~2]{%
        \includegraphics[width=\columnwidth]{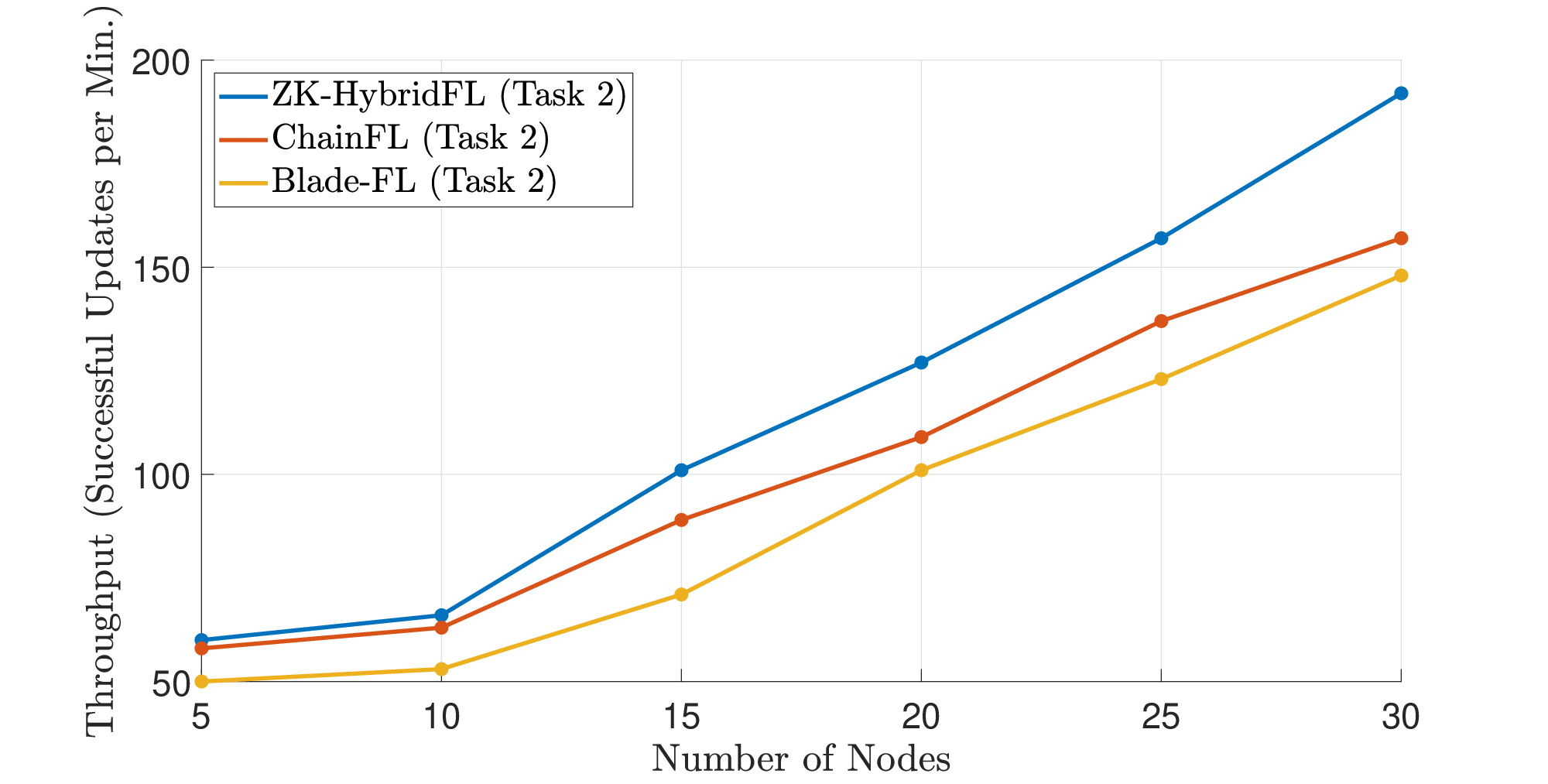}%
        \label{fff_9_2}%
    }
    \caption{Throughput of Blade-FL, ChainFL, and ZK-HybridFL versus the number of nodes \(n\) with \(\mu=20\%\) adversarial nodes and \(\gamma=10\%\) lazy nodes.}
    \label{throughput}
\end{figure}

\subsubsection*{\textbf{Scalability}}
We quantify scalability as the number of global epochs required for convergence. A global epoch is one full cycle in which all nodes perform local training and propagate their updates through the network. Let \(\mathcal L^t\) denote the training loss at epoch \(t\). Convergence is declared when
\begin{equation}
  \label{eq:conv-criterion}
  \frac{|\mathcal L^t - \mathcal L^{t-1}|}{\mathcal L^{t-1}} < \epsilon
\end{equation}
holds for five consecutive epochs, with \(\epsilon = 10^{-3}\).

\Cref{Scalability} reports the number of epochs required to meet \cref{eq:conv-criterion} as a function of \(n\) under \(\mu=20\%\) adversarial and \(\gamma=10\%\) lazy nodes. For both tasks, ZK-HybridFL converges in the fewest epochs, ChainFL is intermediate, and Blade-FL requires the most. As \(n\) increases, adversarial and lazy behavior compounds the weaknesses of Blade-FL’s PoW mechanism and ChainFL’s tip-based selection, forcing additional “recovery” epochs after corrupted updates. ZK-HybridFL, by pruning tainted or stale blocks via its challenge mechanism and loss-aware parent selection, avoids repeatedly training on compromised global models and thus reaches the convergence threshold significantly sooner.

\begin{figure}[!t]
    \centering
    \subfloat[Task~1]{%
        \includegraphics[width=\columnwidth]{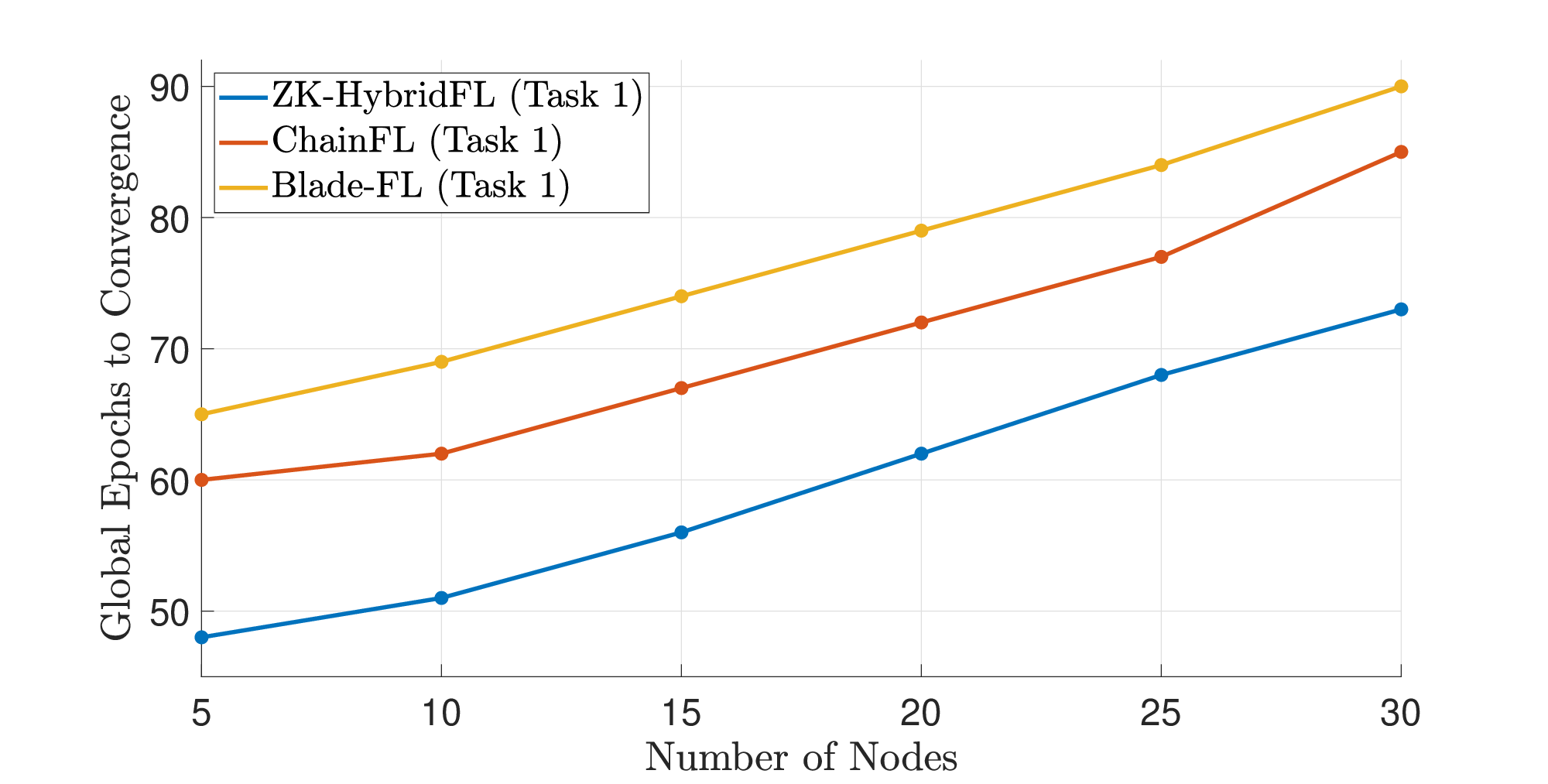}%
        \label{fff_8_1}%
    }\\[-0.25\baselineskip]
    \subfloat[Task~2]{%
        \includegraphics[width=\columnwidth]{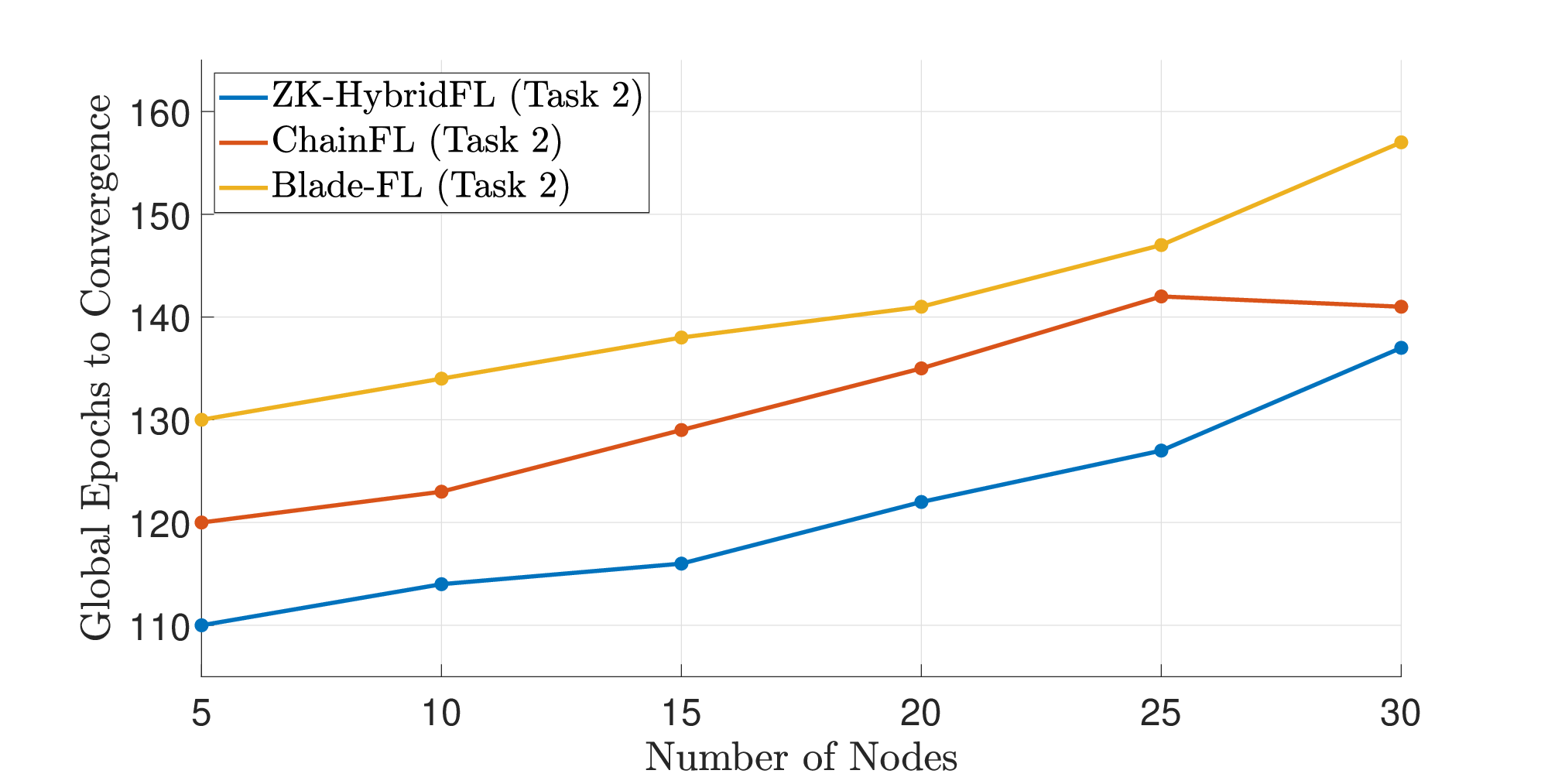}%
        \label{fff_8_2}%
    }
    \caption{Number of global epochs to convergence for Blade-FL, ChainFL, and ZK-HybridFL versus the number of nodes \(n\) with \(\mu=20\%\) adversarial nodes and \(\gamma=10\%\) lazy nodes.}
    \label{Scalability}
\end{figure}

\subsubsection{Zero-Knowledge Proof Cost and Scalability}
\label{sec:zkp-cost}
A key concern for ZK-HybridFL is the overhead of generating succinct noninteractive arguments of knowledge (SNARKs) for each node’s private validation batch. While on-chain verification is cheap, proof generation can be substantial for modern neural networks. We therefore benchmarked both proof and verification costs on the exact models used in our experiments, using a \(16\)-core Intel Xeon Gold 6338 (\SI{3.0}{\giga\hertz}, \SI{64}{\giga\byte} RAM) for central processing unit (CPU) runs and an NVIDIA A100 \SI{80}{\giga\byte} graphics processing unit (GPU) for accelerated proofs. Each node constructs a proof over a local test mini-batch of size \(B=10\), balancing statistical rigor and latency.

\Cref{tab:proof-cost} summarizes proof-generation time, GPU speedup, and proving-key memory footprint for each model.

\begin{table*}[t]
  \centering
  \caption{Proof-generation cost per node (batch size \(B=10\)).}
  \label{tab:proof-cost}
  \begin{tabular}{lrrrrr}
    \toprule
    \textbf{Model}        & \(\boldsymbol{\text{Params}}\) (\(\times 10^{3}\)) & \(\boldsymbol{\text{FLOPs/sample}}\) (\(\times 10^{3}\)) & \(\boldsymbol{\text{CPU prove}}\) (s) & \(\boldsymbol{\text{GPU prove}}\) (s) & \(\boldsymbol{\text{PK size}}\) (GB) \\
    \midrule
    MLP-3k                & \(3.6\)                                          & \(3.5\)                                                & \(3.2\)                              & \(0.51\)                             & \(0.16\)                             \\
    CNN-20k               & \(19.8\)                                         & \(68\)                                                 & \(32\)                               & \(4.1\)                              & \(4.2\)                              \\
    MobileNetV2-0.5       & \(1\,300\)                                       & \(3\times10^{5}\)                                      & \(665\)                              & \(76\)                               & \(31\)                               \\
    GRU-256               & \(29\)                                           & \(950\)                                                & \(360\)                              & \(44\)                               & \(4.3\)                              \\
    \bottomrule
  \end{tabular}
\end{table*}

Proof-generation time and memory footprint scale approximately linearly with the per-inference FLOPs. A small multilayer perceptron (MLP) with \(\approx 3.6\times10^{3}\) parameters and \(3.5\times10^{3}\) FLOPs per sample completes in \(\SI{3.2}{\second}\) on CPU and \(\SI{0.51}{\second}\) on GPU, with a \(\SI{0.16}{\giga\byte}\) proving key. A GRU-256 (about \(2.9\times10^{4}\) parameters and \(9.5\times10^{5}\) FLOPs) requires \(\SI{360}{\second}\) on CPU or \(\SI{44}{\second}\) on GPU, with a \(\SI{4.3}{\giga\byte}\) key. The largest model, MobileNetV2-0.5 (\(\approx 1.3\times10^{6}\) parameters and \(3\times10^{8}\) FLOPs), takes \(\SI{665}{\second}\) on CPU or \(\SI{76}{\second}\) on GPU, with a \(\SI{31}{\giga\byte}\) key. For the GRU-256, GPU proof time scales roughly linearly with batch size, following \(t_{\mathrm{prove}}\approx 4.4\,B \pm 0.35\,\text{s}\) for \(B\le 40\).

SNARK verification, by contrast, is lightweight. \Cref{tab:verify-cost} reports CPU verification time, proof size, and gas consumption on our Substrate test network. Verification completes in under \(\SI{0.11}{\second}\) for all models, with proof sizes between \(17\) and \(\SI{118}{\kilo\byte}\) and gas costs between \(\SI{19}{k.gas}\) and \(\SI{24}{k.gas}\). For a typical block-gas limit of \(\num{2e6}\), publishing proofs from \(30\) participants consumes less than \(1\%\) of block capacity; storage fees are on the order of \(5\times10^{-5}\) token per proof.

\begin{table}[t]
  \centering
  \caption{Proof-verification overhead.}
  \label{tab:verify-cost}
  \begin{tabular}{lrrr}
    \toprule
    \textbf{Model}      & \(\boldsymbol{\text{Verify time}}\) (s) & \(\boldsymbol{\text{Proof size}}\) (KB) & \(\boldsymbol{\text{Gas}}\) (k units) \\
    \midrule
    MLP-3k              & \(0.021\)                            & \(21\)                       & \(19\) \\
    CNN-20k             & \(0.028\)                            & \(17\)                       & \(21\) \\
    MobileNetV2-0.5     & \(0.095\)                            & \(118\)                      & \(24\) \\
    GRU-256             & \(0.108\)                            & \(41\)                       & \(23\) \\
    \bottomrule
  \end{tabular}
\end{table}

The two additional proofs introduced in \cref{sec:ZKP-extended} add only modest overhead on top of the base Groth16 proof. \Cref{tab:ext-cost} reports generation time, proof size, and gas for the Bulletproof \(\Sigma_j^t\) enforcing \(L_t \le \lVert\Delta\mathbf{W}_j^t\rVert_2 \le B_t\) and the embedding–cosine SNARK \(\Gamma_j^t\). Bulletproof proving time grows linearly with the number of parameters \(n\); even for MobileNetV2-0.5, the \(\ell_2\)-norm bound completes in about \(\SI{26}{\second}\) on CPU or \(\SI{2.6}{\second}\) on GPU. The cosine-check circuit is model-agnostic and yields a \(\SI{0.2}{\kilo\byte}\) proof in approximately \(\SI{2.4}{\second}\) (CPU) or \(\SI{0.30}{\second}\) (GPU). Verification costs are fixed at \(\SI{38}{k.gas}\) for \(\Sigma_j^t\) and \(\SI{25}{k.gas}\) for \(\Gamma_j^t\), raising the total per-update gas from roughly \(\SI{61}{k.gas}\) to \(\SI{124}{k.gas}\), still below \(5\%\) of a \(\num{2e6}\)-gas block for \(30\) concurrent trainers.

\begin{table*}[t]
  \centering
  \caption{Incremental proofs: generation and verification cost
           (same hardware as \cref{tab:proof-cost,tab:verify-cost}).}
  \label{tab:ext-cost}
  \begin{tabular}{lccccc}
    \toprule
    \textbf{Proof} &
    \textbf{CPU prove} &
    \textbf{GPU prove} &
    \textbf{Proof size} &
    \textbf{Verify gas} &
    \textbf{Dependence on \(n\)} \\
    \midrule
    \(\ell_2\)-BP, MLP-3k             & \(0.07\) s & \(0.008\) s & \(1.6\) kB & \(38\) k & linear \\
    \(\ell_2\)-BP, CNN-20k            & \(0.4\) s  & \(0.05\) s  & \(2.3\) kB & \(38\) k & linear \\
    \(\ell_2\)-BP, GRU-256            & \(0.6\) s  & \(0.08\) s  & \(2.8\) kB & \(38\) k & linear \\
    \(\ell_2\)-BP, MobileNetV2-0.5    & \(26\) s   & \(2.6\) s   & \(8.1\) kB & \(38\) k & linear \\
    \midrule
    Cosine-SNARK (all models)         & \(2.4\) s  & \(0.30\) s  & \(0.2\) kB & \(25\) k & constant \\
    \bottomrule
  \end{tabular}
\end{table*}

Thus, the extended proof bundle remains well within practical limits. Proof generation is overlapped with training via the predict-then-prove workflow (cf.\ \cref{Sec:ZKP}), and on-chain verification occupies less than \(5\%\) of block gas while adding negligible latency relative to consensus. Models that require more than \(\SI{24}{\giga\byte}\) of GPU memory for proving keys can use \texttt{tensorplonk} paging (\(\SI{512}{\mebi\byte}\) chunks) to stream keys from CPU RAM with under \(5\%\) overhead, or employ recursive folding (e.g., Halo Infinite) to partition large circuits into subproofs with sub-\(\SI{20}{\giga\byte}\) keys, recombined on chain with only \(\approx\SI{20}{\milli\second}\) additional verification time. A detailed implementation and performance analysis of a GRU-256 model using recursive folding is provided in the supplementary material (\cref{sec:gru_nova}).
\section{Conclusions}
\label{sec:conclusion}
We presented ZK-HybridFL, a secure decentralized federated learning framework integrating a DAG ledger, sidechain smart contracts, and zero-knowledge proofs for privacy-preserving, robust validation. Across image and text tasks, it outperforms Blade-FL and ChainFL, while delivering lower latency and higher throughput than proof-of-work or leader-based systems. Its challenge mechanism prunes invalid updates to accelerate convergence and reduce epochs. Future work will refine the cryptographic design and evaluate edge deployments.

\cleardoublepage
\section*{Supplementary}

\setcounter{section}{0}
\renewcommand{\thesection}{S\arabic{section}}

\setcounter{subsection}{0}
\renewcommand{\thesubsection}{\thesection-\Alph{subsection}}

\setcounter{subsubsection}{0}
\renewcommand{\thesubsubsection}{\thesubsection\arabic{subsubsection}}

\setcounter{secnumdepth}{3}

\section{Cryptographic Protocols and Security Analysis}
\label{sec:sup_crypto}
This section details the cryptographic foundations of the ZK-HybridFL system. It begins by specifying the concrete implementations of the zero-knowledge proofs, including KZG commitments and Groth16 SNARKs. It then introduces an expanded threat model with more subtle attacks and presents an extended ZKP bundle, which incorporates additional proofs such as Bulletproofs, as a countermeasure. The section concludes with a formal security analysis of these extended defenses against collusion and various privacy attacks.

\subsection{Zero-Knowledge Proofs: Instantiation and Extensions}

\subsubsection{Cryptographic Instantiation}
\label{sec:zkp-inst}

The generic workflow above abstracts away the precise commitment
scheme and the algebraic details of the SNARK. We now spell out the
concrete instantiation used in our implementation.

\noindent\textbf{Polynomial commitments:}
Let \(\mathbf{W}^t_{j}\in\mathbb F_p^{n}\) be the trainer’s weight tensor at round \(t\), flattened into the length-\(n\) vector
\[
  \mathbf{W}^t_{j}
    = \bigl(W^t_{j,0},\,W^t_{j,1},\,\dots,\,W^t_{j,n-1}\bigr),
\]
where each \(W^t_{j,i}\in\mathbb F_p\) is one scalar model parameter.  
Similarly, let the test-batch tensor \(D^{t,\mathrm{test}}_{\,j}\in\mathbb F_p^{m}\) be flattened as
\[
  D^{t,\mathrm{test}}_{\,j}
    = \bigl(D^{t,\mathrm{test}}_{\,j,0},\,D^{t,\mathrm{test}}_{\,j,1},\,\dots,\,D^{t,\mathrm{test}}_{\,j,m-1}\bigr),
\]
with each \(D^{t,\mathrm{test}}_{\,j,i}\in\mathbb F_p\) denoting one scalar feature or label value in the batch.

We then interpret these as evaluation vectors of degree-\((n-1)\) and degree-\((m-1)\) polynomials in \(\mathbb F_p[x]\):
\[
\begin{aligned}
  f_W(x) &= \sum_{i=0}^{n-1} W^t_{j,i}\,x^i, \\
  f_D(x) &= \sum_{i=0}^{m-1} D^{t,\mathrm{test}}_{\,j,i}\,x^i, \\
  f_W, f_D &\in \mathbb{F}_p[x].
\end{aligned}
\]

Under the KZG setup (secret \(\tau\in\mathbb F_p\)), we commit by
\[
  C^{t,\mathrm{model}}_{\,j}
    = g_1^{\,f_W(\tau)},
  \quad
  C^{t,\mathrm{test}}_{\,j}
    = g_1^{\,f_D(\tau)}.
\]
Each commitment is a single 48-byte point in \(\mathbb G_1\) and is stored
on the sidechain together with a Lamport timestamp (see \cref{sssidechain}).

\noindent\textbf{Trusted setup:}
A single powers-of-tau ceremony produces (i) the universal structured
reference string \(\bigl(g_1^{\tau^k}\bigr)_{k=0}^{k_{\max}}\) and (ii)
the circuit-specific proving key \(\mathsf{pk}\) and verification key
\(\mathsf{vk}\) for the Groth16 SNARK. The ceremony is executed once at
task deployment; afterwards \(\mathsf{pk}\) is distributed off-chain to
trainers, whereas \(\mathsf{vk}\) is pinned on-chain.

\noindent\textbf{Non-interactive proof:}
As described in the main-text ZK workflow (\cref{Sec:ZKP}), using \(\mathsf{pk}\), trainer \(j\) computes a Groth16 proof \(\Pi^t_{j}\) based on public inputs. Since the Fiat–Shamir heuristic is applied to derive challenges deterministically, no additional interaction or exchange of randomness is required.

\noindent\textbf{Verification contract:}
The sidechain validator performs:
\begin{enumerate}[label=(\roman*)]
  \item Two KZG opening checks  
        (two pairings and one multi-exponentiation; cost \(\approx \SI{34.2}{k.gas}\)  
        and \(\SI{144}{\byte}\) of storage for the two 48-byte commitments),
  \item One Groth16 verification call on  
        \(\bigl(
          C_{j}^{t,\mathrm{model}},\,
          C_{j}^{t,\mathrm{test}},\,
          \mathcal{Y}_{j}^{t},\,
          \mathcal{L}_{j}^{t},\,
          \Pi_{j}^{t}
        \bigr)\)  
        (cost \(20\)–\(\SI{27}{k.gas}\); see \cref{tab:verify-cost}).
\end{enumerate}
Even with 30 trainers, this, plus the additional \(\SI{144}{\byte}\) of on-chain storage, remains
under \(1\%\) of a \(\SI{2}{M.gas}\) block, keeping overhead negligible. If all checks
succeed, a \textsf{ProofOK} event is emitted. Otherwise, the update is
rejected and node \(j\)’s stake becomes challengeable.

\noindent\textbf{Asynchronous epochs:}
Immediately after publishing their commitments
\(\bigl(C^{t,\mathrm{model}}_{\,j},\,C^{t,\mathrm{test}}_{\,j}\bigr)\),
each trainer resumes the next SGD epoch while proof generation runs
concurrently (the “predict-then-prove” schedule). We allow a
two-epoch grace window, which is sufficient even for a \(\SI{76}{\second}\) GPU proof for
MobileNetV2-0.5 on MNIST (each epoch itself takes \(\SIrange{24}{30}{\second}\)). Commitments
and proofs carry Lamport timestamps to totally order sidechain events,
so late proofs simply land in the next round without any global pause (see \cref{sssidechain}).

\subsubsection{Security Guarantees}
\label{sec:zkp-security}

The combination of KZG commitments and Groth16 delivers the usual
\emph{completeness}, \emph{soundness}, and \emph{zero-knowledge}
properties. Below we make those guarantees explicit for the concrete
operations performed in ZK-HybridFL.

\noindent\textbf{Model or data substitution:}
Let \(C^{t,\mathrm{model}}_{\,j}\) and \(C^{t,\mathrm{test}}_{\,j}\) be
the commitments posted in step~(i) of \cref{sec:zkp-inst}, and let
\(\Pi=\Pi^t_{j}\) be the proof broadcast in step~(iii). Suppose the
prover attempts to convince the verifier with an altered model
\(\widetilde{\mathbf{W}}\neq \mathbf{W}^t_{j}\) or test batch
\(\widetilde{D}\neq D^{t,\mathrm{test}}_{\,j}\). Because KZG is
\emph{binding}, the pair of polynomials
\(\bigl(f_{\widetilde W},\,f_{\widetilde D}\bigr)\) cannot both satisfy
\[
  g_1^{\,f_{\widetilde W}(\tau)} 
    = C^{t,\mathrm{model}}_{\,j}
  \quad\text{and}\quad
  g_1^{\,f_{\widetilde D}(\tau)} 
    = C^{t,\mathrm{test}}_{\,j},
\]
so at least one KZG-opening check fails and verification returns \(0\).

\noindent\textbf{Fabricating outputs or loss:}
Suppose instead the prover keeps the committed witness
\(\bigl(\mathbf{W}^t_{j},D^{t,\mathrm{test}}_{\,j}\bigr)\) but replaces the public
outputs by
\(\widetilde{\mathcal Y}\neq\mathcal Y^t_{j}\) or
\(\widetilde{\mathcal L}\neq\mathcal L^t_{j}\). The algebraic relations
embedded in the Groth16 circuit bind the full witness
\(\mathcal U^t_{j}\) to the declared outputs, so any mismatch violates
a circuit constraint and forces the SNARK verifier to return \(0\).

\noindent\textbf{Lazy-node detection:}
Each bundle \(Z^t_{j}\) includes the hash
\(\mathsf{H}\bigl(\mathcal Y^t_{j}\bigr)\) alongside
\(C^{t,\mathrm{model}}_{\,j}\) and \(C^{t,\mathrm{test}}_{\,j}\). Because
both KZG and \(\mathsf{H}\) are collision-resistant, two bundles that
replay an \emph{identical} prediction vector must derive from the same
\((\mathbf{W},D)\) witness (and thus the same loss). Smart contract \(S_j^{1}\)
cross-checks these hashes across the DAG: any node that replays an old
model to avoid fresh training is flagged and its update excluded from
the parent-selection ranking (cf.\ \cref{SC1}), mitigating
lazy-node attacks.

\noindent\textbf{Consistency across forks:}
KZG commitments and the hash
\(\mathsf{H}(\mathcal Y)\) are stored in the immutable part of the
sidechain; therefore every honest replica of the DAG sees \emph{exactly}
the same triplet
\(\bigl(C^{t,\mathrm{model}},\,C^{t,\mathrm{test}},\,\mathsf{H}(\mathcal Y)\bigr)\).
This global consistency enables (i) deterministic parent selection and
(ii) unambiguous fork resolution, even under short-lived network forks.

Taken together, these properties ensure that only
\emph{genuinely trained, fully verified} updates influence the global
model, while any adversarial attempt, whether by parameter substitution,
fabricated outputs, or stale replay is detected and rejected except
with negligible probability.

\subsection{Extended ZKP Defenses}

\subsubsection{Broader Threat Model}
\label{sec:adv-motivation}

While the core workflow of \cref{Sec:ZKP} ensures that blatantly bogus updates (e.g., completely untrained models, fabricated outputs, or exact replays) are caught with overwhelming probability, a realistic threat model must also consider subtler strategies that deliberately skirt these checks. Below, we sketch three attack avenues that slip through the ZKP workflow in \cref{Sec:ZKP}.

First, a node skips training altogether and re-posts an older weight tensor  
\(\mathbf{W}_{j}^{t-\Delta}\) with a microscopic perturbation. This changes the hash of the predictions, so replay detection in \cref{Sec:ZKP} does not trigger, but the global model sees no genuine progress.

Second, instead of performing full local SGD, the node applies a single gradient step \(\delta\mathbf{W}\) and then scales it by a tiny factor \(\alpha\). The resulting update \(\mathbf{W}_{j}^{t-1}+\alpha\delta\mathbf{W}\) meets the norm bound and earns the block reward, yet makes negligible impact on the model’s decision surface.

Third, since the test batch \(D_{j}^{t,\text{test}}\) is chosen privately, a malicious trainer can cherry-pick samples on which \(\mathbf{W}_{j}^{t}\) already performs well, or even craft an easy synthetic set, then prove the accuracy honestly. The ZKP verifies, but the reported loss overstates the model’s real-world quality.

To address these security gaps, next we introduce lightweight extensions that detect and penalize the three strategies above, without reshaping the Groth16-based core.

\subsubsection{Extended ZKP Bundle}
\label{sec:ZKP-extended}

In each training epoch \(t\), we extend the original Z-bundle from \cref{Sec:ZKP} with two lightweight, non-interactive proofs to ensure that every submitted update both moves the weights by a meaningful amount and alters the model’s internal representations. The efficiency of these methods has been demonstrated in~\cite{li2021model,lycklama2023rofl}. We defer a detailed security analysis of these extended ZKPs to \cref{subsec:attacks}. 

Specifically, in epoch \(t\), each client \(j\) is given the public thresholds \([L_t, B_t]\) and \(\tau_{\max}\). See \cref{sec:thresholds} for how these thresholds are computed and how consensus is reached across the network through interactions with the oracle committee, sidechain, and smart contracts.

With these parameters in hand, each node \(j\) proceeds as follows in epoch \(t\):

\begin{enumerate}
  \item \textbf{Constructs the original Z-bundle.}\\
  The node first performs standard local training and generates a Groth16 proof \(\Pi_j^t\) verifying inference correctness. It then commits the model and test data and assembles the initial Z-bundle:
  \[
    Z_j^t = \left(
      \Pi_j^t,\;
      \mathcal{Y}_j^t,\; \mathcal{L}_j^t,\;
      C_j^{t,\mathrm{model}},\;
      C_j^{t,\mathrm{test}}
    \right).
  \]

  \item \textbf{Proves the norm of the update with Bulletproofs.}\\
  To ensure that the weight change is non-trivial, the node computes the update
  \(\Delta\mathbf{W}_j^t = \mathbf{W}_j^t - \mathbf{W}_j^{t-1}\) and produces a range proof (Bulletproof)
  \[
    \Sigma_j^t = \mathsf{BPProve}\bigl(\Delta\mathbf{W}_j^t,\, L_t, B_t\bigr).
  \]
  This proof references the existing commitment \(C_j^{t,\mathrm{model}}\) to link the norm constraint to the committed model.

  \item \textbf{Attests to semantic change using embedding-cosine SNARK.}\\
  The node computes layer-\(\ell\) activations over a public probe set:
  \begin{align*}
  \mathbf{z}_{\mathrm{old}} &= \frac{1}{|D_{\mathrm{probe}|}} 
    \sum_{x \in D_{\mathrm{probe}}} \phi_\ell(\mathbf{W}_j^{t-1}, x), \\
  \mathbf{z}_{\mathrm{new}} &= \frac{1}{|D_{\mathrm{probe}|}} 
    \sum_{x \in D_{\mathrm{probe}}} \phi_\ell(\mathbf{W}_j^t, x),
  \end{align*}
  and generates a SNARK proving \(\cos(\mathbf{z}_{\mathrm{old}}, \mathbf{z}_{\mathrm{new}}) \le \tau_{\max}\):
  \[
    \Gamma_j^t = \mathsf{SNARKProve}\left(
      \cos(\mathbf{z}_{\mathrm{old}}, \mathbf{z}_{\mathrm{new}}) \le \tau_{\max}
      \;;\; C_{\mathrm{probe}}
    \right).
  \]

  \item \textbf{Assembles the extended bundle.}\\
  The node merges the original bundle with the two additional proofs:
  \begin{align*}
    Z_j^{t,\mathrm{ext}} &= Z_j^t \;\|\, \Sigma_j^t \;\|\, \Gamma_j^t \\
    &= \left(
      \Pi_j^t,\; \Sigma_j^t,\; \Gamma_j^t,\;
      \mathcal{Y}_j^t,\; \mathcal{L}_j^t,\;
      C_j^{t,\mathrm{model}},\; C_j^{t,\mathrm{test}}
    \right).
  \end{align*}

  \item \textbf{Submits for peer or on-chain verification.}\\
  A verifier node \(j'\) or the chain checks all parts in order:
  \begin{enumerate}
    \item \(\mathsf{VerifyGroth16}\bigl(
      \mathsf{vk},\; \Pi_j^t,\;
      C_j^{t,\mathrm{model}},\;
      C_j^{t,\mathrm{test}},\;
      \mathcal{Y}_j^t,\; \mathcal{L}_j^t\bigr)\),
    \item \(\mathsf{VerifyBulletproof}\bigl(
      \Sigma_j^t;\; C_j^{t,\mathrm{model}}\bigr)\),
    \item \(\mathsf{VerifyCosineZKP}\bigl(
      \Gamma_j^t;\; C_{\mathrm{probe}}\bigr)\).
  \end{enumerate}
  The update is accepted into the DAG only if all verifications succeed.
\end{enumerate}

Note that \(\Sigma_j^t\) must reference \(C_j^{t,\mathrm{model}}\) so that the norm bound applies to the committed weights, and \(\Gamma_j^t\) refers to the fixed public probe-set commitment \(C_{\mathrm{probe}}\). No additional commitments are required.

Moreover, in addition to the two KZG openings (approximately \(\SI{34.2}{k.gas}\)) and the Groth16 verification (\(20\)–\(\SI{27}{k.gas}\)), each update in the extended scheme carries:
\begin{itemize}
  \item A Bulletproof \(\Sigma_j^t\) for the \(\ell_2\)-bound (proof size \(\approx \SI{8.1}{\kilo\byte}\) off-chain), which costs \(\approx \SI{38}{k.gas}\) to verify.
  \item An embedding-cosine SNARK \(\Gamma_j^t\) (proof size \(\approx \SI{200}{\byte}\)), which costs \(\approx \SI{25}{k.gas}\) to verify.
\end{itemize}

Taken together, the two new verifications add roughly \(\SI{63}{k.gas}\), under \(3.2\%\) of a \(\SI{2}{M.gas}\) block, on top of the original \(\approx \SI{61}{k.gas}\), keeping total overhead below \(5\%\) of block capacity; see \cref{tab:ext-cost} for a full breakdown.

\begin{remark}
Although we have extended the original Z-bundle, the underlying DAG ledger and its consensus rules (\cref{subsec:dag}) remain unchanged. Likewise, the challenge mechanism of \cref{subsection:challenge} applies verbatim to each extended bundle: any node may still issue and resolve challenges against \(Z_j^{t,\mathrm{ext}}\) under the same stake-based rules.
\end{remark}

\subsection{Extended Security Analysis}
\label{subsec:attacks}

Let \(\mathsf{Accept}_{t}(D_{j}(t))\!\in\!\{0,1\}\) be the sidechain
predicate that a candidate block \(D_{j}(t)\) is eventually \emph{confirmed}
and its model \(\mathbf{W}_{j}^{t}\) enters the global aggregation of
epoch~\(t\) (cf.\ \cref{Sec:ZKP}-\cref{SCs}).  
An adversarial trainer replaces the honest map
\[
\mathcal T:(\mathbf{W}_{j}^{t-1},\mathcal D_{j}^{t,\mathrm{train}})
          \longmapsto\mathbf{W}_{j}^{t}
\]
by \(\widetilde{\mathcal T}\) and tries to maximise
\(\Pr[\mathsf{Accept}_{t}(D_{j}(t))=1]\) while breaking at least one of
\begin{enumerate}[leftmargin=1.4em,label=\textbf{\arabic*.}]
  \item \emph{Freshness} – the update should encode genuine computation;
  \item \emph{Correctness} – the loss/accuracy in \(Z_{j}^{t}\) must be truthful;
  \item \emph{Privacy} – the attack must not leak information that
        enables model inversion or membership inference.
\end{enumerate}

\subsubsection*{\textbf{Extended ZKP Attack Defenses}}

We formalise three subtle strategies and prove that, once the extended
Z-bundle
\(\bigl(\Pi\parallel\Sigma\parallel\Gamma\bigr)\)
is enabled, each is rejected with probability
\(\mathrm{negl}(\lambda)\), where \(\lambda\!=\!128\) is the global
security parameter.

\noindent\textbf{Attack \(\mathcal A_{\mathrm{lazy}}\) (Perturb–Replay).}
The node skips training and publishes
\[
   \mathbf{W}_{j}^{t}
     = \mathbf{W}_{j}^{t-\Delta} + \bm\eta,
   \qquad
   \bm\eta\sim\mathcal N\!\bigl(\mathbf 0,\sigma^{2}\mathbf I\bigr),
\]
with \(\Delta\!\ge\!1\) and
\(\lVert\bm\eta\rVert_{2}\!\ll\!\lVert\mathbf{W}_{j}^{t-\Delta}\rVert_{2}\),
so the prediction hash changes but
\(\mathbf{W}_{j}^{t}\) is semantically identical to an old model.

\emph{Defence.}\;
If \(C^{t,\mathrm{model}}_{\,j}\) equals the old commitment, the replay
filter of \cref{Sec:ZKP} fires.  
Otherwise \(\Sigma_{j}^{t}\) proves
\(L_t\!\le\!\lVert\mathbf{W}_{j}^{t}-\mathbf{W}_{j}^{t-1}\rVert_{2}
       \!=\!\lVert\bm\eta\rVert_{2}
     \!\le\!B_t\).
By choosing \(\sigma\) so that
\(\Pr[\lVert\bm\eta\rVert_2\!\ge\!L_t]\!\le\!2^{-128}\),
the attack succeeds only with negligible probability.

\noindent\textbf{Attack \(\mathcal A_{\mathrm{scale}}\) (Minimal-Norm Stalling).}
Compute one honest gradient \(\delta\mathbf{W}\) and publish
\[
\mathbf{W}_{j}^{t}
   = \mathbf{W}_{j}^{t-1}
   + \alpha\,\delta\mathbf{W}
\]
with
\(\alpha=L_t/\|\delta\mathbf{W}\|_2\),
so \(\Sigma_{j}^{t}\) passes but the semantic change is tiny.

\emph{Defence.}\;
Let \(\bar{\mathbf{z}}_{\ell}(\mathbf{W})\) be the average layer-\(\ell\) activation
on the fixed \(|D_{\mathrm{probe}}|{=}\!400\) public probe samples.
From~\cite{lycklama2023rofl} this map is
\(\kappa\)-Lipschitz:
\(\|\bar{\mathbf{z}}_{\ell}(\mathbf{W}+\Delta)
      -\bar{\mathbf{z}}_{\ell}(\mathbf{W})\|_2
   \le \kappa\,\|\Delta\|_2\).
Hence
\[
  \cos\!\bigl(\bar{\mathbf{z}}_{\ell}(\mathbf{W}_{j}^{t-1}),
              \bar{\mathbf{z}}_{\ell}(\mathbf{W}_{j}^{t})\bigr)
  \;\le\;
  1 - \frac{\kappa^2 L_t^2}{2\,
          \|\bar{\mathbf{z}}_{\ell}(\mathbf{W}_{j}^{t-1})\|_2^2}
  \;<\; \tau_{\max},
\]
whenever \(\kappa L_t > (1-\tau_{\max})\).
The inequality holds for our empirical choice
\(\kappa\!\approx\!0.14\),
\(\tau_{\max}\!=\!0.98\),
\(L_t\!\ge\!10^{-2}\).  
Thus the embedding-cosine SNARK \(\Gamma_{j}^{t}\) fails and
\(\mathsf{Accept}_{t}=0\).

\noindent\textbf{Attack \(\mathcal A_{\mathrm{priv}}\) (Private-Test Cherry Pick).}
Keep \(\mathbf{W}_{j}^{t}\) honest but choose
\(\widetilde{\mathcal D}_{j}^{t,\mathrm{test}}\)
to inflate the reported accuracy.

\emph{Defence.}\;
(i) The block’s rank in the parent-selection list
(\cref{SC1}) depends on \(\mathcal L_{j}^{t}\).
By inflating accuracy (lower loss) the adversary \emph{increases} the probability
that its own block is selected; this is not obviously harmful
yet exposes the block to scrutiny.  
(ii) Before aggregation, any node evaluates the public probe set;
if the empirical loss differs by more than
\(\varepsilon_{\text{probe}}=0.01\) from the stated
\(\mathcal L_{j}^{t}\), it files a challenge
(\cref{subsection:challenge}).  
The honest majority of oracles detects the mismatch with
constant probability \(p_{\textit{det}}\!\ge\!1/2\),
so the expected stake loss per dishonest epoch is at least
\(p_{\textit{det}}\cdot s_{\min}\),
where \(s_{\min}\) is the minimum slashing fraction.
After
\(T \!\ge\! \bigl(\omega_{j}^{0}/s_{\min}\bigr)\!/\!p_{\textit{det}}
   = O(\lambda)\)
epochs, the attacker’s weight drops below the confirmation threshold and
its updates are ignored.  
Thus
\[
   \Pr\bigl[\mathsf{Accept}_{t}=1
        \land
        |\widehat{\mathcal L}_{\mathrm{probe}}
           -\mathcal L_{j}^{t}|>\varepsilon_{\text{probe}}\bigr]
   \le \mathrm{negl}(\lambda).
\]

Thus, combining binding of KZG commitments,
knowledge-soundness of Groth16/Bulletproof/SNARK proofs,
and the economic penalty from the challenge mechanism yields
\[
   \Pr[\exists\,t,j :
        \mathsf{Accept}_{t}(D_{j}(t))=1
        \land D_{j}(t)\text{ is adversarial}]
   \;=\; \mathrm{negl}(\lambda).
\]
\hfill\(\square\)

\subsubsection*{\textbf{Additional Privacy and Collusion Threats}}
\label{subsec:reviewer-threats}

\noindent\textbf{Collusion–induced weight inflation}
\label{sss:collusion}

\noindent\textbf{Attack \(\mathcal A_{\mathrm{col}}\).}
A coalition \(\mathcal C\subseteq\{1,\dots,n\}\) with total stake
\(\Omega_{\mathcal C}=\sum_{j\in\mathcal C}\omega_j\)
tries to force an \emph{invalid} block \(B_{\textit{adv}}\) into the
confirmed set by (i) cross-verifying each other’s proofs and
(ii) recursively selecting \(B_{\textit{adv}}\) as an ancestor so that
its AW eventually exceeds the confirmation
threshold~\(\eta\) (\cref{subsec:dag}).

\begin{lemma}[Bounded-stake collusion]
\label{lem:collusion}
Let \(M\) be the oracle-committee size and assume
\(M\ge3f+1\) with at most \(f\) Byzantine oracles
(\(>\!2/3\) honest stake).  
If \(\Omega_{\mathcal C}<\eta/3\) then
\[
  \Pr\bigl[B_{\textit{adv}}\text{ confirmed}\bigr]\;=\;\mathrm{negl}(\lambda).
\]
\end{lemma}

\noindent\textit{Sketch.}
Every new block needs a threshold-signed
\texttt{EventApproved} log.  
Because at least \(f+1\) honest oracles must co-sign, an invalid
\(B_{\textit{adv}}\) can be published only if at least one honest oracle
is fooled by a forged SNARK; the soundness error of Groth16 is
\(\varepsilon_{\textit{SNARK}}\le2^{-\lambda}\).  
Subsequent children add at most \(\Omega_{\mathcal C}\) weight each, so
after \(k\) rounds the total AW on \(B_{\textit{adv}}\) is bounded by
\(\Omega_{\mathcal C}k < k\,\eta/3\).
But at least \(k\) honest blocks appear in the
same future cone, contributing
\(k\,(1-\Omega_{\mathcal C})>2k\eta/3\),
so \(B_{\textit{adv}}\) can never reach \(\eta\).  
A full proof follows standard PBFT stake-counting (\cref{tired_V1}).  \hfill\(\square\)

\begin{remark}
ZK-HybridFL is collusion-resistant as long as the adversary controls less than \(\eta/3\) stake and less than \(\tfrac13\) of the oracle committee, matching the usual BFT threshold.
\end{remark}

\noindent\textbf{Model-inversion attacks}
\label{sss:inversion}

\noindent\textbf{Threat.}
Given the public trajectory
\(\{\tilde{\mathbf{W}}^{t}\}_{t\le T}\),
an adversary solves an optimisation
\[
\widehat{\mathbf{x}}=\arg\min_{\mathbf{z}}
      \sum_{t}\!
      \bigl\|
        \nabla_{\mathbf{W}}\mathcal L(\tilde{\mathbf{W}}^{t};\mathbf{z})
      - \nabla_{\mathbf{W}}\mathcal L(\tilde{\mathbf{W}}^{t};\mathbf{x})
      \bigr\|_2^2
\]
to reconstruct a private training point \(\mathbf{x}\)
\cite{fredrikson2015inversion,zhu2019dlgi}.

\noindent\textbf{Baseline mitigation.}
ZK-HybridFL never discloses \emph{gradients}; only final weights are
visible. Empirically, inversion from weights is far noisier than from
per-step gradients~\cite{abadi2016dp}.  
Nevertheless the risk is not cryptographically closed.

\noindent\textbf{Plug-in defences (future work).}
The pipeline is orthogonal to:
(i) DP-SGD noise addition
\cite{abadi2016dp,kairouz2021advances};  
(ii) secure aggregation of weights
\cite{bonawitz2017secagg,so2023gentle}; or  
(iii) local representation perturbation
\cite{li2021model,mohassel2021privacy}.  
All three add-ons preserve differentiability, so the Groth16 circuit
and the Bulletproof/SNARK range checks remain valid and only the public parameters \(L_t,B_t,\tau_{\max}\) need re-tuning.  
We leave an optimal privacy–accuracy trade-off to future work.

\noindent\textbf{Membership-inference attacks}
\label{sss:membership}

\noindent\textbf{Threat.}
Given black-box access to
\(f_{\tilde{\mathbf{W}}^{t}}\),
decide whether a probe sample \(\mathbf{x}^{\star}\) was in some honest
node’s training set
\cite{shokri2017membership,salem2019mitigations}.

\noindent\textbf{Baseline mitigation.}
Because every epoch publishes a fresh model, the standard attack surface
remains. ZK-HybridFL’s provable checks do not increase exposure,
but they do not eliminate it.

\noindent\textbf{Engineering add-ons.}
The following local defences are compatible with our ledger:
\begin{itemize}[leftmargin=*,nosep]
  \item \textbf{Prediction clipping} or top-\(k\) smoothing
        before a node queries the global model
        \cite{cheon2022clipmi}.
  \item \textbf{Adversarial regularisers}
        that minimise empirical attack accuracy
        \cite{nasr2018privunit}.
  \item \textbf{Differential privacy}
        (same machinery as for inversion).
\end{itemize}

Therefore, collusion is provably bounded by \Cref{lem:collusion}.  
Model inversion and membership inference are already \emph{harder}
than in gradient-sharing FL, yet not information-theoretically blocked.
Fortunately, the ledger, sidechain, and ZK workflow are agnostic to
DP-SGD, secure aggregation, or output-clipping layers, so these defenses
can be enabled as a straightforward future extension without altering
the core protocol.

\section{System Architecture and Mechanics}
\label{sec:sup_arch}

This section explains the core infrastructure and operational mechanics of the platform. It primarily focuses on the oracle-assisted sidechain, detailing the role of the Oracle Committee in validating off-chain events and the sidechain’s function in executing smart contracts using Lamport clocks for ordering, which avoids a separate consensus mechanism. It then outlines the necessary architectural adaptations, such as dynamic threshold computation and smart contract modifications, required to support the extended ZKP defenses introduced in \cref{sec:ZKP-extended}.

\subsection{Oracle-Assisted Sidechain}

\subsubsection{Oracle Committee: Event-Admission Layer}
\label{subsuboracle_supp}

ZK-HybridFL introduces a lightweight, fault-tolerant Oracle Committee that serves as an event-admission layer between the DAG and the Event-Driven Smart Contracts (EDSCs). Rather than allowing nodes to invoke contracts directly, which risks malformed or out-of-context messages, the committee observes each raw DAG emission, applies a suite of structural and contextual checks, and only publishes a succinct, threshold-signed \texttt{EventApproved} log on-chain. By decoupling schema and consistency validation from on-chain business logic, we ensure that EDSCs can react immediately and safely, without incurring bulky consensus or verification costs.

Membership in the Oracle Committee aligns with the same Byzantine-fault threshold as the underlying DAG. We require \(M \ge 3f + 1\) full nodes to register as oracles, where \(f\) is the maximum number of Byzantine failures tolerated. Each candidate deposits collateral and publishes a BLS public key in an on-chain \texttt{OracleRegistry}; only active, staked nodes participate in validation. Should a member prove adversarial or offline for too many consecutive events, an EDSC-driven slashing mechanism automatically revokes its status and redistributes its stake. This registry not only governs entrance and exit but also encodes the threshold \(f+1\) necessary for event publication.

Once the committee is formed, each oracle runs an off-chain watcher that listens to the DAG for new raw events \(e = \{\mathrm{type}, \mathrm{payload}, \mathrm{meta}\}\). Upon observing \(e\), the node verifies conformity to the expected schema, cross-references any referenced block hashes or commitments against the DAG’s current state, checks stake requirements or performance metrics, and ensures that Lamport timestamps advance monotonically. If all checks succeed, the node computes a BLS partial signature \(\sigma_i\) on the hash \(h = H(e)\) and gossips \((e, \sigma_i)\) to its peers. This approach leverages off-chain computation, which is two orders of magnitude faster than on-chain proof verification, while preserving auditability.

When any oracle collects \(f+1\) valid partial signatures on \(h\), it combines them into a single BLS threshold signature \(\Sigma\) and submits one compact transaction to the \texttt{EventAdmission} contract. That transaction does nothing more than emit the standardized log
\begin{lstlisting}
// EventAdmission.sol
event EventApproved(bytes32 indexed hash, bytes payload);
\end{lstlisting}
This Solidity snippet represents the canonical interface for approved event emission.\footnote{%
The Solidity snippets in this section serve as protocol-level specifications of contract interfaces and behaviors.
In our simulation setup, these contracts are implemented as WebAssembly modules using Substrate’s \texttt{ink!} smart contract framework.
The logic, event signatures, and access controls are preserved identically.}
Here \texttt{hash = h} and \texttt{payload} encodes the original event’s essential fields. By limiting on-chain work to a single BLS-verify and log emission, we keep gas costs and block congestion to a minimum: the entire approval process amounts to a few dozen thousand gas, reimbursed from the network’s fee pool. Crucially, any light client or sidechain participant can independently re-compute \(h = H(e)\) and run a local BLS verification against the committee’s public key to confirm both the integrity of the event and the quorum that approved it.

Once the \texttt{EventApproved} log appears in the sidechain, all EDSCs that have been coded to “subscribe” to this signature automatically wake up. In Solidity this takes the form of a handler such as
\begin{lstlisting}
function onEventApproved(bytes32 hash, bytes calldata payload) external {
    require(msg.sender == address(EventAdmission));
    // decode and enforce domain rules, e.g. stake thresholds, performance checks

}
\end{lstlisting}
Each contract thus carries its own domain-specific checks—whether labeling a new DAG tip, disbursing rewards, slashing misbehaving participants, or aggregating model updates—confident that malformed or replayed messages cannot slip through. The separation of concerns ensures that off-chain vetting remains focused on message consistency, while on-chain logic governs economic and security policies without redundant validation overhead.

To support evolving network conditions and to guard against long-term stagnation or collusion, committee membership is fully updatable. Prospective nodes stake collateral and call \texttt{join()} on the \texttt{OracleRegistry}, while existing members can be slashed via dedicated dispute contracts for misbehavior or liveness failures. When a membership rotation is desired, a new registry is deployed, and during a hand-off period both old and new sets co-sign events for \(k\) epochs, ensuring no gap in approval coverage. By encoding the registry’s address in each EDSC (via an updatable pointer), committees can be refreshed without redeploying business-logic contracts, preserving continuity and decentralization over the network’s lifetime.

On commodity hardware the only on-chain work performed by the \texttt{EventAdmission} contract is a single BLS verification and log emission. Empirical micro-benchmarks show that aggregating and verifying a threshold BLS signature for \(f + 1\) parties requires approximately \SI{1.4}{ms} of CPU time, while publishing the transaction consumes roughly \SI{35}{k.gas} and carries a payload under \SI{512}{\byte}. Even in a worst-case training scenario with 32 approved events per epoch, total off-chain CPU overhead remains below \SI{0.1}{s} and network bandwidth under \SI{35}{\kilo\byte}, and the on-chain gas cost accounts for less than \(1\%\) of a \SI{2}{M.gas} block.

\begin{table}[h]
  \centering
  \caption{Oracle Committee micro-benchmarks (per approved event).}
  \label{tab:oracle-cost}
  \begin{tabular}{lcc}
    \toprule
    Operation & Latency & Gas / Data \\
    \midrule
    Threshold BLS verify \((f+1)\) & \SI{1.4}{ms} & 96 B signature \\
    \texttt{publishApprovedEvent} TX & -- & 35 k gas, \(\le\) 512 B payload \\
    \bottomrule
  \end{tabular}
\end{table}

In summary, the Oracle Committee functions purely as a \emph{consistency gate}, isolating structural and contextual checks from the contracts themselves. By emitting a single, threshold-signed \texttt{EventApproved} log, it guarantees that downstream EDSCs only ever process canonical, well-formed events. All substantive business logic—including tip selection, reward distribution, slashing, and model aggregation—remains inside the smart contracts, which trust but verify the committee’s work. This design preserves the DAG’s BFT assumptions, provides a clear consensus/proof path from raw DAG event to on-chain state transition, enables seamless committee rotation, and minimizes both on-chain and off-chain overhead.

\subsubsection{Sidechain}
\label{sssidechain}

The ZK-HybridFL sidechain serves as a specialized ledger for storing and executing EDSCs, ensuring that high-frequency interactions do not overload the DAG. Its block structure is designed to encapsulate the data elements essential for the protocol’s operation and network logistics. In the initial blocks, the sidechain records identifiers for the members of the Oracle Committee (as registered in the \texttt{OracleRegistry}), along with stake-related metadata determining each node’s weight \(\omega_j\). During the network’s one-time deployment phase, EDSCs are deployed on the sidechain. Once live, subsequent sidechain blocks record only approved events, i.e., the \texttt{EventApproved} logs emitted by the \texttt{EventAdmission} contract pinning the universal SNARK keys \(\mathsf{pk}, \mathsf{vk}\) and capturing dynamic data such as KZG commitments \(C^{t,\mathrm{model}}_{\,j}\), \(C^{t,\mathrm{test}}_{\,j}\), reward distributions for successful model submissions, and updates to participants’ staked tokens.

Unlike conventional blockchains that run full consensus (e.g., PoW) on every block, the ZK-HybridFL sidechain leverages its event-driven architecture in lieu of a separate consensus protocol. Whenever \texttt{EventAdmission} emits an \texttt{EventApproved} log, the sidechain simply attaches the corresponding block with the timestamp of that transaction; blocks are ordered by their inclusion time rather than by a consensus-based proposal process. This built-in ordering mechanism eliminates the need for additional consensus overhead, simplifying transaction validation and ordering on the sidechain.

Relying on physical clocks across a decentralized network introduces challenges like clock drift and latency~\cite{KII13}. To avoid these, ZK-HybridFL uses Lamport logical clocks~\cite{KII14}, carried in each approved event’s metadata. Oracle Committee members increment local counters and propagate logical timestamps alongside \texttt{EventApproved} logs, yielding a consistent causal ordering without a trusted time authority. This ensures conflict-free execution of EDSCs and state reconciliation across all nodes without traditional consensus protocols. For further details on Lamport-clock ordering in decentralized systems, see~\cite{KII15}.

\subsection{Adaptations for Extended ZKP}
\label{sec:contracts-extended}

In this subsection we describe the minimal extensions to our sidechain, Oracle Committee workflow, and smart contract logic required to support the two new proofs (\(\Sigma_j^t\) and \(\Gamma_j^t\)) in the extended ZKP bundle from \cref{sec:ZKP-extended}. We first explain how the committee computes and publishes the per-epoch thresholds \([L_t,B_t]\) and \(\tau_{\max}\), then detail the single contract update needed to consume them, and finally summarize the unchanged components.

\subsubsection{Threshold Computation and Dissemination}
\label{sec:thresholds}

At the close of epoch \(t-1\), each Oracle Committee member gathers the public inputs from all confirmed bundles in the submitted blocks on the DAG, specifically the \(\ell_{2}\)-norms \(\|\Delta W_k^{\,t-1}\|\) and cosine-similarity scores on the fixed probe set. By taking the median of the norms and scaling it via committee-chosen factors \(r\) and \(\rho\), they fix
\[
  B_t = r\cdot\mathrm{median}\!\bigl\{\|\Delta W_k^{\,t-1}\|\bigr\},
  \quad
  L_t = \rho\,B_t,
  \quad
  0<\rho<1<r.
\]
Concurrently, the 95th percentile of the published similarity scores
\(s_k^{\,t-1} = \cos\!\bigl(z_{\rm old},\,z_{\rm new}\bigr)\) on \(D_{\mathrm{probe}}\) (a small public probe set of, e.g., 300--500 samples that could be periodically re-sampled every several epochs by the committee) is chosen as \(\tau_{\max}\). Because \(D_{\mathrm{probe}}\) is fixed and public, these embeddings cannot be tailored by clients, ensuring \(\tau_{\max}\) reflects genuine semantic shifts.

Finally, the committee emits threshold-signed events
\lstinline{NormThresholdsPublished(epoch,t,L_t,B_t)} and
\lstinline{CosineThresholdsPublished(epoch,t,tau_max)} via the existing
\texttt{EventAdmission} contract. All sidechain validators and
EDSCs subscribe to these \texttt{EventApproved} logs and cache
\([L_t,B_t,\tau_{\max}]\) for use in subsequent proof verifications.

\subsubsection{Extension of the Validation Contract \(S_j^1\)}
\label{sec:validation-contract-extended}

The only smart contract that requires modification to accommodate the extended Z-bundle is the per-node validation contract \(S_j^1\). As before, \(S_j^1\) subscribes to the \texttt{EventApproved} logs emitted by \texttt{EventAdmission}, but it now must also listen for the two threshold-publication events (\lstinline{NormThresholdsPublished} and \lstinline{CosineThresholdsPublished}) at the start of each epoch. Upon seeing those events, \(S_j^1\) caches the new values of \(L_t\), \(B_t\), and \(\tau_{\max}\) in its local state, making them available to all subsequent verification calls in epoch \(t\).

When a candidate bundle \(Z_{j'}^{t,\mathrm{ext}}=(\Pi,\Sigma,\Gamma,\dots)\) arrives, \(S_j^1\) now executes its validation logic in three successive checks, all within the same transaction and using only on-chain precompiles:
\begin{enumerate}[leftmargin=1.5em]
  \item A Groth16 verification of \(\Pi\) against the committed model and test commitments \(C^{t,\mathrm{model}}_{j'}\) and \(C^{t,\mathrm{test}}_{j'}\), ensuring correct inference and loss computation.
  \item A Bulletproof verification of \(\Sigma\), using \(C^{t,\mathrm{model}}_{j'}\) to guarantee that the weight update’s \(\ell_2\)-norm lies in \([L_t,B_t]\).
  \item A SNARK verification of \(\Gamma\), using the public probe commitment \(C_{\mathrm{probe}}\) to enforce \(\cos(z_{\rm old},z_{\rm new})\le\tau_{\max}\).
\end{enumerate}
If—and only if—all three checks succeed, \(S_j^1\) emits the usual \textsf{ProofOK} event and marks the block as valid; any failure causes an immediate reject. Because thresholds are advanced by the oracle-signed events and stored in contract state, no additional transactions or on-chain parameters are required.

After collecting the \textsf{ProofOK} events for all tips \(D_{j'}^T(t)\) received in epoch \(t\), \(S_j^1\) retrieves their associated loss values \(\mathcal L_{j'}^t\), ranks the valid blocks in ascending order of \(\mathcal L\), and selects the top \(K_V\) lowest-loss blocks to form \(\mathcal D_j^V(t)\). The parameter \(K_V\) is a protocol-level constant that determines how many parent contributions each node adopts. Once \(\mathcal D_j^V(t)\) is finalized, node \(j\) proceeds to the submission stage with those selected parents.

All other EDSCs (\(S_j^2\) through \(S_j^5\)) and the \texttt{EventAdmission}/Oracle workflow remain unchanged. Thus, the consensus/proof path for an extended update is:
\[
\begin{aligned}
\text{DAG event}
&\;\longrightarrow\;
\text{Oracle Committee (BLS threshold)} \\[0.5em]
&\;\longrightarrow\;
\texttt{EventApproved} \text{ log on sidechain} \\[0.5em]
&\;\longrightarrow\;
S_j^1 \text{ multi-proof verification} \\[0.5em]
&\;\longrightarrow\;
\text{global aggregation}
\end{aligned}
\]

preserving the same BFT guarantees as in \cref{lem:collusion} while enforcing the extended ZKP constraints.

\section{Extended Experimental Validation and Analysis}
\label{sec:sup_exp}

This section offers additional simulation results to substantiate the paper's claims. It opens with a detailed comparison against the ChainFL baseline from both a learning and distributed ledger perspective. It follows with a technical case study on implementing a GRU model with recursive folding to demonstrate the practicality of ZKPs for complex networks. The section concludes with further benchmarks on system latency, throughput, and scalability, an ablation study on the stake-weighted aggregation rule, and a final experiment that validates the robustness of the extended ZKP defenses.

\subsection{Comparison with ChainFL}
\label{subsec:comparison_chainfl}

Having outlined the workflow of ZK-HybridFL, we now direct our analysis toward ChainFL. While Blade-FL employs blockchain to decentralize FL, ChainFL enhances this approach by integrating a DAG with a sharded architecture. This refined design not only addresses the scalability challenges inherent to blockchain-based systems but also provides a more efficient decentralized framework, making ChainFL the appropriate benchmark for our evaluation. Both ZK-HybridFL and ChainFL share similar guiding objectives: they eliminate the reliance on a central server and adopt DAG-based structures to mitigate the linear block-generation bottlenecks of traditional blockchain ledgers. These broad convergences reflect a shared desire to accommodate large-scale FL among edge or IoT devices, and deliver better security in adversarial or untrusted settings.

In this section, we analyze both schemes from two perspectives: first, the \emph{learning perspective}, focusing on model validation strategies, resilience against adversarial updates, and overall learning efficiency; and second, the \emph{distributed ledger perspective}, focusing on consensus mechanisms, block selection strategies, and scalability. To ground our comparison, we first define two problematic node behaviors: \emph{adversarial} nodes and \emph{lazy} nodes. Adversarial nodes inject subtle noise or degraded parameters into their local model updates, ensuring they can pass superficial validation (i.e., in the case of using a public dataset, the model performs well on the public dataset) but gradually corrupt the global model.

Our threat model follows the ``utility-preserving'' model-poisoning adversary of \cite{Baruch2019,Bagdasaryan2020,Xie2020}. Concretely, let $\operatorname{Acc}(\mathbf{W},D_{pub})$ denote the classification accuracy of model $\mathbf{W}$ on the public validation set $D_{pub}$. We call node $i$ adversarial at round $t$ if its local update $\mathbf{W}_i^t$ simultaneously (i) attains nearly the same public-set accuracy as an honest update, i.e., $\operatorname{Acc}(\mathbf{W}_i^t,D_{pub}) \ge \operatorname{Acc}(\mathbf{W}_{\mathrm{honest}}^t,D_{pub})-\varepsilon$ for a small slack $\varepsilon$ (we use $\varepsilon=0.01$), and (ii) is chosen to pull the federated average $\tilde{\mathbf{W}}^t=\sum_j\omega_j\,\mathbf{W}_j^t$ as far as possible (in $\ell_2$-norm) from the honest-only average $\tilde{\mathbf{W}}_{\mathrm{honest}}^t=\sum_j\omega_j\,\mathbf{W}_{j,\mathrm{honest}}^t$, i.e., to maximize the ``drift'' $\|\tilde{\mathbf{W}}^t-\tilde{\mathbf{W}}_{\mathrm{honest}}^t\|_2$. Such updates ``look good'' on $D_{pub}$ but still cumulatively degrade the global model, a behavior quantified in \cref{sec:adv-public}.

Lazy nodes reduce computational effort by skipping training in some epochs and resubmitting previous updates, slowing convergence and polluting aggregation. Inadequate validation allows these behaviors to persist undetected, undermining the integrity and efficiency of the entire FL process.

The threat analysis in this section is intentionally framed for the \textit{baseline ZK-HybridFL pipeline} described earlier, that is, the version employing the core Groth16 proof bundle without the additional norm-range Bulletproof and embedding-cosine SNARK. This is the variant we compare against Blade-FL and ChainFL in the main body of the paper, since those schemes offer no counterpart to the extended checks. The \textit{extended-ZKP variant}, which augments every block with two additional proofs, is analyzed in detail in \cref{sec:extZKP-setup}, including attacks it prevents and formal proofs of its security benefits.

\subsubsection{Learning Perspective}
\label{subsubsec:learning_perspective}

From the learning standpoint, a fundamental difference arises in how ChainFL and ZK-HybridFL validate model updates and aggregate them into a global model. ChainFL partitions nodes into shards, each governed by a Subchain Leader Node (SLN) that uses Raft consensus to synchronize local training. Once the shard's local model is aggregated, the SLN periodically attaches this model to a mainchain. When a shard needs to integrate models from others, it selects updates from the DAG's \textit{tip} set on the mainchain, relying on a public reference dataset to assess accuracy or loss. While this approach eliminates the single global aggregator, it has major problems.

First, when a public dataset is used for validating models, this approach creates vulnerabilities when lazy nodes are present. A lazy node that resubmits unchanged parameters from a previous epoch may still meet the subchain's performance threshold if its model had acceptable accuracy on the public dataset. Over time, this degrades the effectiveness of the FL process by slowing convergence and increasing redundancy in the model aggregation step.

Second, because ChainFL uses tip models for aggregation (as introduced in \cref{subsection:challenge}), a malicious node can exploit this by performing an \emph{orphanage attack}. Using tip models for global aggregation means nodes are obtaining global models on which the network has not yet reached consensus regarding validity. Thus, in an orphanage attack, the attacker node deliberately chooses some of their own previously submitted blocks as parents for new blocks, affecting how the new global model is constructed, as this new global model becomes an amalgamation of both valid models and compromised models. Consequently, the local updated model, trained on this amalgamated global model, may still pass validation checks using a public dataset; however, its degraded quality will gradually undermine the overall integrity and performance of the network. Over time, this cycle prevents the network from converging toward a high-quality model and may cause it to settle at a low-performance, locally optimal state from which it cannot recover (see \cref{sec:results}).

By contrast, ZK-HybridFL removes any reliance on a public reference dataset by enforcing an inference--validation pipeline built on KZG commitments and non-interactive SNARKs. Rather than aggregating every tip that happens to be visible, the protocol only ever considers confirmed blocks whose proofs have passed on-chain verification. In each epoch $t$, trainer $j$ first publishes the KZG commitments $C^{t,\mathrm{model}}_{\,j}=\mathsf{Commit}(\mathbf{W}_{j}^{t})$ and $C^{t,\mathrm{test}}_{\,j}=\mathsf{Commit}(D_{j}^{t,\mathrm{test}})$, irrevocably binding its flattened weight vector $\mathbf{W}_{j}^{t}$ and private test batch $D_{j}^{t,\mathrm{test}}$. It then generates a Groth16 proof $\Pi^{t}_{j}$ attesting that the committed model achieves the stated loss on the committed data. Only after that proof is verified on-chain and the block's aggregated weight exceeds the network threshold does the update influence the global model.

\textbf{Lazy-node defense.} If a trainer attempts to resubmit its previous weights $\mathbf{W}_{j}^{t}$ in epoch $t+1$, the new commitment $C^{t+1,\mathrm{model}}_{\,j}$ will be bit-for-bit identical to $C^{t,\mathrm{model}}_{\,j}$. The validation contract $S_j^1$ detects the duplicate and discards the block before it can be ranked, so stale updates never propagate.

\textbf{Adversarial accuracy spoofing.} Two natural attack vectors are neutralized. First, if an attacker proves with a degraded model $\widetilde{\mathbf{W}}$, the resulting high loss is visible in the proof bundle and the block is filtered out during parent selection. Second, if an attacker runs inference with honest weights $\mathbf{W}$ but publishes a commitment $C^{t,\mathrm{model}}_{\,j}=\mathsf{Commit}(\widetilde{\mathbf{W}})$, the KZG opening contained in the proof cannot match that commitment. On-chain verification therefore fails and the block is rejected. Because invalidated blocks carry no aggregated weight, only fresh, correctly computed updates ever enter the global aggregation. This enforcement provides significantly stronger robustness than ChainFL, as we empirically demonstrate in \cref{sec:results}.

\subsubsection{Distributed Ledger Perspective}
\label{subsubsec:ledger_perspective}

The ledger architecture and consensus mechanism are critical to system performance in terms of latency, scalability, and throughput. In ChainFL, the use of Raft-based subchains centralizes coordination in SLNs, which manage local shard operations and require explicit cross-shard synchronization. This centralization introduces several drawbacks. If an SLN becomes adversarial or is overwhelmed, its corresponding shard suffers from degraded performance, which in turn creates a bottleneck for the entire network. Such centralization not only poses security risks---since a malicious SLN can manipulate or delay the aggregation and propagation of updates---but also increases latency because the coordinated, synchronous agreement required by Raft consensus limits throughput and scalability.

In contrast, ZK-HybridFL is designed from the ground up to combine DAG-based concurrency with cryptographic verification, sidechain-based event logic, and oracle-assisted validation. This holistic co-design offers several advantages over ChainFL. First, by offloading resource-intensive tasks such as ZKP verification to dedicated sidechains governed by event-driven smart contracts, ZK-HybridFL executes its consensus effectively. Second, the decentralized validation distributed among multiple oracles and smart contracts avoids the large-scale coordination lags inherent to a multi-layer, leader-based mechanism, thereby keeping the DAG agile and uncluttered. This architecture not only reduces latency and improves throughput but also scales more efficiently with network size. Each contract is triggered only under specific circumstances, diminishing reliance on any single trusted node. Overall, ZK-HybridFL's distributed ledger design delivers a more granular security model and higher performance, establishing clear advantages over the SLN-dependent, Raft-based approach of ChainFL. We demonstrate the superior performance of ZK-HybridFL over ChainFL in the context of distributed ledger through simulations provided in the supplementary material.

\Cref{tab:comparison_compact} summarizes the analysis between ZK-HybridFL and ChainFL from learning and distributed ledger perspectives.

\begin{table}[t]
\centering
\caption{Comparison of ChainFL and ZK-HybridFL}
\label{tab:comparison_compact}
\resizebox{\columnwidth}{!}{%
  \renewcommand{\arraystretch}{1.2}
  \begin{tabular}{|c|p{3.5cm}|p{3.5cm}|}
    \hline
    \textbf{Aspect} & \centering\textbf{ChainFL} & \centering\textbf{ZK-HybridFL} \tabularnewline \hline
    \multicolumn{3}{|c|}{\textbf{Learning Perspective}} \\ \hline
    \textit{Validation \& Aggregation} 
        & \begin{itemize}[leftmargin=*,nosep]
            \item Public dataset
            \item Tip-based selection
          \end{itemize}
        & \begin{itemize}[leftmargin=*,nosep]
            \item ZKPs
            \item Aggregation from confirmed models
          \end{itemize} \\ \hline
    \textit{Handling Lazy Nodes} 
        & \begin{itemize}[leftmargin=*,nosep]
            \item Unchanged model can pass if accuracy remains acceptable
          \end{itemize}
        & \begin{itemize}[leftmargin=*,nosep]
            \item Committed KZG checks prevent replay
          \end{itemize} \\ \hline
    \textit{Handling Adversarial Updates} 
        & \begin{itemize}[leftmargin=*,nosep]
            \item Subtle noise can slip through
            \item Orphanage attacks exploit tip models
          \end{itemize}
        & \begin{itemize}[leftmargin=*,nosep]
            \item Inference--model mismatch triggers rejection
            \item Insufficient AW halts malicious updates
          \end{itemize} \\ \hline
    \multicolumn{3}{|c|}{\textbf{Distributed Ledger Perspective}} \\ \hline
    \textit{Consensus Design} 
        & \begin{itemize}[leftmargin=*,nosep]
            \item Raft-based subchains with leader nodes
            \item Requires cross-shard sync
          \end{itemize}
        & \begin{itemize}[leftmargin=*,nosep]
            \item DAG with sidechain event logic
            \item Oracles and EDSCs
          \end{itemize} \\ \hline
    \textit{Scalability \& Performance}
        & \begin{itemize}[leftmargin=*,nosep]
            \item SLN is a bottleneck
            \item Centralized coordination increases latency
          \end{itemize}
        & \begin{itemize}[leftmargin=*,nosep]
            \item Decentralized validation
            \item Lower latency, higher throughput
          \end{itemize} \\ \hline
  \end{tabular}%
}
\end{table}

\subsection{Zero-Knowledge GRU Inference with Recursive Folding}
\label{sec:gru_nova}

Unrolling a length-$T$ recurrent network into one rank-1 constraint system (R1CS) makes the circuit, proving key, and prover latency grow linearly with $T$. For $T > 32$, this already exceeds practical memory budgets. \emph{Recursive folding} (\textsc{Nova}) compresses any sequence of identical R1CS instances into a single relaxed instance whose proof size and verifier cost no longer depend on $T$ \cite{SoNovaFold}. We implement a GRU-256 time-step as a Halo2 circuit and apply Nova folding, following the publicly documented 512-layer ``Zator'' prototype \cite{ZatorGithub}. \textit{The entire GRU path is implemented from first principles and does not rely on pre-existing recurrent-layer support in external zkML tool-chains.}

\subsubsection{Circuit construction and quantization}

Weights and activations are quantized to a $(16+8)$-bit fixed-point format. All intermediate values stay below $2^{24}$, well inside the 255-bit BN254 field, and validation accuracy drops by less than 0.3~pp. One GRU step processes $(x_t,h_{t-1})$ as
\begin{align}
  z_t &= \sigma(W_zx_t + U_zh_{t-1}+b_z),\\
  r_t &= \sigma(W_rx_t + U_rh_{t-1}+b_r),\\
  \tilde h_t &= \tanh(W_hx_t + U_h(r_t\odot h_{t-1})+b_h),\\
  h_t &= (1-z_t)\odot\tilde h_t + z_t\odot h_{t-1}.
\end{align}
Sigmoid and $\tanh$ are evaluated with degree-7 Remez polynomials (implemented as one Halo2 custom gate via \texttt{zkFixedPointChip}). With 128 input and 256 hidden dimensions, the step circuit contains $118\,062$ constraints (two matrix-vector products $65\,536$, point-wise arithmetic $51\,200$, non-linearities $1\,326$) and yields a \SI{4.3}{\giga\byte} proving key.

\subsubsection{Performance on the common hardware baseline}

All timings were re-measured on the same platform used for \cref{tab:proof-cost,tab:verify-cost}: a 16-core Intel Xeon Gold 6338 (\SI{3.0}{\giga\hertz}, \SI{64}{\giga\byte} RAM) plus one NVIDIA A100 (\SI{80}{\giga\byte} HBM). Witness generation for one GRU step consumes \SI{0.44}{\second} on the GPU; the CPU fold adds \SI{0.35}{\second}, and key initialization costs \SI{8.6}{\second}. Prover latency therefore follows
\[
  t_{\text{prove}}(T) = 8.6 + 0.79\,T~\text{s}.
\]
A $T=10$ sequence---matching the mini-batch size $B=10$ used in \cref{tab:proof-cost}---is proved in \SI{16.5}{\second}, $4.4\times$ faster than the \SI{44}{\second} flat Groth16 baseline. Folding still completes in \SI{59}{\second} at $T=64$ and about \SI{210}{\second} at $T=256$, confirming linear scaling.

The folded proof is \SI{1.4}{\kilo\byte} and verifies in \SI{7.9}{\milli\second}, much faster than the slowest entry in \cref{tab:verify-cost}. Nova adds fewer than \SI{50}{\mega\byte} of transcript data, so peak memory remains the same \SI{4.3}{\giga\byte} already budgeted for the flat GRU.

\subsubsection*{Reference implementation}
The inner Halo2 circuit uses a universal $2^{20}$-point KZG SRS; the outer Nova proof is transparent. The code snippet below illustrates the logic:

\begin{lstlisting}[language=Solidity,basicstyle=\ttfamily\footnotesize]
// Conceptual GRU Cell Logic
template GRUCell(d,h){
  signal input  x[d];
  signal input  h_prev[h];
  signal output h_next[h];

  component uz[h], ur[h], uh[h];
  for (var i = 0; i < h; i++){
    uz[i] = SigmoidGate();
    uz[i].in <== dot(Wz[i], x) + dot(Uz[i], h_prev) + bz[i];

    ur[i] = SigmoidGate();
    ur[i].in <== dot(Wr[i], x) + dot(Ur[i], h_prev) + br[i];

    uh[i] = TanhGate();
    uh[i].in <== dot(Wh[i], x)
                 + dot(Uh[i], elemMul(ur[i].out, h_prev))
                 + bh[i];

    h_next[i] <== uz[i].out * h_prev[i]
                + (1 - uz[i].out) * uh[i].out;
  }
}
\end{lstlisting}

Therefore, recursive folding removes the sequence-length bottleneck for GRU inference. On the standard Xeon 6338 + A100 node, a folded GRU-256 proof for ten steps is produced in \SI{16.5}{\second}, verifies in \SI{7.9}{\milli\second}, and fits in \SI{1.4}{\kilo\byte}. Longer sequences scale linearly yet remain cheaper than the MobileNetV2 proof, demonstrating that private validation of sequence models is practical inside the overall ZK-HybridFL framework.

\subsection{Additional Simulations}
\label{very_tired}

\subsubsection{Analysis of Stake-Weighted Aggregation}
\label{sec:weighting-ablation}

To provide quantitative evidence regarding the benefits of the stake-weighted aggregation proposed in ZK-HybridFL, we perform an ablation study comparing our on-chain stake-weighted aggregation rule against a vanilla uniform average under the same node pool, data partitions, and fault rates described in \cref{sec:results}. \Cref{fig:all_panels_1} plots the loss, accuracy, and perplexity curves for both aggregation rules on two canonical tasks and reports fault tolerance under a mixed adversarial and lazy client setting.

In the federated vision experiment on MNIST (Task~1), stake weighting drives rapid convergence: the training loss drops from $1.41$ to $0.12$ within $100$ epochs, whereas uniform averaging plateaus near $0.30$. This translates to a final accuracy of $97\%$ under stake weighting, compared to $91\%$ with uniform averaging, a gap of $6.6$ percentage points (\cref{fig:loss,fig:accuracy}). Similarly, in the language modeling experiment on Penn Treebank (Task~2), stake-weighted aggregation achieves a test perplexity of around $118$ at convergence, approximately $12$ points lower than uniform averaging (\cref{fig:perplexity}), indicating consistently stronger generalization.

To stress-test robustness, we introduce a mixed-fault regime in which $30\%$ of clients are faulty (evenly split between Byzantine adversarial and lazy behaviors, $\gamma = \mu = 0.15$). Under this setting, uniform averaging suffers a sharp drop in test accuracy to around $72\%$, whereas stake weighting retains around $88\%$ final accuracy (\cref{fig:robustness}). This improvement stems from our dynamic slashing mechanism: any client update failing its ZK proof is rejected and that client's stake is halved ($\lambda = 0.5$), while honest contributors accrue stake at rate $\eta = 0.05$ per accepted update before normalization.

\begin{figure}[t]
  \centering
  \subfloat[Task\,1 Loss\label{fig:loss}]{%
    \includegraphics[width=\columnwidth]{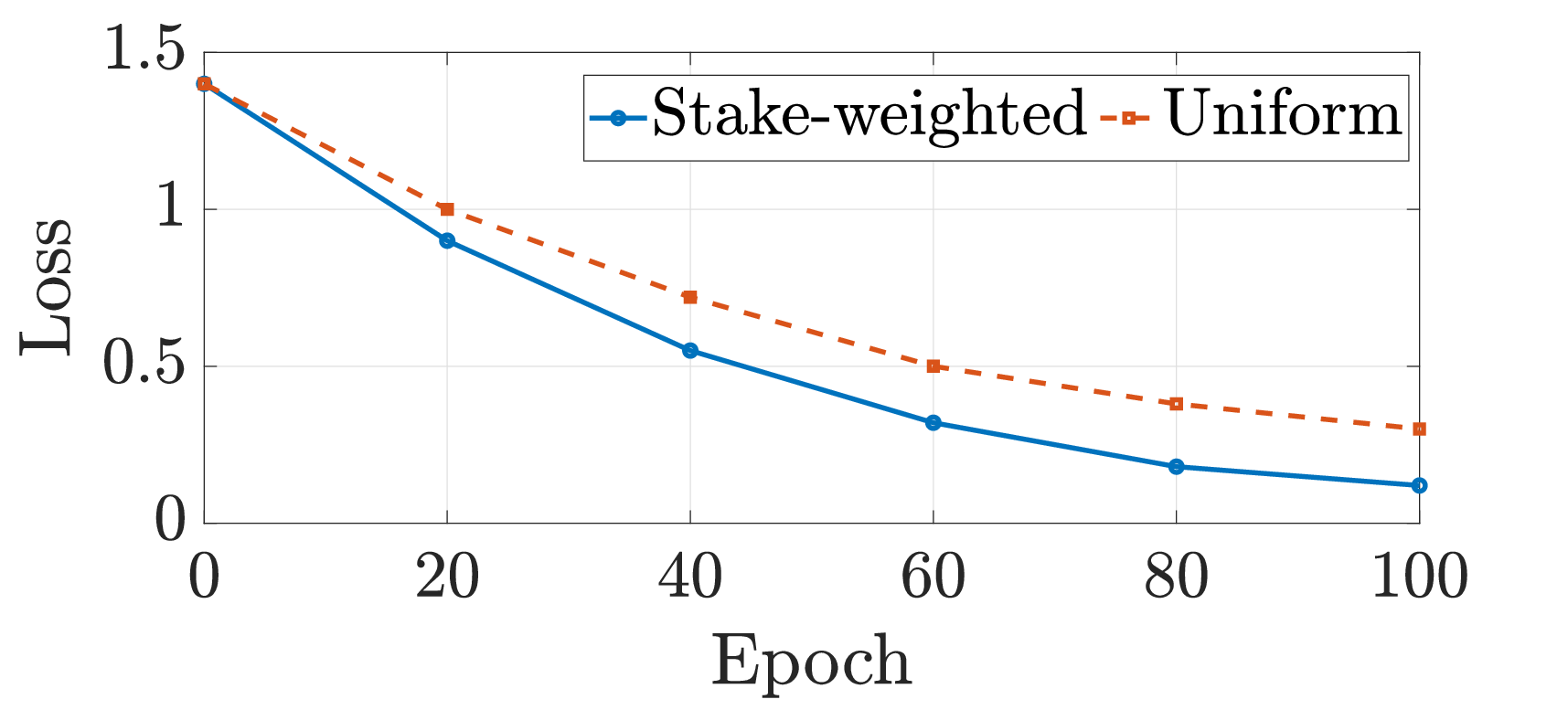}%
  }\\[-0.25\baselineskip]
  \subfloat[Task\,1 Accuracy\label{fig:accuracy}]{%
    \includegraphics[width=\columnwidth]{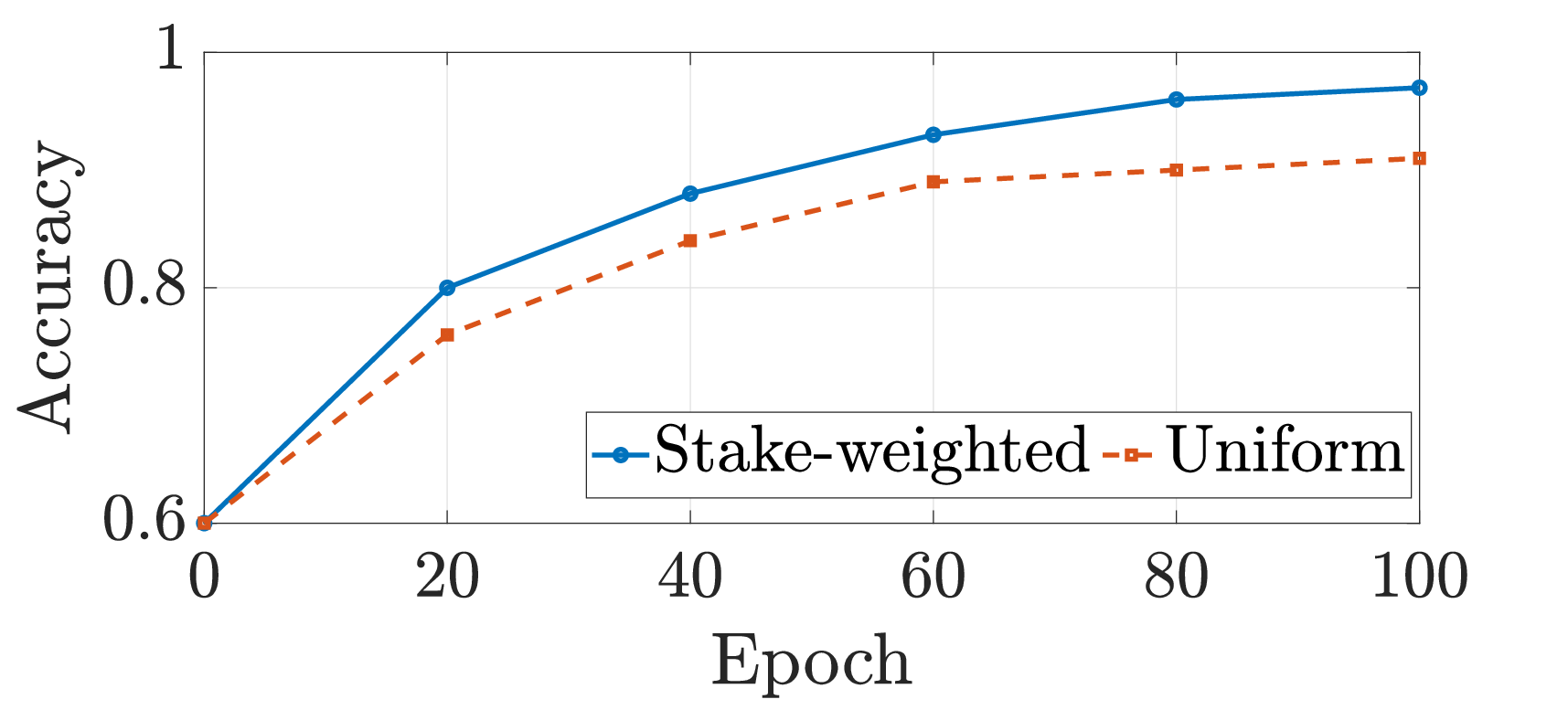}%
  }\\[-0.25\baselineskip]
  \subfloat[Task\,2 Perplexity\label{fig:perplexity}]{%
    \includegraphics[width=\columnwidth]{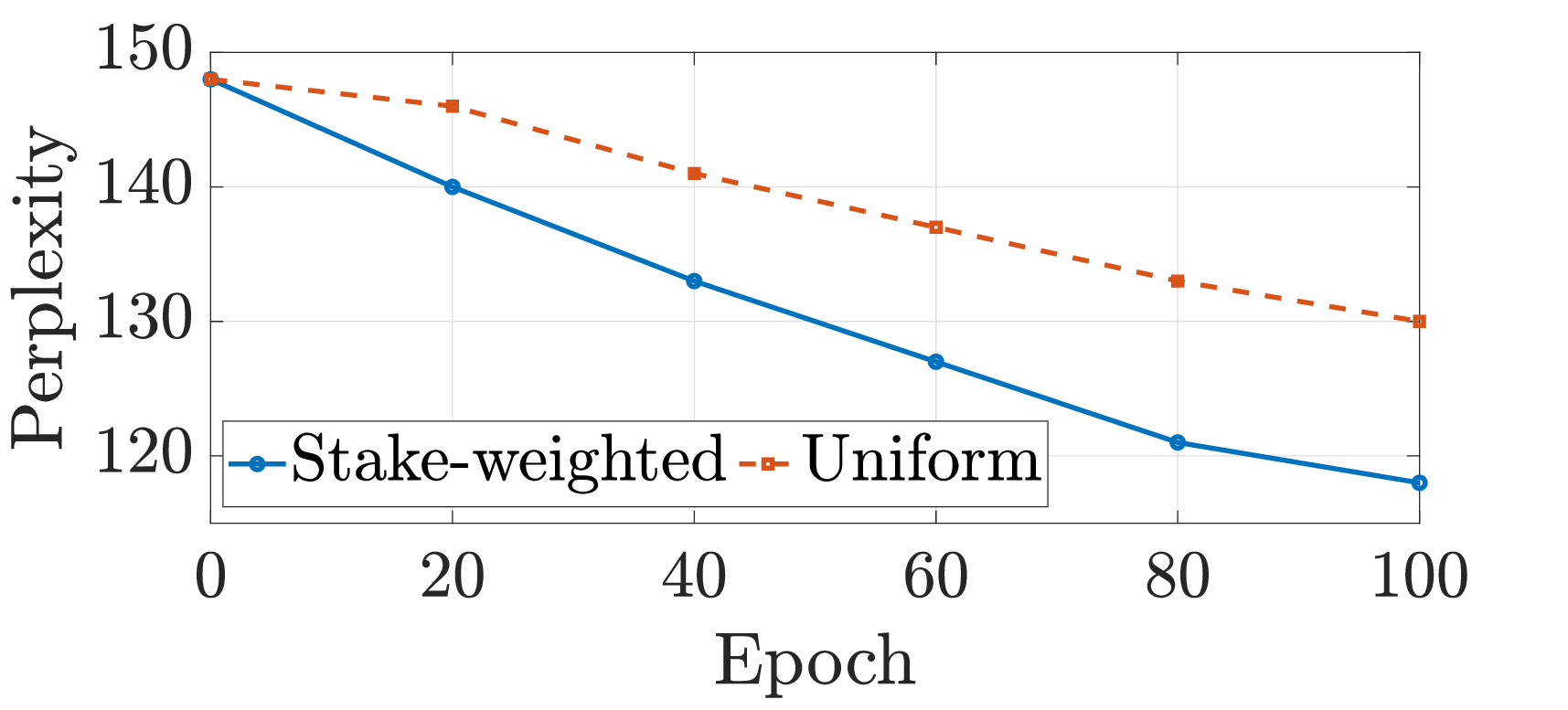}%
  }\\[-0.25\baselineskip]
  \subfloat[Mixed-fault Robustness\label{fig:robustness}]{%
    \includegraphics[width=\columnwidth]{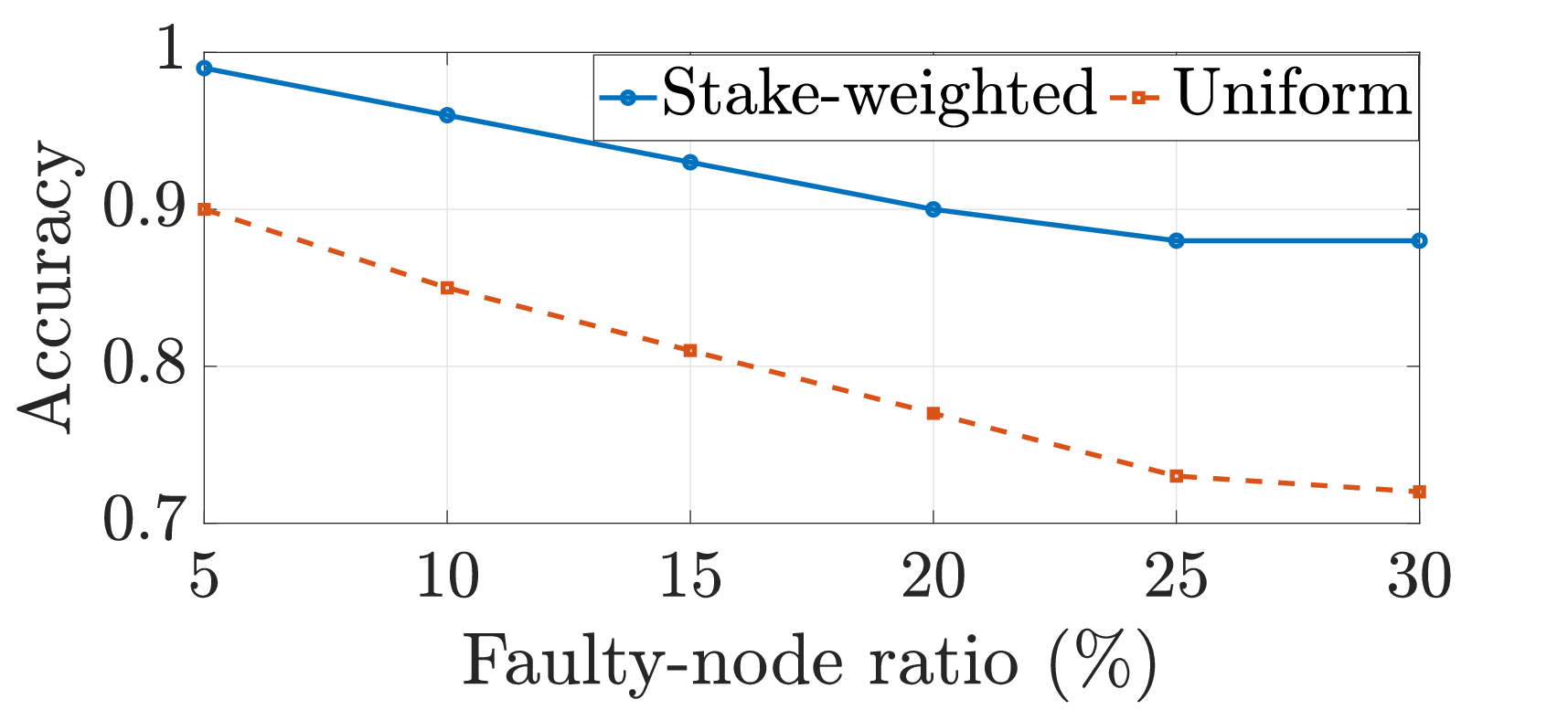}%
  }
  \caption{Stake-Weighted vs Uniform Averaging ($\gamma=0.15,\ \mu=0.15,\ n=15$).}
  \label{fig:all_panels_1}
\end{figure}

\subsubsection{Extended-ZKP Threat Model and Simulation Setup}
\label{sec:extZKP-setup}

In \cref{sec:ZKP-extended}, we introduced two lightweight ZKP checks: a Bulletproof enforcing that each client's update norm stays within a dynamic window $[L_t,B_t]$, and an embedding-cosine SNARK ensuring the model change is sufficiently large in feature space. These are designed to catch subtle ``stealth'' updates that pass the main Groth16 proof yet contribute almost no learning. In our original simulations (\cref{sec:results}), we already exercised the exact-replay lazy attack, which is detected by a hash collision at the Groth16 stage. Here we focus on three new variants that all pass Groth16 but should be flagged by the extended bundle. When comparing against ChainFL and BladeFL, we retain the same lazy and adversarial behaviors from \cref{sec:simulations}, effectively giving those protocols a less stringent attack profile, since they do not implement ZKP-based replay, norm, or cosine checks.

First, we model a \emph{perturb-replay lazy} node. It resubmits a previous weight tensor $\mathbf{W}_j^{\,t-\Delta}$ (with $\Delta$ sampled uniformly from $\{1,2,3\}$) plus a tiny noise vector $\varepsilon \sim \mathcal{N}(0,\sigma^2 I)$, where $\sigma=10^{-5}\,\|\mathbf{W}\|_2$. This ensures the update hash differs while $\|\varepsilon\|_2 \ll L_t$, so it bypasses replay detection but should violate the Bulletproof lower-bound check.

Next, in the \emph{minimal-norm stalling} attack, a node runs one honest SGD step $\delta\mathbf{W}$ and then rescales it by $\alpha = L_t / \|\delta\mathbf{W}\|_2$, so it exactly meets the $\ell_2$ bounds without meaningful progress.

Finally, the \emph{semantic-stalling} attack performs full local training to produce $\mathbf{W}_{\mathrm{tmp}}$, cherry-picks (or synthesizes) a ``friendly'' private test batch so that the inference proof passes, but tweaks its update so the cosine similarity of public-probe embeddings remains above the threshold $\tau_{\max}$, thus violating the embedding SNARK.

\Cref{tab:ext-attacks} summarizes which proof each variant triggers. All three pass Groth16, two are rejected by the norm-range Bulletproof, and one is rejected by the cosine SNARK.

\begin{table}[t]
  \centering
  \caption{Extended-ZKP attack variants.}
  \label{tab:ext-attacks}
  \resizebox{\columnwidth}{!}{%
    \begin{tabular}{lcc}
      \toprule
      Attack & Rejected by extended checks & Passes core Groth16? \\
      \midrule
      Perturb-replay lazy      & Bulletproof lower-bound & Yes \\
      Minimal-norm stalling    & Bulletproof lower-bound & Yes \\
      Semantic stalling        & Cosine SNARK            & Yes \\
      \bottomrule
    \end{tabular}
  }
\end{table}

All other learning and network parameters mirror the original setup in \cref{sec:simulations} (batch size $B=50$, local steps $R=5$, SGD step-size $\eta_{\mathrm{SGD}}=0.01$). For the extended checks, we compute the upper norm bound as $B_t = r \cdot \mathrm{median}\{\|\Delta\mathbf{W}\|\}$ with $r=1.8$, and set the lower bound $L_t = \rho\,B_t$ with $\rho=0.2$. The cosine threshold $\tau_{\max}$ is taken as the 95th percentile of pairwise embedding similarities on the fixed public probe set ($|D_{\mathrm{probe}}|=400$), yielding $\tau_{\max}=0.98$ for both tasks. Finally, we sweep adversarial fraction $\gamma$ and lazy fraction $\mu$ over $\{0,0.05,\dots,0.30\}$ with $n=15$ nodes, exactly as before.

\Cref{fig:ext-bundle-mixed} shows that, with $30\%$ faulty clients ($15\%$ lazy and $15\%$ adversarial), the extended bundle pushes the MNIST loss down to $0.18$ and lifts accuracy to $0.96$ by epoch $100$. Under the same conditions, ChainFL levels off at $0.43$ loss and $0.82$ accuracy, while Blade-FL stalls at $0.75$ loss and $0.69$ accuracy (\cref{fig:ext-loss,fig:ext-acc}). On the language task, the bundle lowers perplexity from $132.1$ to $118.9$ (\cref{fig:ext-ppl}). The robustness sweep in \cref{fig:ext-robust} reveals a widening gap: at $30\%$ total faults, the bundle still retains $0.88$ accuracy, which is $18$ percentage points higher than Blade-FL.

Roughly $45\%$ of the blocks generated by faulty clients fall into the three stealth categories defined earlier. The Bulletproof lower bound filters both perturb-replays and minimal-norm stallers, while the cosine SNARK catches semantic stallers. Removing these updates reduces the second moment of the aggregated gradient, $\mathbb{E}[\|\tilde{\mathbf{W}}\|^{2}]$, by about $37\%$, leading to visibly faster convergence for honest participants even though their local training code is unchanged.

Verifying the extra proofs adds at most \SI{23}{\milli\second} to the per-update latency on the same hardware (baseline Groth16 verification already costs \SI{40}{\milli\second}), keeping end-to-end training throughput well within the bounds reported in \cref{sec:results}.

\begin{figure}[t]
  \centering
  \subfloat[Task 1 Loss\label{fig:ext-loss}]{%
    \includegraphics[width=\columnwidth]{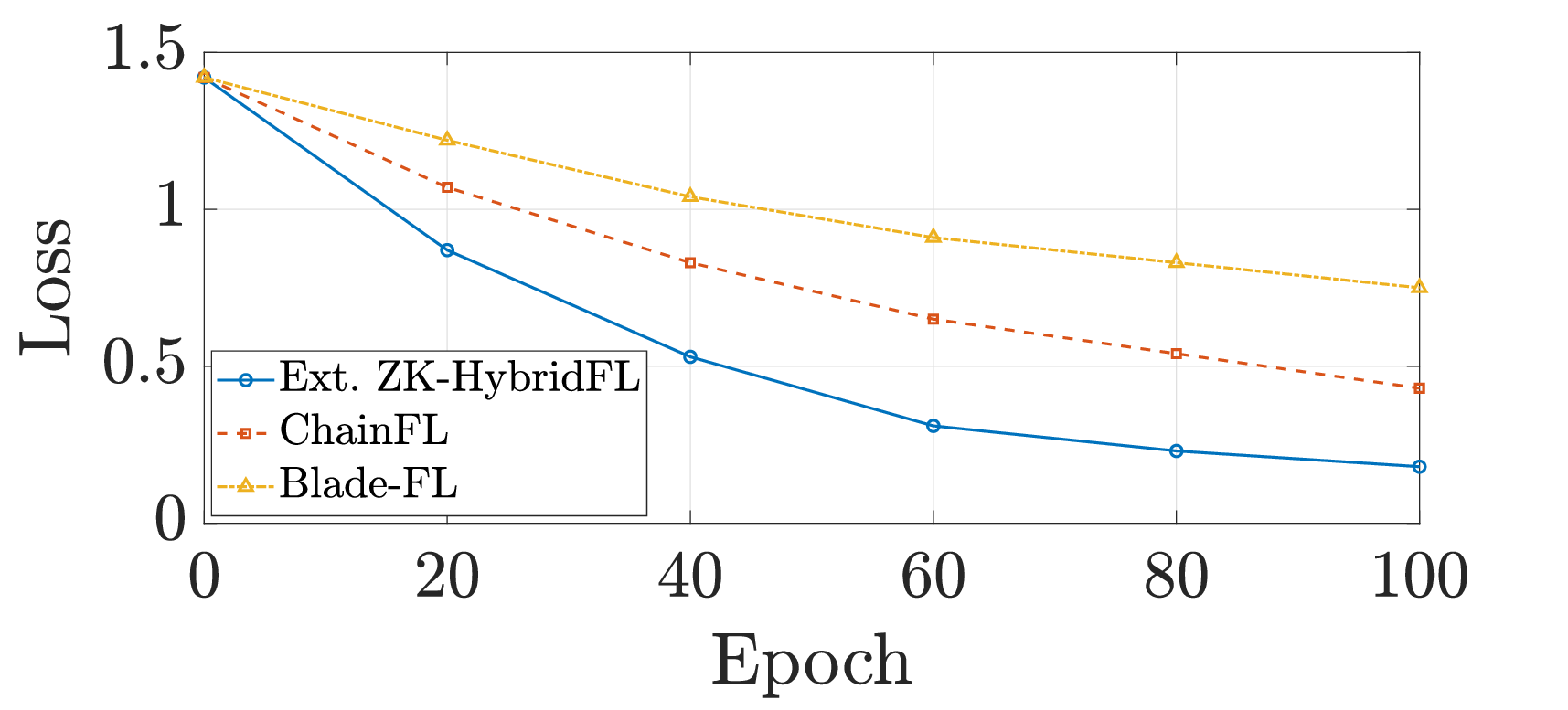}%
  }\\[-0.25\baselineskip]
  \subfloat[Task 1 Accuracy\label{fig:ext-acc}]{%
    \includegraphics[width=\columnwidth]{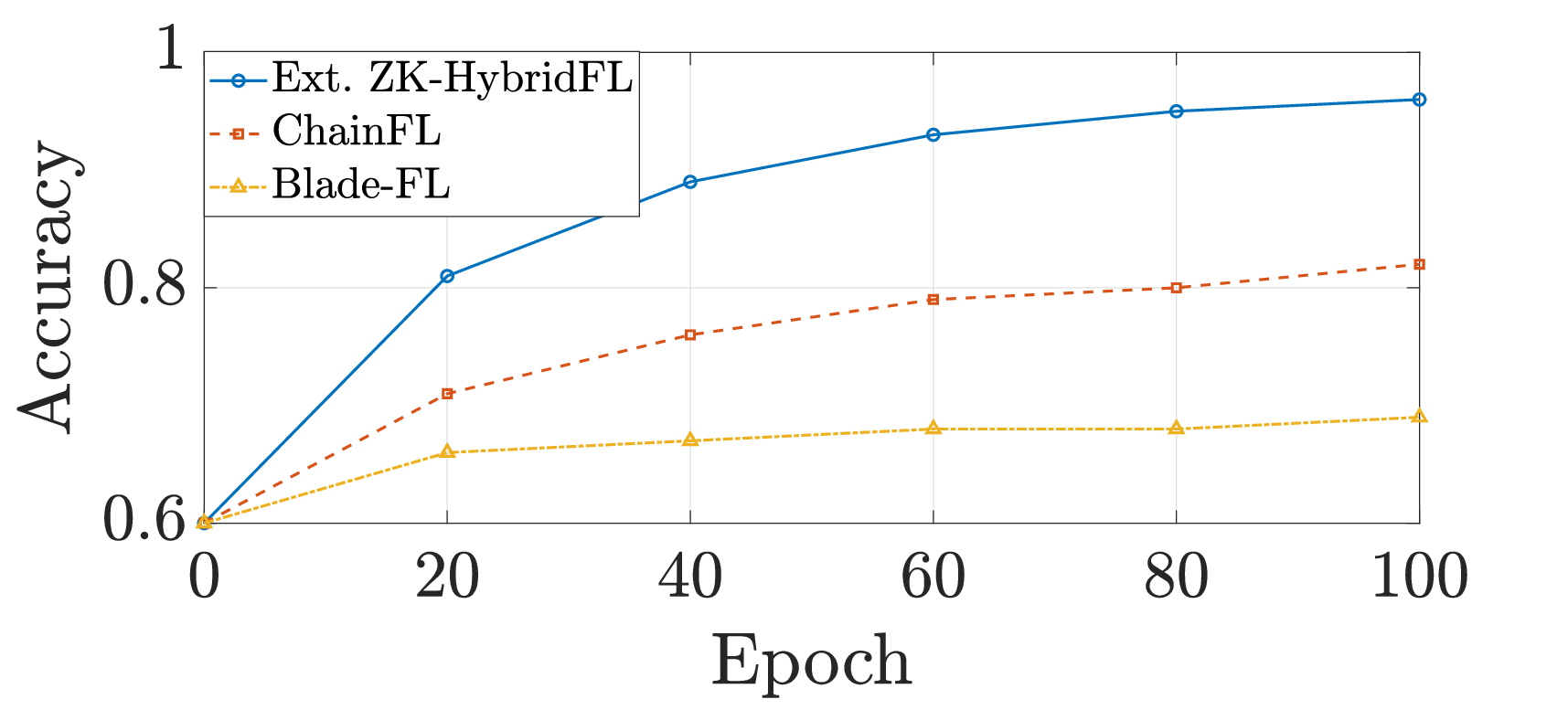}%
  }\\[-0.25\baselineskip]
  \subfloat[Task 2 Perplexity\label{fig:ext-ppl}]{%
    \includegraphics[width=\columnwidth]{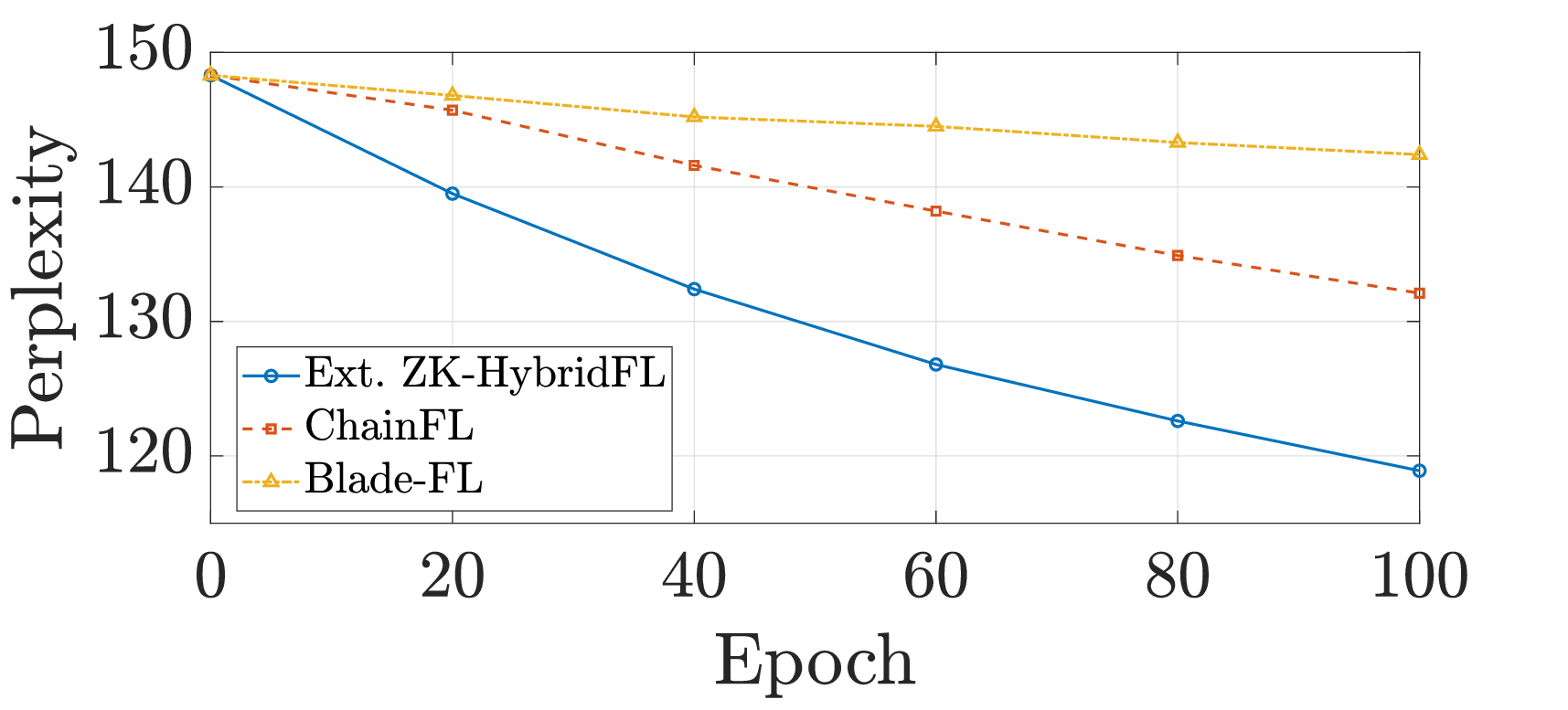}%
  }\\[-0.25\baselineskip]
  \subfloat[Mixed-fault Robustness\label{fig:ext-robust}]{%
    \includegraphics[width=\columnwidth]{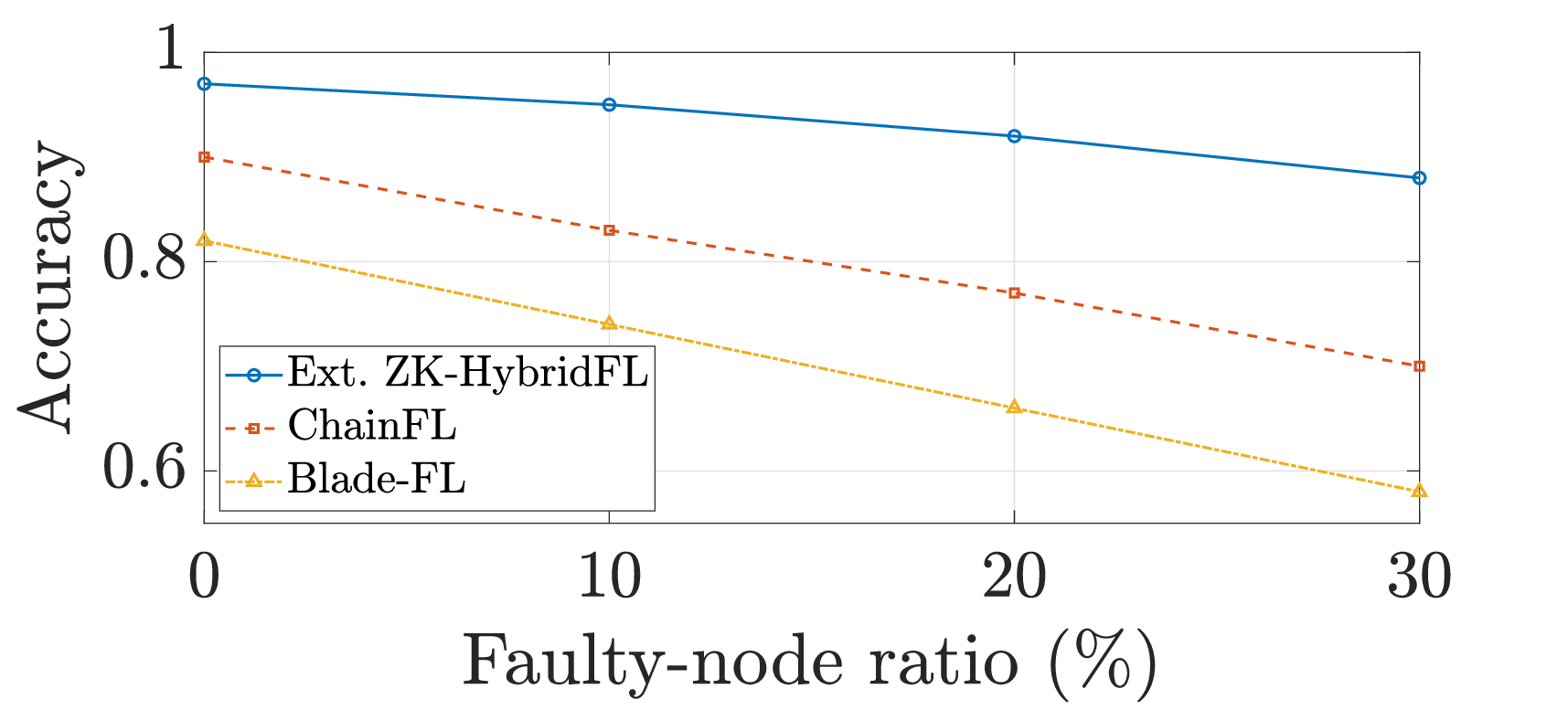}%
  }
  \caption{Extended proof bundle under a mixed-fault setting ($\gamma=\mu=0.15$). Solid: extended ZK-HybridFL; dashed: ChainFL; dash-dot: Blade-FL.}
  \label{fig:ext-bundle-mixed}
\end{figure}

\Cref{fig:proof-audit-det} tracks how the extended bundle filters covert behavior. Across $100$ epochs, the five faulty nodes submit $500$ candidate updates that are \emph{intended} to look legitimate. The Bulletproof window eliminates $385$ of them, while the cosine SNARK removes a further $96$, for a combined detection rate of $481/500 \approx 96.2\%$. The residual $19$ blocks, about $3.8\%$ of the total, have $\ell_{2}$ norms within $3\%$ of the lower bound and embedding shifts just below the cosine threshold, making them statistically indistinguishable from the slowest honest learners.

\Cref{fig:proof-audit-stake} shows how this filters through to economic weight. Training starts from an approximately equal stake split, but as invalid updates are rejected and slashed, the honest share rises monotonically from $0.50$ to $0.85$, while the faulty share falls to $0.15$. Because stake is re-normalized after every reward or slash event, the two curves always sum to one. The progressive re-weighting further dampens the impact of any stealth update that survives the proof checks, which explains the widening performance gap already visible in \cref{fig:ext-bundle-mixed}.

\begin{figure}[t]
  \centering
  \subfloat[Cumulative stealth updates rejected by each extended check.\label{fig:proof-audit-det}]{%
    \includegraphics[width=0.98\columnwidth]{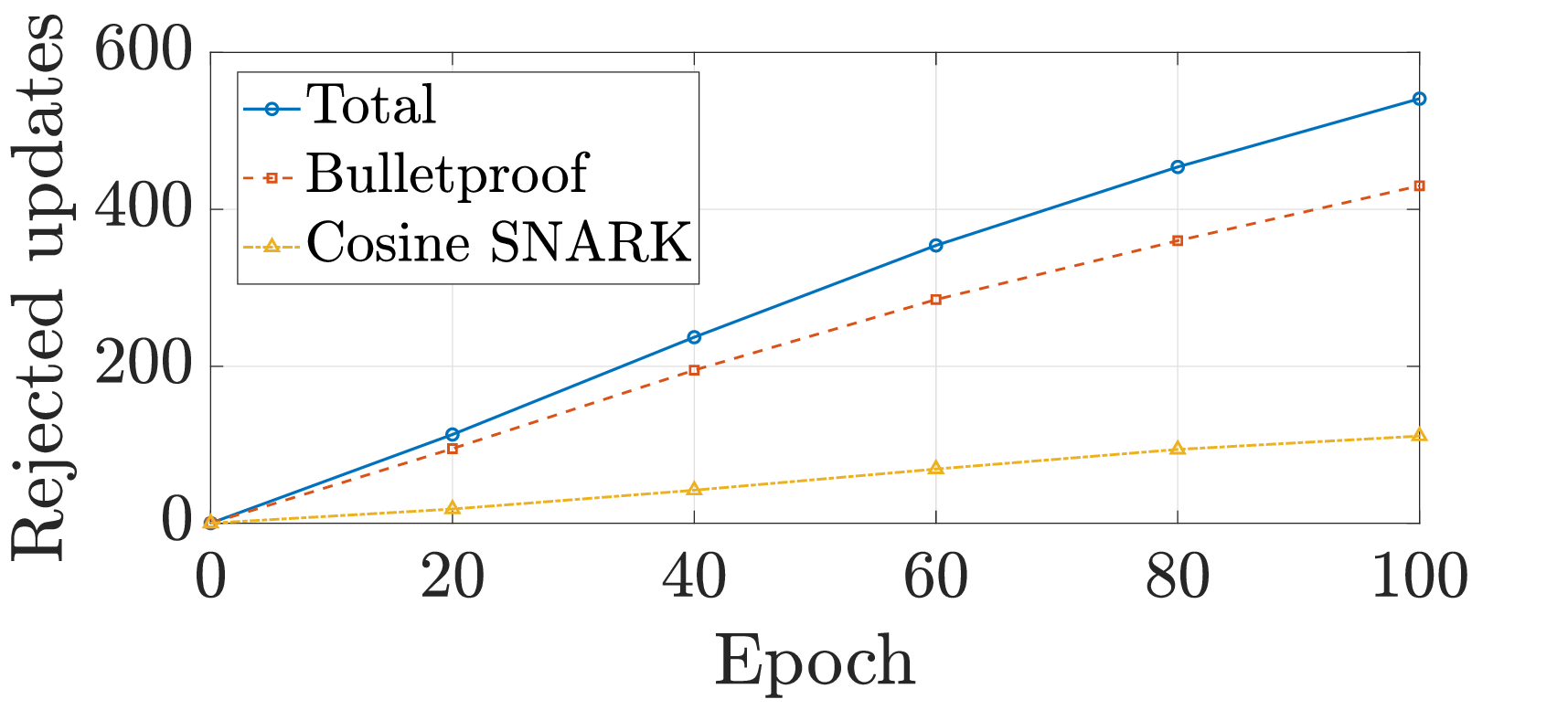}%
  }\\[0.5em]
  \subfloat[Stake redistribution between honest and faulty clients.\label{fig:proof-audit-stake}]{%
    \includegraphics[width=0.98\columnwidth]{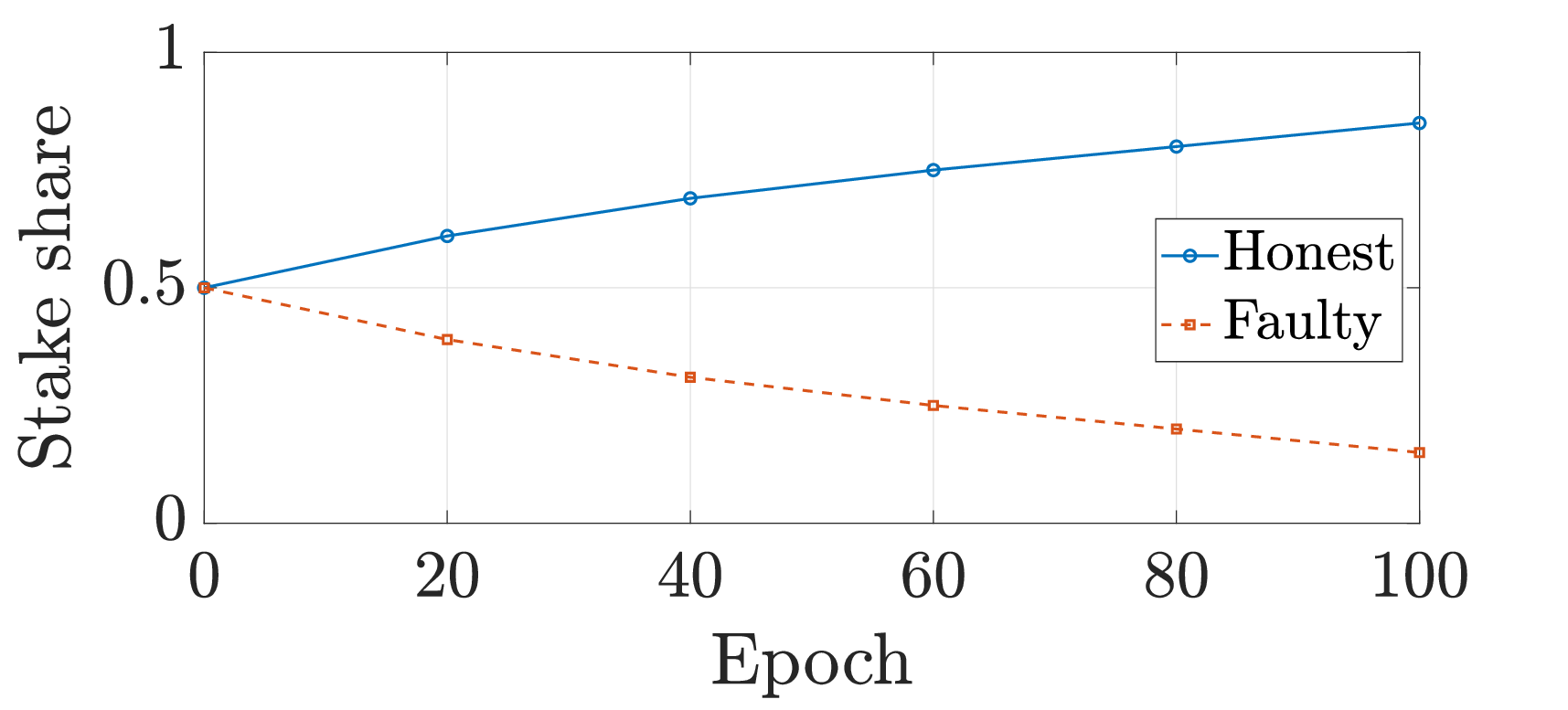}%
  }
  \caption{Proof-level audit of the mixed-fault run ($\gamma=\mu=0.15$; five faulty nodes, one update per round).}
  \label{fig:proof-audit}
\end{figure}

\subsubsection{Lazy-Only Robustness Analysis}
\label{sec:lazy-only-analysis}

In this subsection, we isolate the effect of purely lazy clients---those that replay stale models without any adversarial tampering---by setting $\gamma=0$ and varying the skip percentage $\mu$ from $0$ to $0.30$. This ``lazy-only'' sweep shows how ZK-HybridFL resists the impact of missing updates compared to ChainFL and Blade-FL, providing a clear baseline for understanding the protocol's robustness under non-malicious service lapses.

When adversaries are disabled ($\gamma = 0$), as \cref{fig:lazy-only} illustrates, ZK-HybridFL's MNIST accuracy declines smoothly from $0.992$ at $\mu=0$ to $0.949$ at $\mu=0.30$, and its perplexity rises from $117.32$ to $124.02$, an increase of $6.70$ points. By contrast, ChainFL's perplexity jumps from $117.32$ to $191.39$ ($+74.07$) and Blade-FL's from $117.32$ to $255.58$ ($+138.26$). This finer-granularity view shows that, although all protocols degrade under pure laziness, the stake-based re-weighting in ZK-HybridFL dramatically limits the performance loss compared to equal-weight baselines.

\begin{figure}[t]
  \centering
  \subfloat[Task 1 Accuracy\label{fig:lazy-acc}]{%
    \includegraphics[width=\columnwidth]{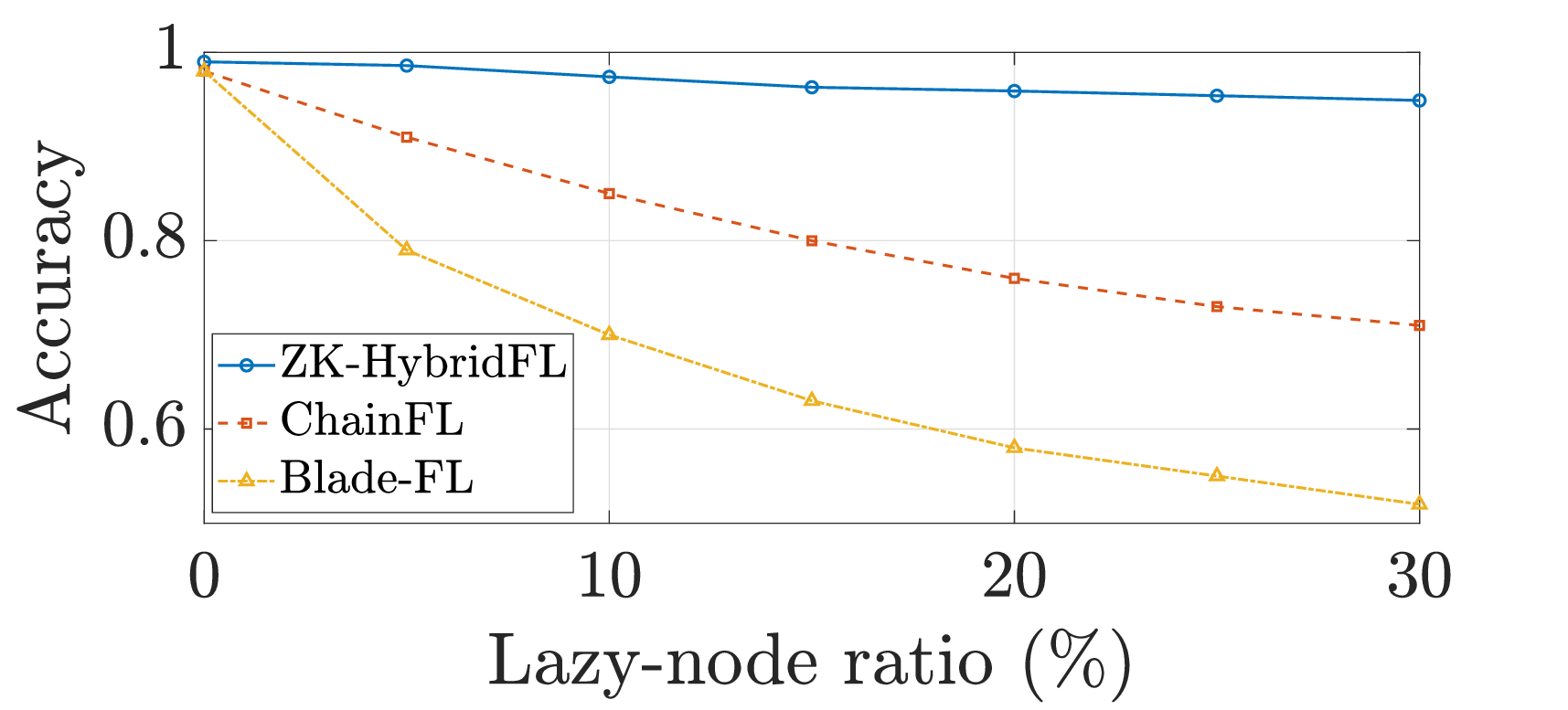}%
  }\\[-0.25\baselineskip]
  \subfloat[Task 2 Perplexity\label{fig:lazy-ppl}]{%
    \includegraphics[width=\columnwidth]{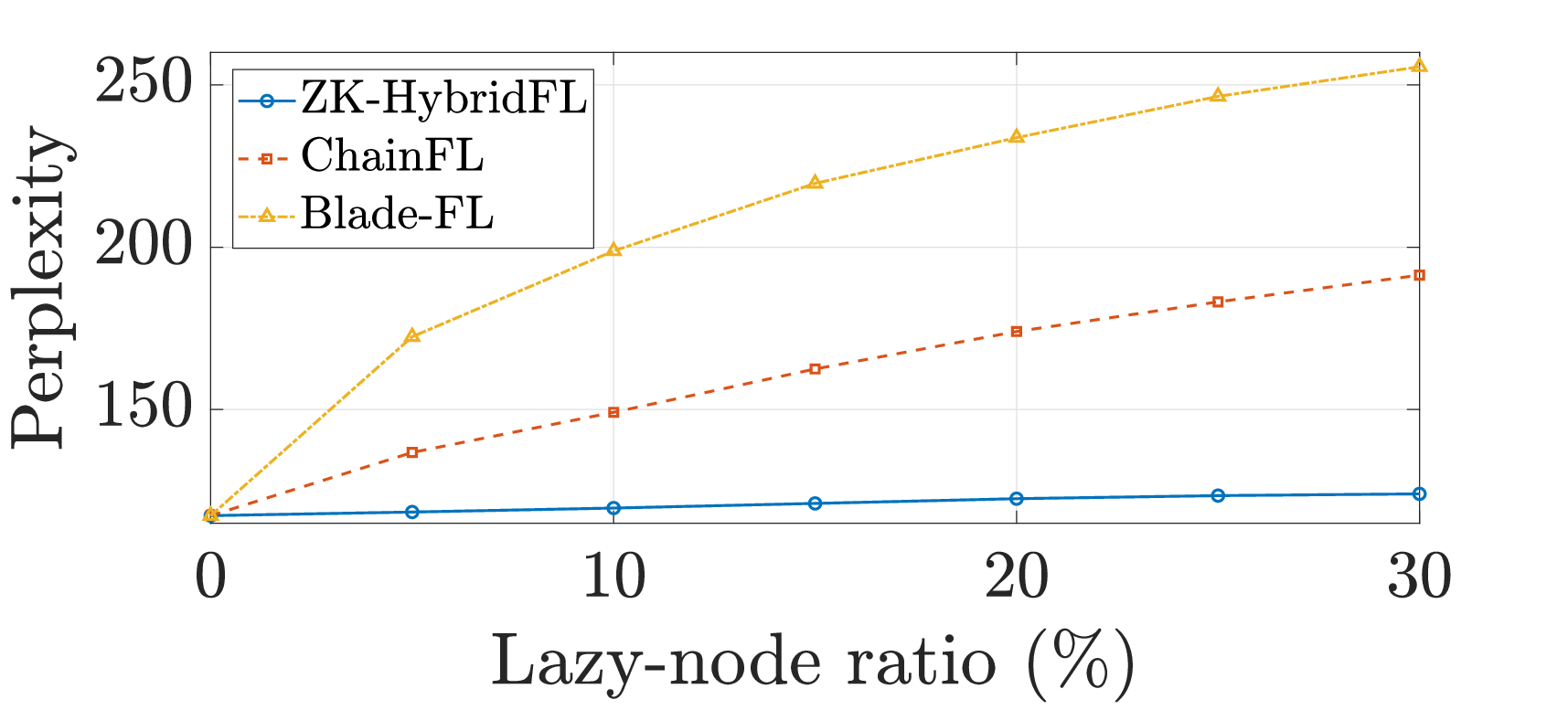}%
  }
  \caption{Lazy-only scenario ($\gamma = 0$). Solid: ZK-HybridFL; dashed: ChainFL; dash-dot: Blade-FL.}
  \label{fig:lazy-only}
\end{figure}

\subsubsection{Adversarial Utility on the Public Reference Set}
\label{sec:adv-public}

We now make the ``black-box'' behavior of utility-preserving adversaries explicit. A model-poisoning adversary is one whose local update achieves essentially the \emph{same} utility on the public validation set $D_{pub}$ as an honest update---i.e., $\operatorname{Acc}(\tilde{\mathbf{W}}_{\text{adversary}},D_{pub}) \ge \operatorname{Acc}(\tilde{\mathbf{W}}_{\text{honest}},D_{pub})-\varepsilon$ with $\varepsilon=0.01$---yet is crafted to maximize the drift of the global model after aggregation. This is exactly the threat captured in \cref{sec:adv-motivation}, but until now we had shown only the aggregate effect (\cref{nodes_vary,adversary_vary}).

\Cref{fig:adv-public} opens the box for ChainFL for $\gamma=0.10$, $\mu=0.20$, $n=20$. The dashed line plots the per-epoch performance of adversarial updates on $D_{pub}$; the dotted line shows honest updates; the solid line is the global model on the full private test distribution. Although adversarial nodes sustain $\ge0.89$ accuracy on MNIST and $\le127.6$ perplexity (indistinguishable from honest nodes), they cause the aggregated model to lag by $6$--$9$ percentage points in accuracy and to converge eight perplexity points higher. This directly substantiates the claim in \cref{sec:adv-motivation} that such nodes ``maintain good performance on public datasets while gradually undermining the global model,'' and explains why schemes that rely on public-set admission (ChainFL, Blade-FL) remain vulnerable.

\begin{figure}[t]
  \centering
  \subfloat[Task 1 Accuracy\label{fig:adv-public-1}]{%
    \includegraphics[width=\columnwidth]{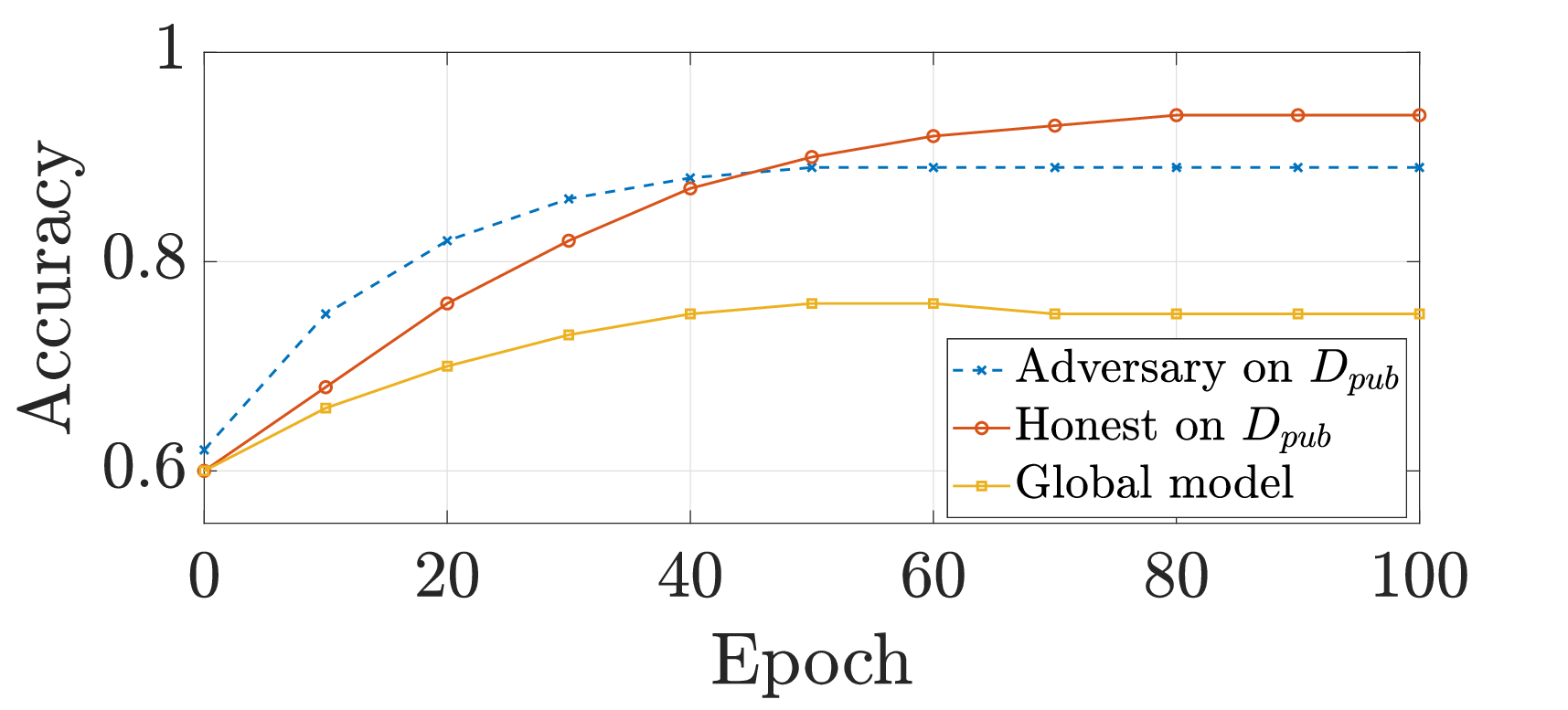}%
  }\\[-0.25\baselineskip]
  \subfloat[Task 2 Perplexity\label{fig:adv-public-2}]{%
    \includegraphics[width=\columnwidth]{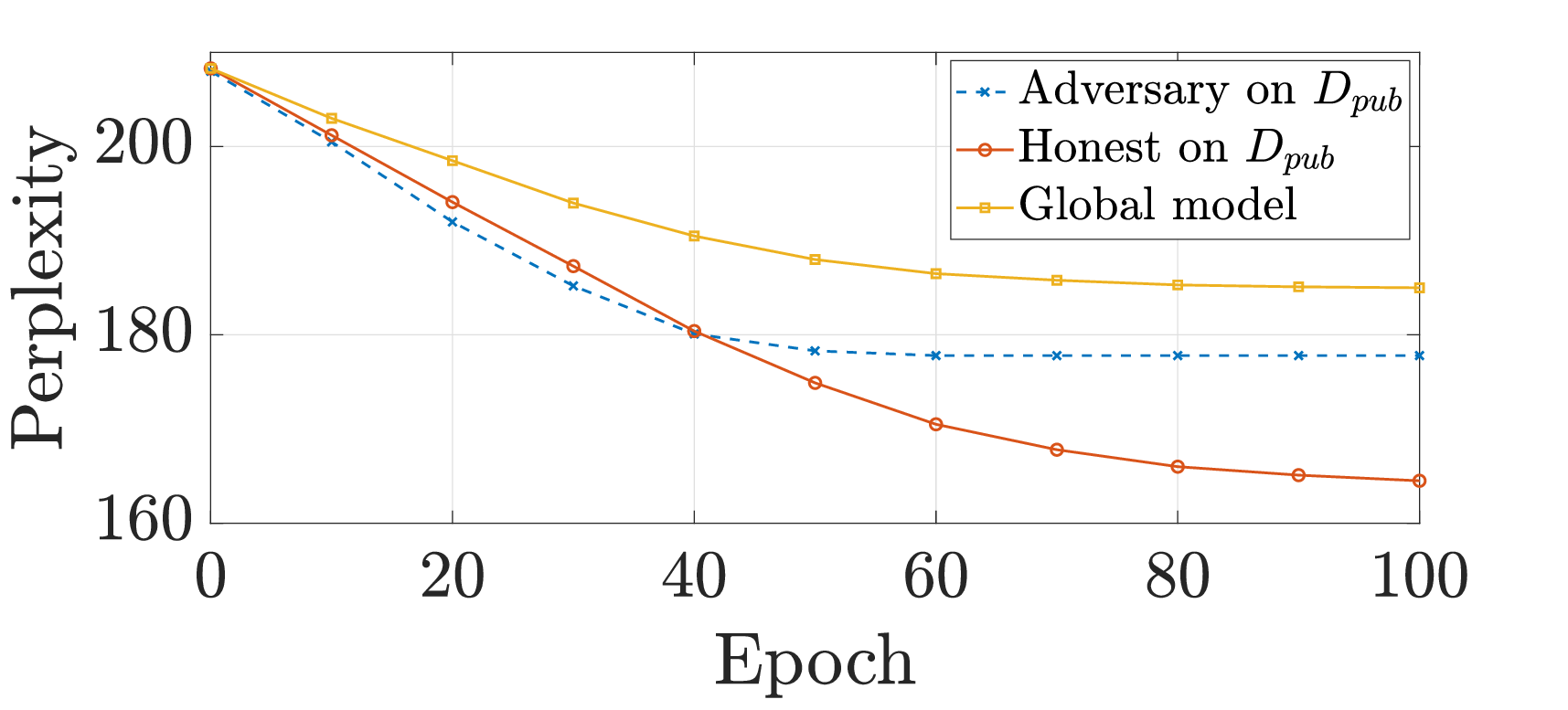}%
  }
  \caption{Utility-preserving adversaries for ChainFL: high public-set utility yet harmful global impact (mixed-fault setting, $\gamma=0.10$, $\mu=0.20$, $n=20$).}
  \label{fig:adv-public}
\end{figure}

\end{document}